\documentclass{article}

\PassOptionsToPackage{numbers, sort&compress}{natbib}

\usepackage{iclr2027/iclr2027_conference,times}

\usepackage{amsmath,amsfonts,bm}

\def\eqref#1{equation~\ref{#1}}

\def\1{\bm{1}}

\DeclareMathAlphabet{\mathsfit}{\encodingdefault}{\sfdefault}{m}{sl}
\SetMathAlphabet{\mathsfit}{bold}{\encodingdefault}{\sfdefault}{bx}{n}

\usepackage{etoolbox}
\newbool{published}
\setbool{published}{true} 

\usepackage{titletoc}

\definecolor{mygreen}{RGB}{0,120,70}
\definecolor{myblue}{RGB}{0,70,150}

\usepackage{calc}
\usepackage{multirow}

\newcommand{\mathcolorbox}[2]{%
  \mathchoice
    {\colorbox{#1}{$\displaystyle #2$}}%
    {\colorbox{#1}{$\textstyle #2$}}%
    {\colorbox{#1}{$\scriptstyle #2$}}%
    {\colorbox{#1}{$\scriptscriptstyle #2$}}%
}

\usepackage[utf8]{inputenc} 
\usepackage[T1]{fontenc}    
\usepackage{hyperref}       
\usepackage{url}            
\usepackage{booktabs}       
\usepackage{amsfonts}       
\usepackage{nicefrac}       
\usepackage{microtype}      
\usepackage{xcolor}         

\usepackage{braket}
\usepackage{tabularx,makecell,array}

\usepackage{upgreek}
\usepackage{bm}
\usepackage{subcaption}
\usepackage{comment} 

\def\vocab{\mathbb{V}}
\def\m{\mathbf{m}}
\def\n{\mathbf{n}}
\def\k{\mathbf{k}}
\def\l{\mathbf{l}}
\def\bpsi{\bm{\uppsi}}
\def\mc#1{\mathcal{#1}}
\def\mb#1{\mathbb{#1}}

\def\NumToken{N}
\def\tr{\text{Tr}}
\def\op#1{{#1}}

\usepackage{amsmath}
\usepackage{amssymb}
\usepackage{mathtools}
\usepackage{amsthm}

\usepackage[capitalize,noabbrev]{cleveref}

\theoremstyle{plain}

\theoremstyle{definition}

\theoremstyle{remark}

\usepackage[most]{tcolorbox}
\usepackage{amsmath, amsthm, amssymb}

\definecolor{proofblue}{HTML}{DEB887}
\definecolor{proofline}{HTML}{DEB887}

\tcbset{
    academicstyle/.style={
        enhanced,
        boxrule=0pt,
        frame hidden,     
        interior hidden,  
        left=3mm,
        right=0mm,
        top=0mm,
        bottom=0mm,
        sharp corners,
        borderline west={1.5pt}{0pt}{gray!80}, 
        breakable,
        before skip=12pt,
        after skip=12pt,
    }
}

\newtcbtheorem[number within=section]{mydefinition}{Definition}{
    academicstyle,
    fonttitle=\bfseries,
    coltitle=black,
    attach title to upper,
    after title={\hspace{0.5em}}
}{def}

\crefname{tcb@cnt@mydefinition}{Definition}{Definitions}

\newtcbtheorem[use counter from=mydefinition]{mytheorem}{Theorem}{
    academicstyle,
    colback=blue!3,
    borderline west={1.5pt}{0pt}{proofblue},
    fonttitle=\bfseries,
    coltitle=black,
    attach title to upper,
    after title={\hspace{0.5em}},
    label type=theorem 
}{thm}

\crefname{tcb@cnt@mytheorem}{Theorem}{Theorems}

\usepackage[most]{tcolorbox}

\newtcolorbox{myproofbox}{
    enhanced,
    boxrule=0pt,
    frame hidden,
    interior hidden,
    left=4mm,
    right=0mm,
    top=0mm,
    bottom=2mm,
    sharp corners,
    borderline west={0.2pt}{0pt}{proofline, dashed},
    breakable,
    before skip=2pt,
    after skip=12pt,
}

\renewenvironment{proof}[1][\proofname]{
    \begin{myproofbox}
    \tcbox[
        on line, 
        sharp corners, 
        boxrule=0.5pt, 
        colframe=proofline, 
        colback=white, 
        boxsep=0pt, 
        left=3pt, right=3pt, top=2pt, bottom=2pt,
        after skip=5pt
    ]{\textit{#1}}
}{
    \qed \end{myproofbox} 
}

\title{Diffusion Language Models as Open Quantum Systems}

\author{%
  Akinori Tanaka$^{1,2,3}$, Yukinao Akamatsu$^{4}$, Takuki Tomita$^{1,3}$, Yuji Hirono$^{5}$\\
  {\small $^{1}$RIKEN Center for Advanced Intelligence Project (AIP), Japan}\\
  {\small $^{2}$RIKEN Center for Interdisciplinary Theoretical and Mathematical Sciences (iTHEMS), Japan}\\
  {\small $^{3}$Department of Mathematics, Keio University, Japan}\\
  {\small $^{4}$National Institute of Technology, Matsue College, Japan}\\
  {\small $^{5}$Institute of Systems and Information Engineering, University of Tsukuba, Japan}\\
  {\small\texttt{akinori.tanaka@riken.jp}, \texttt{akamatsu.yukinao@matsue-ct.ac.jp}}\\
  {\small\texttt{takuki.tomita@riken.jp}, \texttt{hirono@iit.tsukuba.ac.jp}}
}

\iclrfinalcopy

\begin{document}

\maketitle
\lhead{Preprint}

\begin{abstract}

Diffusion language models generate text by reversing a classical diffusion process over token sequences.
We show that their continuous-time forward and reverse dynamics can be embedded in the framework of open quantum systems, which describes quantum systems interacting with their environment.
This connection provides a basis for extending diffusion sampling to superpositions of sentence states and for modifying the underlying generation dynamics.
We construct a $\tau$-leaping unraveling scheme that simulates stochastic trajectories of these states while updating multiple token positions at each step.
This enables generation conditioned on superposed prompts or thoughts, which we call Superposition of Prompts (SoP) and Superposition of Thoughts (SoT).
The same formulation yields Quantum-Inspired Remasking (QIR): perturbing the masked-diffusion dynamics produces a confidence-based rule for revising previously generated tokens.
Experiments with pretrained diffusion language models show that superposition-based generation can retain information from conditioning components not selected at the final measurement, while QIR improves LLM-assessed text quality when applied at an appropriate stage of generation.

\end{abstract}

\section{Introduction}
\label{sec:intro}

Diffusion models have achieved remarkable success in generative tasks such as image and video synthesis \citep{Sohl2015DeepUL, Ho2020DenoisingDP, Song2019GenerativeMB, Song2020ScoreBasedGM, Dhariwal2021DiffusionMB, Rombach2021HighResolutionIS, Ho2022VideoDM, Singer2022MakeAVideoTG}.
Diffusion-based approaches are also increasingly being explored for text generation as an alternative to autoregressive models \citep{Sun2022ScorebasedCD, Lou2023DiscreteDM, Xu2024EnergyBasedDL, Ou2024YourAD, Sahoo2025TheDD, Nie2025LargeLD, Zhao2025d1SR}, including recent commercial systems \citep{Khanna2025MercuryUL, Song2025SeedDA}.

From a theoretical point of view, diffusion models are deeply connected to stochastic processes in non-equilibrium statistical mechanics \citep{Sohl2015DeepUL}. 
This connection has motivated a growing body of work that analyzes diffusion models using concepts and formulations developed in physics \citep{Hirono2024UnderstandingDM, Yu2024NonequilbriumPO, Yaguchi2025TheGO, Ikeda2024SpeedaccuracyTF}. 
Conventional diffusion models can be viewed as analogs of \textit{classical} thermal systems, whose dynamics can be described within a more general \textit{quantum} framework.

In this paper, we formulate diffusion language models within the framework of \textit{open quantum systems}
\citep{Breuer2002, Lindblad1976, Gorini1976}.
Specifically, we show that their continuous-time dynamics can be embedded in the Lindblad equation, which describes the evolution of a quantum system interacting with its environment.
This formulation opens the way to extensions of diffusion sampling based on concepts from quantum physics.
One such concept is \textit{superposition}.
As illustrated in \cref{fig:concept_oqs}, distinct sentences can be represented as basis states, and a linear combination of these states defines a superposition state $|\bpsi_t\rangle$.
Superposition is fundamentally different from concatenation: the sentences remain distinct alternatives within a quantum state, rather than becoming successive parts of a longer text.
The corresponding density operator evolves according to the Lindblad equation.
This representation introduces a form of generation in which multiple sentence components and their relative weights evolve together, with the components influencing one another rather than evolving as independent samples.

\begin{figure}[htbp]
    \centering
    \begin{minipage}{0.5\textwidth}
        \includegraphics[width=\textwidth]{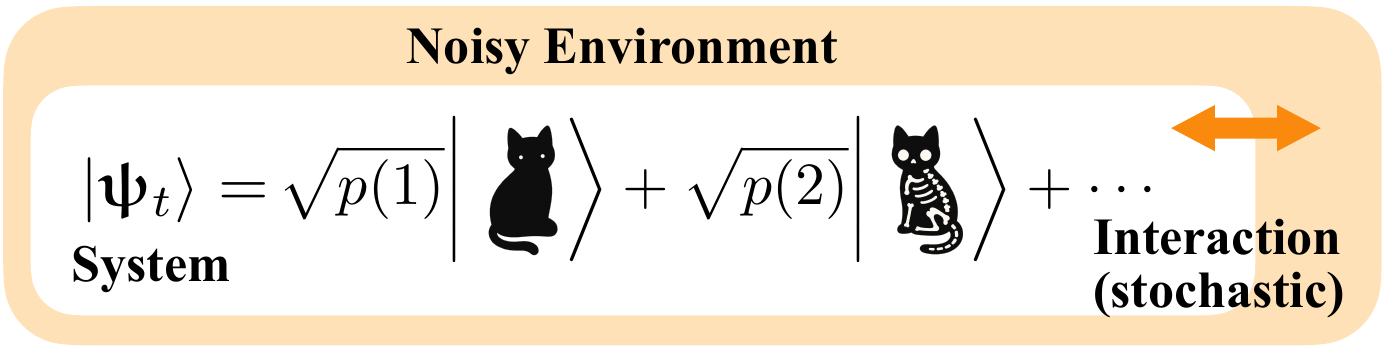}
    \end{minipage}
    \hspace{1em}
    \begin{minipage}{0.45\textwidth}
        \caption{A superposition of sentence states viewed as an open quantum system interacting with its environment.}
        \label{fig:concept_oqs}
    \end{minipage}
\end{figure}

We develop a sampling procedure based on $\tau$-leaping unraveling, which evolves superpositions of sentence states by updating multiple token positions in each step.
It builds on \textit{stochastic unraveling}, an established technique for representing Lindblad dynamics through stochastic pure-state trajectories \citep{Dalibard1992, Dum1992, Carmichael1993}.
The idea is analogous to simulating trajectories of a stochastic differential equation instead of directly solving the corresponding Fokker--Planck equation.
A final measurement selects an output sentence according to the Born rule.
The same formulation also leads to quantum-inspired remasking: a perturbation of the masked-diffusion dynamics yields a confidence-based rule that preferentially remasks tokens the model is less confident in reconstructing from their context.

Our main contributions are summarized as follows:
{
\setlength{\leftmargini}{2em}
\begin{itemize}
  \item We show that continuous-time discrete diffusion language models can be embedded in the Lindblad equation, with reverse diffusion recovered in the diagonal sector of the reverse-time dynamics. This embedding allows us to extend diffusion language models by modifying the dynamics or allowing superpositions of sentence states.

  \item By allowing superpositions in the same formulation, we introduce \textbf{Superposition of Prompts (SoP)} and its application to thought-conditioned generation, \textbf{Superposition of Thoughts (SoT)}.
  These methods generate a text through stochastic unraveling followed by measurement. Our experiments show that the output can retain information from conditioning components that are not selected in the final measurement.

  \item By perturbing the embedded masked diffusion dynamics, we derive \textbf{Quantum-Inspired Remasking (QIR)}.
  Our experiments show that appropriately timed QIR improves LLM-assessed text quality, whereas applying it too early increases repetition despite lowering perplexity.
\end{itemize}
}
These methods illustrate the broader design space opened by our formulation, which brings concepts from quantum physics into the design of language generation processes.
Exploring this space may lead to further sampling methods on classical hardware, as well as future implementations on quantum computers.

\subsection{Related Works}
\label{sec:related_work}

\textbf{Diffusion Language Models.}\quad
Discrete diffusion models were initially introduced for binary data \citep{Sohl2015DeepUL} and later extended to discrete state spaces.
Our mathematical foundation is the continuous-time Markov formulation of discrete diffusion \citep{Campbell2022ACT, Sun2022ScorebasedCD}, particularly that of \citet{Lou2023DiscreteDM}.
Two widely studied approaches are \textit{uniform diffusion} and \textit{masked diffusion}.

Uniform diffusion allows tokens to be revised throughout generation, providing a mechanism for correcting earlier predictions \citep{Hoogeboom2021ArgmaxFA}.
Although early models faced challenges in language modeling \citep{Austin2021StructuredDD, Lou2023DiscreteDM}, recent advances have improved training \citep{Sahoo2025TheDD} and led to practical open-weight models \citep{team2026diffusiongemma}.
Masked diffusion generates a text by progressively replacing mask tokens and has been widely adopted for its stable training \citep{Austin2021StructuredDD, Lou2023DiscreteDM}.
\citet{Ou2024YourAD} showed that the structure of masked diffusion permits time-independent modeling and reduces the number of function evaluations during sampling.
Masked diffusion also underlies large-scale language models such as LLaDA \citep{Nie2025LargeLD} and iLLaDA \citep{Nie2026iLLaDA}.

Standard masked diffusion cannot revise tokens once unmasked. 
\textit{Remasking} allows these tokens to be masked again and regenerated.
\citet{Nie2025LargeLD} retain a prescribed number of token predictions based on confidence at each step.
WINO \citep{hong2026wide} also revisits previously generated tokens, remasking those below a verification confidence threshold.
Our QIR instead derives stochastic remasking rates from perturbed Lindblad dynamics, assigning higher rates to tokens with lower reconstruction confidence.
Alternatively, \citet{wang2026remasking} reformulate sampling to incorporate remasking, drawing on DDIM \citep{song2021denoising}.

\textbf{Quantum Computation and Diffusion Models.}\quad
The connection between quantum computation and generative modeling has been explored in several directions. Early approaches include quantum generative models such as quantum Boltzmann machines and variational quantum circuits \citep{Amin2018QBM, Biamonte2017QML, Benedetti2019VQC}, where probability distributions are represented by quantum states and sampled via measurement.

\textbf{Diffusion Models and Physics.}\quad
Diffusion models have been studied from the perspectives of
non-equilibrium statistical mechanics and stochastic processes \citep{Sohl2015DeepUL, Hirono2024UnderstandingDM, Yu2024NonequilbriumPO, Yaguchi2025TheGO, Ikeda2024SpeedaccuracyTF}.
Connections between classical diffusion and quantum dynamics also arise through imaginary-time evolution and path-integral formulations \citep{FeynmanHibbs1965}.
\citet{Nasu2025QuantumReversibility} connect Lindblad dynamics to classical reverse diffusion through a semiclassical approximation.
Our work considers an exact embedding of discrete token dynamics and uses it to develop sampling procedures for diffusion language models.
\section{Background}
\label{sec:dlm_as_oqs}

We review discrete diffusion models and the Lindblad description of open quantum systems as the background for the formulation in the next section.
Our notation is summarized in \cref{sec:notation}.

\subsection{Discrete Diffusion Models}

\newcommand{\mathhighlight}[1]{
  \begingroup
  \setlength{\fboxsep}{1pt}
  \mathcolorbox{blue!10}{#1}
  \endgroup
}

In continuous-time discrete diffusion models, the forward process gradually corrupts token sequences.
The corresponding probability distribution evolves according to the master equation
\begin{align}
    \frac{d}{dt}p_t(\m)
    =
    \sum_{\n \in \vocab^\NumToken}
        Q_t(\m, \n) p_t(\n)
    , \quad 
    \left\{
      \begin{array}{ll}
        Q_t(\m ,\n) \geq 0 & (\m \neq \n)
        \\
        \sum_{\m \in \vocab^\NumToken} Q_t(\m , \n) = 0
      \end{array}
    \right.
    , \label{eq:diff_eq}
\end{align}
where $\m, \n \in \vocab^\NumToken$ are token sequences of length $\NumToken$, and $Q_t$ is the transition-rate matrix.
Generation proceeds by reversing this diffusion process.
The probability distribution then evolves according to
\begin{align}
    -
    \frac{d}{dt}p_t(\m)
    =
    \sum_{\n \in \vocab^\NumToken}
        \tilde{Q}_t(\m, \n) p_t(\n)
    , \quad 
    \left\{
      \begin{array}{ll}
        \tilde{Q}_t(\m ,\n) 
        = 
        Q_t(\n, \m) 
        \mathhighlight{
          \frac{p_t(\m)}{p_t(\n)} 
        }
          & (\m \neq \n)
        \\
        \sum_{\m \in \vocab^\NumToken} \tilde{Q}_t(\m , \n) = 0
      \end{array}
    \right.
    . \label{eq:inv_diff_eq}
\end{align}

Simulating the reverse process requires the probability ratios $\mathhighlight{p_t(\m)/p_t(\n)}$, known as concrete scores \cite{Lou2023DiscreteDM}, which the model is trained to estimate.
A common choice is to let each transition change only one token position.
The forward and reverse generators then take the form
\begin{align}
\hspace{-0.8em} Q_t(\m, \n)
  &=
  \sigma(t)
  \sum_{p=1}^{\NumToken}
    1_{\m_{-p} \n_{-p}}
    Q^\text{tok}(m_p, n_p)
  ,\label{eq:delta_p}
  \\
\hspace{-0.8em} \tilde{Q}_t(\m, \n)
  &=
  \sigma(t)
  \sum_{p=1}^{\NumToken}
    1_{\m_{-p} \n_{-p}}
    \tilde{Q}^{(p)\text{tok}}(m_p, \n)
  ,\label{eq:tdelta_p}
  \;
  \tilde{Q}^{(p)\text{tok}}(m_p, \n)
  =
  Q^\text{tok}(n_p, m_p)
  \smash{
    \overbrace{
      \mathhighlight{
        \frac{p_t(\n_{p \to m_p})}{p_t(\n)}
      }
    }^{
      \mathhighlight{s(\n,t)_{p \to m_p}}
    }
  }
  \vphantom{\frac{p_t(\n_{p \to m_p})}{p_t(\n)}}
  ,
\end{align}
where $\sigma(t)$ controls the time-dependent transition rate, and $p$ indexes token positions.
The indicator $1_{\m_{-p}\n_{-p}}$ equals one when $\m$ and $\n$ agree at all positions other than $p$, and zero otherwise.
The sequence $\n_{p\to m_p}$ is obtained from $\n$ by replacing its $p$-th token with $m_p$.
With this structure, the model only needs to estimate the scores $\mathhighlight{s(\n,t)_{p\to m_p}}$ associated with single-token changes, rather than probability ratios between arbitrary pairs of sequences.
For further details, see \cref{sec:dlm_intro}.

\subsection{Open Quantum Systems and the Lindblad Equation}

An open quantum system describes the effective dynamics of a quantum system interacting with its environment.
Its state is represented by a density operator $\op{\rho}_t$ on a Hilbert space $\mc{H}$.
Classical probability distributions are recovered as a special case.
Given the sentence basis $\{ \ket{\m} \}_{\m \in \vocab^\NumToken}$, a probability distribution $p_t(\m)$ over sequences can be embedded into a diagonal density operator,
\begin{align}
  \op{\rho}_t
  =
  \sum_{\m \in \vocab^\NumToken}
  p_t(\m) \ket{\m}\bra{\m}.
  \label{eq:classical_state_embedding}
\end{align}
Thus, a discrete diffusion process can be viewed as open-system dynamics restricted to the diagonal sector.
Going beyond this diagonal sector amounts to allowing coherent superpositions of token sequences, which will be exploited later in our sampling construction.

The Markovian dynamics of an open quantum system is governed by the Lindblad equation,
\begin{align}
  \frac{d}{dt}\op{\rho}_t
  &=
  - i [\op{H}_t, \op{\rho}_t] 
  + 
  \sum_\mu 
  \op{L}_t^\mu \op{\rho}_t \op{L}_t^{\mu\dagger} 
  - \frac{1}{2} \left\{ 
      \sum_\mu 
      \op{L}_t^{\mu\dagger} \op{L}_t^\mu, \op{\rho}_t
      \right\}.
  \label{eq:lindblad-main}
\end{align}
Here, $\op{H}_t$ is a Hermitian operator, called the Hamiltonian, and $\op{L}_t^\mu$ are jump operators.
The Hamiltonian term generates coherent evolution, while the jump terms describe dissipative transitions induced by the environment.
The Lindblad equation preserves positivity and unit trace of $\op{\rho}_t$, and is therefore the quantum analogue of a classical master equation (\ref{eq:diff_eq}).

The open-system formulation also admits a reverse-time dynamics, paralleling the reverse diffusion equation~(\ref{eq:inv_diff_eq}).
Given the forward Lindblad dynamics and its solution $\op{\rho}_t$, one can define time-reversed operators $\tilde{H}_t$ and $\tilde{L}_t^\mu$ such that
\begin{align}
  -\frac{d}{dt} \op{\rho}_{t}
  &=
  - i [\tilde{H}_t, \op{\rho}_{t}]
  + \sum_\mu \tilde{L}_t^\mu \op{\rho}_{t} \tilde{L}_t^{\mu\dagger}
  - \frac{1}{2} \left\{
      \sum_\mu \tilde{L}_t^{\mu\dagger}\tilde{L}_t^\mu, \op{\rho}_{t} 
      \right\}.
  \label{eq:lindblad_rev-main}
\end{align}
The explicit expressions for $\tilde{H}_t$ and $\tilde{L}_t^\mu$ are reviewed in \cref{thm:thm_time_rev}.
For diagonal embeddings of the form (\ref{eq:classical_state_embedding}), this reverse-time Lindblad equation reduces to the reverse-time dynamics of discrete diffusion models.
This connection will be used below to formulate inference as a stochastic unraveling of the reverse open-system dynamics.
For further details, see \cref{sec:quantum_intro}.
\section{A Quantum Formulation of Diffusion Language Models}
\label{sec:dlm_as_oqs}

\subsection{Embedding Diffusion Language Models into Open Quantum Systems}
\label{sec:emb}

We now show that diffusion language models can be embedded in the Lindblad framework for open quantum systems.
Specifically, the discrete diffusion dynamics can be reproduced by choosing jump operators that act on individual token positions, following the transition structure in \cref{eq:delta_p}.
The following theorems establish this correspondence.
The proofs are given in \cref{sec:pauli}.
\begin{mytheorem}{Embedding of Diffusion Processes}{dlm_as_oqs}
  Define the Hamiltonian and jump operators with $\mu=(p, m, n)$ acting on the sentence Hilbert space $\mc{H}$ as follows:
  \begin{align}
    &
    \op{H}_t
    =
    0,
    \quad
    \op{L}{_t^{\mu=(p, m, n)}}
    =
    {I \otimes \cdots \otimes}
    {\sqrt{\sigma(t)} \op{L}_{m, n}^{\text{$(p)$tok}}}
    {\otimes \cdots \otimes I},
    \label{eq:full_jump_operators}
  \end{align}
  where we define operators acting on the one-token Hilbert space $\mc{H}_\vocab$ by
  \begin{align}
    \op{L}{_{m, n}^\text{tok}}
    =
    \sqrt{
      Q^\text{tok}(m, n)
    }
    \ket{m}\bra{n}
    ,
    \quad
    (m \neq n)
    ,
    \label{eq:mLk-main}
  \end{align}  
    where $m, n \in \vocab$ are tokens before and after the jump.
  Then the Lindblad equation in \cref{eq:lindblad-main} reduces to be equivalent to the classical diffusion equation in \cref{eq:diff_eq} with \cref{eq:delta_p}.
\end{mytheorem}

\begin{mytheorem}{Embedding of Diffusion Models}{tilde_L}
  Under our setting in \cref{thm:dlm_as_oqs}, we have the operators for the time-reversal equation in \cref{eq:lindblad_rev-main} as follows:
  \begin{align*}
    \tilde{H}_t = 0,
    \quad
    \tilde{L}{_t^{\mu=(p, m, n)}}
    &=
      \sum_{\k \in \vocab^\NumToken}
        \begin{array}{c}
          \ket{k_1}\bra{k_1}
            \otimes 
            \cdots 
            \otimes 
            \sqrt{\sigma(t)}
            \tilde{L}{_{m, n, t}^\text{$(p)$tok}}(\k)
            \otimes 
            \cdots 
            \otimes 
            \ket{k_\NumToken}\bra{k_\NumToken}
            ,
        \end{array}
  \end{align*}
  where we define operators acting on the one-token Hilbert space $\mc{H}_\vocab$ by
  \begin{align}
    \tilde{L}{_{m, n, t}^\text{$(p)$tok}}(\k)
        =
        \sqrt{
          \tilde{Q}^\text{$(p)$tok}(n, \k)
          }
            1_{m k_p}
          \ket{n}\bra{k_p}
        =
        \sqrt{
          Q^\text{tok}(k_p, n)
          \mathhighlight{
            \frac{
              p_t(\k_{p \to n})
            }{
              p_t(\k)
            }
          }
          }
            1_{m k_p}
          \ket{n}\bra{k_p}
        .
  \end{align}
With these operators, the time-reversed Lindblad equation (\ref{eq:lindblad_rev-main}) reproduces the reverse diffusion dynamics in \cref{eq:inv_diff_eq}, with the generator given by \cref{eq:tdelta_p}.
\end{mytheorem}
This embedding provides a starting point for extending diffusion sampling to superposition states and perturbed jump dynamics, which lead to the sampling schemes
introduced below.
Below, we show how the well known diffusion model settings are embedded in the Lindblad equation.

\textbf{Uniform Diffusion.}\quad
If we consider $(m,n) \in \vocab^2$ where $m \neq n$ and consider jump operators 
\begin{align}
  L{_{{m}, {n}}^\text{tok}}  =  \ket{{m}}\bra{{n}}
  ,
  \
  \Big(
    Q^\text{tok}(m, n) = 1
  \Big)
  \quad
  \Rightarrow
  \quad
  \tilde{L}{_{{m}, {n},t}^\text{$(p)$tok}}(\k)  
  =  
  \sqrt{
    \mathhighlight{
      \frac{
        p_t(\k_{p \to n})
      }{
        p_t(\k_{p \to m})
      }
    }
  }
    1_{m k_p}
  \ket{{n}}\bra{m}.
  \label{eq:jump_uniform}
\end{align}

\textbf{Masked Diffusion.}\quad
If we consider $i \in \vocab$ where $i \neq 0$ and consider jump operators 
\begin{align}
  L{_{0, i}^\text{tok}}  =  \ket{0}\bra{i}
  ,
  \Big(
    Q^\text{tok}(0, i) = 1
  \Big)
  \quad
  \Rightarrow
  \quad
  \tilde{L}{_{0, i,t}^\text{$(p)$tok}}(\k)  
  =  
  \sqrt{
    \mathhighlight{
      \frac{
        p_t(\k_{p \to i})
      }{
        p_t(\k_{p \to 0})
      }
    }
  }
    1_{0 k_p}
  \ket{i}\bra{0}.
\end{align}
\noindent
\begin{minipage}[c]{0.73\textwidth}
For uniform diffusion defined by \cref{eq:jump_uniform}, transitions between token states occur at equal rates in both directions.
This can be viewed as thermal dynamics among states of equal energy.

In masked diffusion, transitions occur only from non-mask states $\ket{i}$ to the absorbing mask state $\ket{0}$.
A thermal interpretation is obtained by assigning a lower energy to $\ket{0}$: when the energy gap is large compared with the temperature, upward transitions are suppressed relative to downward transitions.
Masked diffusion corresponds to the limiting case in which upward transitions are absent.
\end{minipage}%
\hfill
\begin{minipage}[c]{0.25\textwidth} 
    \includegraphics[
        bb=0 0 100 100,
        trim=10 0 10 0,
        clip,
        width=\columnwidth
    ]{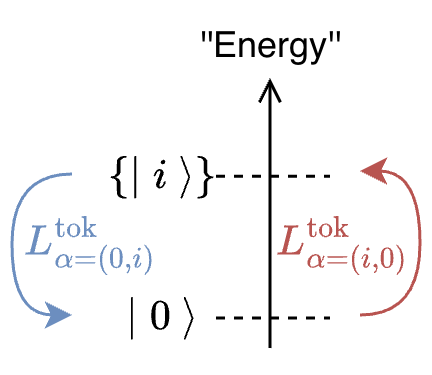}
\end{minipage}

\textbf{Perturbed Masked Diffusion.}\quad
To mitigate this energy gap, one approach would be to introduce corresponding jump operators operating in the opposite direction into the system, i.e., with $0\leq \gamma \leq 1$:
\begin{align}
    L{_{i, 0}^\text{tok}}  = \sqrt{\gamma} \ket{i}\bra{0}
    \quad
    \Rightarrow
    \quad
    \tilde{L}{_{i, 0,t}^\text{$(p)$tok}}(\k)  
    =  
    \sqrt{\gamma}
    \sqrt{
      {
        \frac{
          p_t(\k_{p \to 0})
        }{
          p_t(\k_{p \to i})
        }
      }
    }
      1_{i k_p}
    \ket{0}\bra{i},
  \label{eq:remask_jump}
\end{align}
where $\gamma$ regulates the energy gap, with the gap completely closing when $\gamma = 1$.
When $\gamma \approx 0$, it can be physically interpreted as a normal Masked diffusion process with added perturbations.
Leveraging this latter property, the second half of this paper proposes a Quantum-Inspired Remasking (QIR) method based on \cref{eq:remask_jump} and demonstrates its effectiveness.

\subsection{Simulations: Stochastic Unraveling}
\label{sec:unraveling-main}

The Lindblad formulation also provides a trajectory-level description of the dynamics.
In ordinary diffusion language models, sampling is performed by simulating a Markov process over token sequences.
The open-system counterpart is a stochastic process over wave functions, known as \textit{stochastic unraveling}.
It represents the same density-operator dynamics through random pure-state trajectories.
In this language, a stochastic time-evolution operator $\op{U}_{t - dt | t}^\text{eff}$ updates the state vector as
\begin{align}
        \ket{\bpsi_{t- dt}}
        =
        \tilde{U}_{t - dt | t}^{\text{eff}}
        \ket{\bpsi_t}
        &\propto
        (1 -i \tilde{H}{^\text{eff}_t}dt)\ket{\bpsi_t}
        + \sum_\mu (\tilde{L}{_t^\mu} - 1) dN_t^\mu \ket{\bpsi_t}
        ,
\end{align}
where $\tilde{H}{^\text{eff}_t} = -\frac{i}{2}\sum_\mu \tilde{L}{_t^{\mu\dagger}}\tilde{L}{_t^\mu}$ and $dN_t^\mu$ are multiple jump processes.
For explicit expressions, see \cref{sec:unraveling}.
A direct application of these infinitesimal unravelings is inefficient for the same reason as the naive simulation of discrete diffusion: in a single step, they affect only one token position.
Motivated by the $\tau$-leaping sampler for discrete diffusion, we introduce an unraveling method that allows multiple token positions to be updated within one step.
$
  \ket{\bpsi_{t- \Delta t}}
  =
  \tilde{U}_{t- \Delta t|t}^{\text{$\tau$-leap}}
  \ket{\bpsi_{t}}
$.
We refer to this method as \textit{$\tau$-leaping unraveling}.
This expression is the unraveling analogue of the $\tau$-leaping sampler in \cref{eq:tau-leaping}\footnote{
  Notably, the latter $\tilde{U}_{t - \Delta t | t}^\text{$\tau$-leap}$ becomes essentially equivalent to conventional $\tau$-leaping with $\ket{\bpsi} = \ket{\m}.$
}.
It reproduces the solution of the reverse-time Lindblad equation \cref{eq:lindblad_rev-main} up to first order in $\Delta t$; see \cref{thm:tau-qjump}.

A central feature of stochastic unraveling is that an ensemble of stochastic pure-state trajectories reproduces the solution of the Lindblad equation at the level of the density matrix.
In the sampling procedure, we need to perform a ``projective measurement,'' since the wave function $\ket{\bpsi_t}$ is, in general, a superposition of multiple sentences. 
According to the Born rule, the probability of measuring a particular sentence is given by the squared magnitude of its coefficient.
Starting from $\ket{\n}$, evolving from $t$ to $s$ via
stochastic unraveling, and then measuring yields $\m$
with the same probability as the corresponding classical
diffusion process.
This holds for both forward- and reverse-time dynamics.
For details, see \cref{seq:prf_sampling}.

\begin{figure}[tbp]
    \includegraphics[
        bb=0 0 100 100,
        trim=0cm 0 0cm 0,
        clip,
        width=\columnwidth]{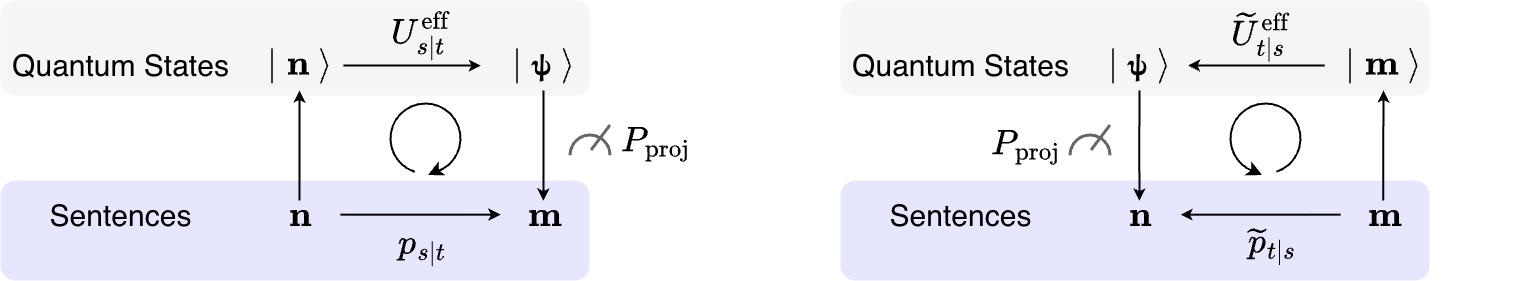}
  \caption{
    Sampling by stochastic unraveling followed by projective measurement.
    The ensemble of pure-state trajectories reproduces the corresponding density-operator dynamics, while the final measurement yields a discrete token sequence.
}
  \label{fig:unraveling_projection}
\end{figure}

\subsection{Quantum-Inspired Samplings}
\label{sec:qis}

Embedding diffusion language models into the Lindblad framework opens up new possibilities for extending both their states and dynamics.
We explore two such extensions here.
Superpositions of sentence states lead to Superposition of Prompts (SoP) and its thought-conditioned variant, Superposition of Thoughts (SoT), while perturbations of the jump operators yield Quantum-Inspired Remasking (QIR).
Implementation details can be found in \cref{sec:implementation-details}. 

\textbf{Superposition of Prompts (SoP).}\quad 
SoP applies stochastic unraveling to a superposition of sentence states conditioned on different prompts.
In the context of score-based diffusion models, it is common to generate based on conditional probability $\m = [\texttt{Prompt} \cdots] \sim p(\cdots|\texttt{Prompt})$.
This procedure is a counterpart of sampling $\m = [\texttt{Prompt} \cdots]$, namely, we generate a reverse-time dynamics of
\begin{align}
  \ket{\bpsi_{t}}
  =
  \sqrt{p_t(1)} \Big|
    \underbrace{
        \texttt{
          Prompt 1
        }
    }_{\text{fixed}}
    \underbrace{
    \cdots_t
    }_{
    \text{generated}
    }
  \Big\rangle
  +
  \sqrt{p_t(2)} \Big|
    \underbrace{
        \texttt{
          Prompt 2
        }
    }_{\text{fixed}}
    \underbrace{
    \cdots_t
    }_{
    \text{generated}
    }
  \Big\rangle
  .
\end{align}
The part $[\cdots_{t=1}]$ is initialized by sampling from the terminal probability $p_{t=1}$.
After running reverse-time unraveling procedure defined in \cref{thm:w-super} from $t=1$ to $t \approx 0$, we apply the projective observation process to $\ket{\bpsi_{t \approx 0}}$ and make a single sentence $\m \sim [p_{t \approx 0}(1), p_{t \approx 0}(2)]$.
Afterwards, we apply usual denoising process optionally\footnote{
    If we use masked diffusion, the mask token \texttt{[m]} sometimes left behind.
    To clear it, we apply final denoizing step that jumps \texttt{[m]} to certain meaningful tokens.
}.

\textbf{Superposition of Thoughts (SoT).}\quad
We consider Superposition of Thoughts (SoT), an enhanced version of SoP, in which the prompt portion of SoP is replaced with reasoning tokens.
As an application of the unraveling with text superposition states, we can construct superposition of reasoning/thought tokens, i.e.,
\begin{align*}
  \texttt{Prompt}
  \to
  \sqrt{p_t(1)} \Bigg|
    \underbrace{
    \begin{array}{ll}
    \texttt{
      <think>
    }\\
    \texttt{
      Thought 1
    }\\
    \texttt{
      </think>
    }
    \end{array}
    }_{\text{fixed}}
    \underbrace{
    \cdots_t
    }_{
    \text{generated}
    }
  \Bigg\rangle
  +
  \sqrt{p_t(2)} \Bigg|
    \underbrace{
    \begin{array}{ll}
    \texttt{
      <think>
    }\\
    \texttt{
      Thought 2
    }\\
    \texttt{
      </think>
    }
    \end{array}
    }_{\text{fixed}}
    \underbrace{
    \cdots_t
    }_{
    \text{generated}
    }
  \Bigg\rangle
  .
\end{align*}
The sampling procedure is completely same as SoP, but it is expected to generate ``mixed thought'' outputs if the model supports reasoning mode.

\textbf{Quantum-Inspired Remasking (QIR).}\quad 
Remasking is considered a technique for improving the quality of generated sentences in mask-based diffusion language models.
Reconsidering this from the perspective of open quantum systems, "remasking" can be viewed as introducing a form of temperature when the mask token \texttt{[m]} labelled by $0$, is treated as the ground state; specifically, this corresponds to adding, in addition to the transition $0 \neq i \to 0$, a small-scale jump operator defined in \cref{eq:remask_jump} that makes $0 \to i \neq 0$ parameterized by $0 < \gamma \ll 1$ to the forward process.

From \cref{eq:remask_jump}, for the unraveling that occurs when such a "perturbation" is applied, it is necessary to calculate the following values.
\begin{align}
  \frac{p_t(\m_{p\to 0})}{p_t(\m_{p\to i})}
  =
  \left(
    \mathcolorbox{blue!10}{
      \frac{p_t(\m_{p\to i})}{p_t(\m_{p\to 0})}
    }
  \right)^{-1}.
  \label{eq:remask_score}
\end{align}
The highlighted ratio measures how strongly the model favors restoring the current token $i$ when position $p$ is masked.
We estimate this score by applying the pretrained model to $\m_{p\to 0}$ and use its reciprocal to determine the remasking rate.
This approximation neglects the change in the score induced by the perturbation, which we assume to be small for sufficiently small $\gamma$.
Thus, tokens that the model is less confident in reconstructing from their current context are more likely to be remasked.
The perturbation therefore yields a confidence-based remasking mechanism.
\section{Experiments}
\label{sec:experiments}

\subsection{Experiment 1: Sentence Generation with Text Superposition}
\label{sec:experiment_superposition}

In this section, we examine whether the superposition of prompts (SoP) and thoughts (SoT) manifests as a semantic blending in the generated text.
The goal of these experiments is not to outperform existing conditional generation methods, but rather to validate the fundamental property of the proposed quantum-inspired sampling: that a superposition of input states (prompts or thoughts) can indeed lead to the generation of a 'mixed' output state that is not simply one or the other, but a probabilistic blend of both.
Details on our experiments with SoP/SoT can be found in \cref{sec:superposition_experiments}.

\textbf{Setup.}\quad
We use masked diffusion for the experiments in this section\footnote{
    In our experiments, uniform diffusion exhibits faster concentration of superposition weights onto a single component than masked diffusion, suggesting that superpositions are less stable under uniform diffusion.
    See \cref{sec:superposition_weight_dynamics} for details. 
}.
The input prompts are qualitatively designed to reflect their respective superposition targets: \textbf{For SoP}, we employ descriptive prompts that define environmental settings to observe the superposition of semantic contexts. \textbf{For SoT}, we utilize prescriptive prompts that impose stylistic and instructional constraints to investigate the superposition of reasoning-driven generation strategies.
In this section, we show results on two sets of \textbf{\color{mygreen}forest}-\textbf{\color{myblue}snow} prompt pairs.
See \texttt{story\_\_forest\_snow} in \cref{tab:sop_prompts} for SoP and \texttt{story\_reasoning\_\_forest\_snow} in \cref{tab:sot_prompts} for SoT for details.

\textbf{Implementation.}\quad 
In the superposition-based unraveling process, each sentence undergoes not only denoising but also updates to the weight coefficients $p(i)$.
The temporal evolution of these coefficients can occur rapidly depending on the output of the trained model. In such cases, the effect of superposition may disappear early in the denoising process. (See Appendix~\ref{sec:superposition_weight_dynamics} for details.)
To mitigate this effect, we implement a regularization parametrized by $\beta \geq 0$ that controls degrees of mixtures.

\textbf{Results.}\quad 
\Cref{fig:forest_snow_similarity} provides a visual overview of the generated sentences' distribution. The scatter plots indicate that the \textbf{SoP}-generated sentences tend to occupy a region intermediate to the single-prompt outputs. This qualitative observation is quantitatively supported by the average similarity values presented in \Cref{tab:align}, which show a statistically significant increase in similarity to the non-selected condition compared to single-prompt generation.
To qualitatively assess the effect of \textbf{SoT}, we present a representative generated sample in \Cref{tab:generated-samples}. The generated text effectively synthesally fuses the divergent stylistic constraints of both "\textbf{\color{mygreen}forest}" and "\textbf{\color{myblue}snow}" prompts into a single narrative. This supports our hypothesis that the superposition of thought tokens allows the model to navigate between conflicting reasoning paths during the denoising process. See \cref{sec:superposition_prompts} for more details.

\begin{figure}[tbp]
    \centering
    \centering
    \includegraphics[
        bb=0 0 100 100,
        trim=0cm 0 0cm 0,
        clip,
        width=.9\columnwidth]{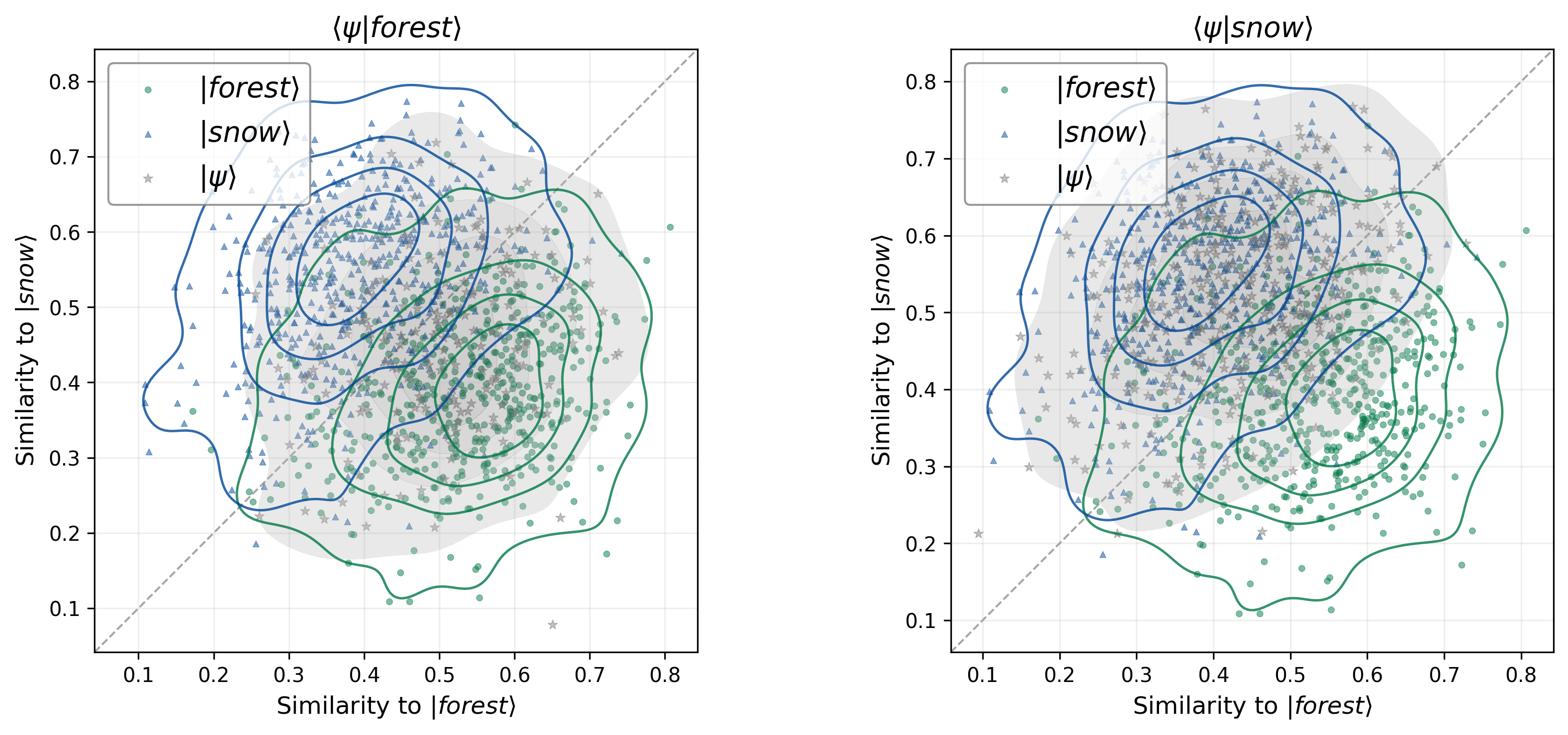}
    \caption{Cosine similarity of SEDD-small outputs to {\color{mygreen}forest-only (horizontal)} and {\color{myblue}snow-only (vertical)} reference texts, computed using MiniLM embeddings.
    {\color{mygreen}Green} and {\color{myblue}blue} show single-condition baselines; $\smash{\colorbox{black!10}{\text{gray}}}$ shows \textbf{SoP} outputs with {\color{mygreen}forest (left)} or {\color{myblue}snow (right)} selected at final measurement.
    Superposition increases average similarity to the unselected condition.
    }
\label{fig:forest_snow_similarity}
\end{figure}

\begin{table}[t]
\centering
\begin{tabular}{p{0.1\linewidth} p{0.82\linewidth}}
\hline
Prompt & 
Write a short story.
\\
\hline
Thoughts &
\begin{tabular}{p{0.39\textwidth}|p{0.39\textwidth}}
\noindent
\textbf{\color{mygreen}[forest]} The narrative must take place in a {\color{mygreen}dense, vibrant summer forest}. [...] The tone should feel {\color{mygreen}alive, earthy, and teeming with natural energy}.
&
\noindent
\textbf{\color{myblue}[snow]} The narrative must take place in a {\color{myblue}desolate, frozen winter tundra}. [...] The tone should feel {\color{myblue}silent, harsh, and starkly isolated}.
\end{tabular}
\\
\hline
Generated $\braket{\bpsi|\textbf{\color{mygreen}forest}}$ &
As the overcliff dawn drifts from \textbf{\color{mygreen}sprawling forest of trees north} – I stand on the edge of the \textbf{\color{myblue}frozen world}, adrift. [...] The air itself toys him away, like an elemental furnace or \textbf{\color{mygreen}pine-scented forest earth}. \textbf{\color{myblue}Snow} piling up over the world, carrying \textbf{\color{myblue}liquid silver}, and where the sun's frame unfurls a \textbf{\color{myblue}chill}; dimming.
\\
\hline
\end{tabular}
\caption{Representative generated samples by iLLaDA-8B-Instruct (\cite{Nie2026iLLaDA}) with \textbf{SoT}. Coloring in the generated texts is generated by gemma-4-26B-A4B (\cite{team2026gemma}).
}
\label{tab:generated-samples}
\end{table}

\subsection{Experiment 2: Quantum-Inspired Remasking}
\label{sec:experiment_remasking}

We evaluate the quantum-inspired remasking (QIR) method introduced in the previous section. We apply this method to pretrained diffusion language models without updating their weights. By varying when remasking begins during generation, we examine how its timing affects the quality of the generated text using several evaluation metrics.
See \cref{sec:remasking_experiment} for details of our experiments.

\textbf{Setup.}\quad 
We evaluate SEDD-medium \citep{Lou2023DiscreteDM}, RADD-medium \citep{Ou2024YourAD}, and iLLaDA-8B-Base \citep{Nie2026iLLaDA}. 
We generate 256 unconditional sequences per condition for SEDD-medium and RADD-medium, and 128 for iLLaDA-8B-Base, with 1024 tokens, 1024 sampling steps, and $\gamma=0.02$.
To reduce the cost of evaluating the reciprocal scores in \cref{eq:remask_score}, we uniformly sample up to $K=8$ non-mask positions per remasking calculation. Remasking begins after approximately 10\%, 50\%, or 90\% of sampling (\texttt{start}, \texttt{middle}, and \texttt{end}) and remains active thereafter. The baseline, \texttt{none}, uses the same stochastic-unraveling decoder without remasking. We verified that, without remasking, our implementation yields results consistent with those of the original SEDD Euler sampler under the same sampling settings.

\textbf{Metrics.}\quad 
We report generative perplexity (GenPPL), computed by evaluating the generated texts with GPT-2-large \citep{radford2019language}, token-frequency entropy, and a CoLA acceptability score \citep{morris2020textattack}. 
Entropy monitors excessive loss of token diversity; it is not a quantity to minimize.
Repetition is measured by $\Delta R_n$, the same-seed difference from no remasking in repeated token $n$-gram fractions, averaged over 128-token windows and expressed in percentage points (pp). 
We evaluate $n=2,\ldots,8$ and report $n=7$ as a representative example in Table~\ref{tab:results}. 
Qwen3-32B compares each sample with its same-checkpoint, same-seed baseline for fluency, coherence, progression, and repetition, without knowing the generation method. 
Wins, ties, and losses score 1, $1/2$, and 0, averaged over both presentation orders; preference above 50\% favors remasking. Definitions and evaluation details are given in Appendices~\ref{sec:metrics} and~\ref{sec:judge}.

\textbf{Results.}\quad 
Enabling QIR from the middle of generation yields Qwen preference scores of $62.0\%$ for SEDD-medium, $62.4\%$ for RADD-medium, and $61.8\%$ for iLLaDA-8B-Base against their no-remasking baselines (Table~\ref{tab:results}).
This schedule achieves the highest mean preference among the evaluated schedules for both models, demonstrating improved judge-assessed text quality without retraining.
Starting remasking earlier further lowers perplexity and raises CoLA scores, but reduces judge preference and token diversity while increasing repetition.
The increase in repetition under \texttt{start} is also observed for the other evaluated $n$-gram lengths. 
These results show that lower generative perplexity and higher grammatical acceptability do not necessarily translate into better overall text quality. Supplementary repetition results are provided in Appendix~\ref{sec:statistics}.

\begin{table}[tbp]\centering\scriptsize
    \resizebox{\linewidth}{!}{
    \begin{tabular}{llrrrr}\toprule Model & Metric & none & end & middle & start\\\midrule
    \multirow{5}{*}{\makecell{SEDD(medium) \\ \cite{Lou2023DiscreteDM}}} & GenPPL & $33.12 _{\pm 0.58}$ & $32.73 _{\pm 0.57}$ & $27.52 _{\pm 0.43}$ & $12.24 _{\pm 0.26}$\\
     & Entropy & $7.57 _{\pm 0.02}$ & $7.57 _{\pm 0.02}$ & $7.52 _{\pm 0.03}$ & $6.42 _{\pm 0.05}$\\
     & CoLA & $0.285 _{\pm 0.006}$ & $0.285 _{\pm 0.006}$ & $0.335 _{\pm 0.006}$ & $0.336 _{\pm 0.007}$\\
     & Qwen (\%) & -- & $51.5 _{\pm 0.8}$ & $\mathbf{62.0} _{\pm 1.3}$ & $53.3 _{\pm 1.9}$\\
     & $\Delta R_7$ (pp) & $0.000 _{\pm 0.000}$ & $-0.005 _{\pm 0.004}$ & $0.025 _{\pm 0.013}$ & $1.405 _{\pm 0.210}$\\
    \midrule
    \multirow{5}{*}{\makecell{RADD(medium) \\ \cite{Ou2024YourAD}}} & GenPPL  & $33.66 _{\pm 0.60}$ & $33.34 _{\pm 0.59}$ & $28.08 _{\pm 0.44}$ & $13.71 _{\pm 0.28}$\\
     & Entropy  & $7.59 _{\pm 0.02}$ & $7.58 _{\pm 0.02}$ & $7.53 _{\pm 0.02}$ & $6.55 _{\pm 0.05}$\\
     & CoLA  & $0.281 _{\pm 0.005}$ & $0.282 _{\pm 0.005}$ & $0.316 _{\pm 0.006}$ & $0.328 _{\pm 0.007}$\\
     & Qwen (\%)  & -- & $49.5 _{\pm 0.8}$ & $\mathbf{62.4} _{\pm 1.3}$ & $58.0 _{\pm 1.8}$\\
     & $\Delta R_7$ (pp) & $0.000 _{\pm 0.000}$ & $0.002 _{\pm 0.007}$ & $0.007 _{\pm 0.010}$ & $1.258 _{\pm 0.333}$\\
     \midrule
     \multirow{5}{*}{\makecell{iLLaDA-8B-Base \\ \cite{Nie2026iLLaDA}}} & GenPPL & $21.14 _{\pm 1.56}$ & $21.25 _{\pm 1.57}$ & $19.83 _{\pm 1.48}$ & $18.19 _{\pm 6.58}$\\
     & Entropy & $6.46 _{\pm 0.08}$ & $6.46 _{\pm 0.08}$ & $6.43 _{\pm 0.07}$ & $5.70 _{\pm 0.09}$\\
     & CoLA & $0.460 _{\pm 0.019}$ & $0.466 _{\pm 0.019}$ & $0.487 _{\pm 0.018}$ & $0.563 _{\pm 0.019}$\\
     & Qwen (\%) & -- & $50.0 _{\pm 1.4}$ & $\mathbf{61.8} _{\pm 1.9}$ & $55.7 _{\pm 2.5}$\\
     & $\Delta R_7$ (pp) & $0.000 _{\pm 0.000}$ & $-0.072 _{\pm 0.054}$ & $-0.081 _{\pm 0.113}$ & $4.042 _{\pm 0.531}$\\
    \bottomrule
    \end{tabular}
    }
    \caption{Results at $\gamma=0.02$. 
    Each condition contains 256 sequences for SEDD and RADD
    and 128 for iLLaDA.
    The no-remasking column is the shared control. All entries give mean $\pm$ standard error; Qwen preference scores are expressed as percentages, and $\Delta R_7$ in percentage points (pp). For $\Delta R_7$, the no-remasking baseline is zero by definition.
    }\label{tab:results}
\end{table}

\section{Conclusion and Outlook}
\label{sec:conclusion}

We embedded continuous-time diffusion language models into Lindblad dynamics, enabling sampling from superposed conditioning states. In the forest/snow experiment, superposition outputs show higher similarity to the unselected condition than single-condition baselines, suggesting that its influence persists after final measurement.

Within this framework, we introduced quantum-inspired remasking through perturbations of the jump operators. Experiments using pretrained models show that the timing of remasking affects generated text quality: earlier remasking lowers perplexity but can increase repetition, whereas an intermediate onset achieves the highest mean judge preference across the tested checkpoints. These results highlight the importance of evaluating text quality beyond perplexity alone.

Future work includes reducing the computational cost of remasking and extending its evaluation to conditional generation. More broadly, this framework opens a path toward quantum-native language models in which superposition and interference play an active role in generation. Such models could evolve multiple candidate sequences coherently, with interference shaping the probabilities of the resulting text. Exploring how these dynamics can be designed and learned may lead to new generative mechanisms beyond classical diffusion.

\section*{Acknowledgments}
We thank K. Fukushima, K. Hagihara, and K. Shiozaki for helpful discussions.
Some of the results presented in this paper were obtained using the ABCI-Q system provided by G-QuAT at AIST.
We also sincerely thank RIKEN AIP and iTHEMS for providing the computer systems.
A.T. acknowledges support from JSPS KAKENHI Grant Numbers
JP22H05111 and JP22H05116.
Y.A. acknowledges support from JSPS KAKENHI Grant Numbers
JP23H01174 and JP23K25870.
Y.H. acknowledges support from JSPS KAKENHI Grant Numbers
JP22H05111, JP22H05118, and JP24K23186, and JST PRESTO Grant Number
JPMJPR24K8.

\section*{AI Use Statement}
Generative AI tools were used for brainstorming, improving the
readability and organization of the manuscript, and assisting with
software implementation. Suggestions produced by these tools were
independently examined and verified by the authors. Generative AI tools
were not used as an autonomous source of scientific claims, proofs, or
result interpretation. The authors take full responsibility for the
final content of this work.

\bibliography{ref}

@string{icml2015 = "Proceedings of the 32nd International Conference on Machine Learning"}

@string{icml2024 = "Proceedings of the 41st International Conference on Machine Learning"}

@string{icml2025 = "Proceedings of the 42nd International Conference on Machine Learning"}

@string{neurips2019 = "Advances in Neural Information Processing Systems 32"}

@string{neurips2020 = "Advances in Neural Information Processing Systems 33"}

@string{neurips2021 = "Advances in Neural Information Processing Systems 34"}

@string{neurips2022 = "Advances in Neural Information Processing Systems 35"}

@string{neurips2025 = "Advances in Neural Information Processing Systems 38"}

@string{iclr2021 = "Proceedings of the 9th International Conference on Learning Representations"}

@string{iclr2023 = "Proceedings of the 11th International Conference on Learning Representations"}

@string{iclr2025 = "Proceedings of the 13th International Conference on Learning Representations"}

@string{cvpr2021 = "Proceedings of the IEEE/CVF Conference on Computer Vision and Pattern Recognition (CVPR)"}

@string{tmlr = "Transactions on Machine Learning Research"}

@string{prx = "Physical Review X"}

@string{pre = "Physical Review E"}

@inproceedings{Sohl2015DeepUL,
  author    = {Sohl-Dickstein, Jascha and Weiss, Eric and Maheswaranathan, Niru and Ganguli, Surya},
  title     = {Deep Unsupervised Learning using Nonequilibrium Thermodynamics},
  booktitle = icml2015,  
  pages     = {2256--2265},
  year      = {2015},
  note      = {\url{https://proceedings.mlr.press/v37/sohl-dickstein15.html}}
}

@inproceedings{Ho2020DenoisingDP,
 author = {Ho, Jonathan and Jain, Ajay and Abbeel, Pieter},
 title     = {Deep Unsupervised Learning using Nonequilibrium Thermodynamics},
 booktitle = neurips2020,
 pages = {6840--6851},
 year = {2020},
 note      = {\url{https://proceedings.neurips.cc/paper/2020/hash/4c5bcfec8584af0d967f1ab10179ca4b-Abstract.html}}
}

@inproceedings{Song2019GenerativeMB,
  title={Generative Modeling by Estimating Gradients of the Data Distribution},
  author={Yang Song and Stefano Ermon},
  booktitle = neurips2019,
  year={2019},
  note = {\url{https://papers.nips.cc/paper_files/paper/2019/hash/3001ef257407d5a371a96dcd947c7d93-Abstract.html}}
}

@inproceedings{Song2020ScoreBasedGM,
  author={Yang Song and Jascha Narain Sohl-Dickstein and Diederik P. Kingma and Abhishek Kumar and Stefano Ermon and Ben Poole},
  title={Score-Based Generative Modeling through Stochastic Differential Equations},
  booktitle = iclr2021,
  year = {2021},
  note      = {\url{https://openreview.net/forum?id=PxTIG12RRHS}}
}

@inproceedings{Dhariwal2021DiffusionMB,
  title={Diffusion Models Beat GANs on Image Synthesis},
  author={Prafulla Dhariwal and Alex Nichol},
  booktitle = neurips2021,
  pages = {8780--8794},
  year={2021},
  note = {\url{https://papers.nips.cc/paper/2021/hash/49ad23d1ec9fa4bd8d77d02681df5cfa-Abstract.html}}
}

@inproceedings{Rombach2021HighResolutionIS,
  title={High-Resolution Image Synthesis with Latent Diffusion Models},
  author={Robin Rombach and A. Blattmann and Dominik Lorenz and Patrick Esser and Bj{\"o}rn Ommer},
  booktitle = cvpr2021,
  pages     = {10684-10695},
  year={2022},
  note = {\url{https://openaccess.thecvf.com/content/CVPR2022/html/Rombach_High-Resolution_Image_Synthesis_With_Latent_Diffusion_Models_CVPR_2022_paper.html}}
}

@inproceedings{Ho2022VideoDM,
  title={Video Diffusion Models},
  author={Jonathan Ho and Tim Salimans and Alexey Gritsenko and William Chan and Mohammad Norouzi and David J. Fleet},
  booktitle = neurips2022,
  pages = {8633--8646},
  year={2022},
  note = {\url{https://proceedings.neurips.cc/paper_files/paper/2022/hash/39235c56aef13fb05a6adc95eb9d8d66-Abstract-Conference.html}}
}

@inproceedings{Singer2022MakeAVideoTG,
  title={Make-A-Video: Text-to-Video Generation without Text-Video Data},
  author={Uriel Singer and Adam Polyak and Thomas Hayes and Xiaoyue Yin and Jie An and Songyang Zhang and Qiyuan Hu and Harry Yang and Oron Ashual and Oran Gafni and Devi Parikh and Sonal Gupta and Yaniv Taigman},
  booktitle = iclr2023,
  year={2023},
  note = {\url{https://openreview.net/forum?id=nJfylDvgzlq}}
}

@inproceedings{Lou2023DiscreteDM,
  author={Aaron Lou and Chenlin Meng and Stefano Ermon},
  title={Discrete Diffusion Modeling by Estimating the Ratios of the Data Distribution},
  booktitle = icml2024,
  pages     = {32819--32848},
  year={2024},
  note      = {\url{https://proceedings.mlr.press/v235/lou24a.html}}
}

@inproceedings{Sun2022ScorebasedCD,
  author={Haoran Sun and Lijun Yu and Bo Dai and Dale Schuurmans and Hanjun Dai},
  title={Score-based Continuous-time Discrete Diffusion Models},
  booktitle = iclr2023,
  year={2023},
  note      = {\url{https://openreview.net/forum?id=BYWWwSY2G5s}}
}

@inproceedings{Sahoo2025TheDD,
  title = 	 {The Diffusion Duality},
  author =       {Sahoo, Subham Sekhar and Deschenaux, Justin and Gokaslan, Aaron and Wang, Guanghan and Chiu, Justin T and Kuleshov, Volodymyr},
  booktitle = icml2025,
  pages={52584--52619},
  year={2025},
  note      = {\url{https://proceedings.mlr.press/v267/sahoo25a.html}}
}

@inproceedings{Xu2024EnergyBasedDL,
  title={Energy-Based Diffusion Language Models for Text Generation},
  author={Minkai Xu and Tomas Geffner and Karsten Kreis and Weili Nie and Yilun Xu and Jure Leskovec and Stefano Ermon and Arash Vahdat},
  booktitle = iclr2025,
  year={2025},
  note      = {\url{https://openreview.net/forum?id=sL2F9YCMXf}}
}

@inproceedings{Ou2024YourAD,
  title={Your Absorbing Discrete Diffusion Secretly Models the Conditional Distributions of Clean Data},
  author={Jingyang Ou and Shen Nie and Kaiwen Xue and Fengqi Zhu and Jiacheng Sun and Zhenguo Li and Chongxuan Li},
  booktitle = iclr2025,
  year={2025},
  note      = {\url{https://openreview.net/forum?id=sMyXP8Tanm}}
}

@inproceedings{Nie2025LargeLD,
  title={Large Language Diffusion Models},
  author={Shen Nie and Fengqi Zhu and Zebin You and Xiaolu Zhang and Jingyang Ou and Jun Hu and Jun Zhou and Yankai Lin and Jirong Wen and Chongxuan Li},
  booktitle = neurips2025,
  year={2025},
  note      = {\url{https://openreview.net/forum?id=KnqiC0znVF}}
}

@inproceedings{Zhao2025d1SR,
  title={d1: Scaling Reasoning in Diffusion Large Language Models via Reinforcement Learning},
  author={Siyan Zhao and Devaansh Gupta and Qinqing Zheng and Aditya Grover},
  booktitle = neurips2025,
  year={2025},
  note      = {\url{https://openreview.net/forum?id=7ZVRlBFuEv}}
}

@inproceedings{Campbell2022ACT,
  title={A Continuous Time Framework for Discrete Denoising Models},
  author={Andrew Campbell and Joe Benton and Valentin De Bortoli and Tom Rainforth and George Deligiannidis and A. Doucet},
  booktitle = neurips2022,
  pages = {28266--28279},
  year={2022},
  note      = {\url{https://proceedings.neurips.cc/paper_files/paper/2022/hash/b5b528767aa35f5b1a60fe0aaeca0563-Abstract-Conference.html}}
}

@inproceedings{Hirono2024UnderstandingDM,
  title={Understanding Diffusion Models by Feynman's Path Integral},
  author={Yuji Hirono and Akinori Tanaka and Kenji Fukushima},
  booktitle = icml2024,
  pages = 	 {18324--18351},
  year={2024},
  note      = {\url{https://proceedings.mlr.press/v235/hirono24a.html}}
}

@article{Yu2024NonequilbriumPO,
  title={Nonequilbrium physics of generative diffusion models},
  author={Zhendong Yu and Haiping Huang},
  journal=pre,
  year={2024},
  note      = {\url{https://doi.org/10.1103/PhysRevE.111.014111}}
}

@article{Yaguchi2025TheGO,
  title={The Geometry of Phase Transitions in Diffusion Models: Tubular Neighbourhoods and Singularities},
  author={Manato Yaguchi and Kotaro Sakamoto and Ryosuke Sakamoto and Masato Tanabe and Masatomo Akagawa and Yusuke Hayashi and Masahiro Suzuki and Yutaka Matsuo},
  journal=tmlr,
  year={2025},
  note      = {\url{https://openreview.net/forum?id=ahVFKFLYk2}}
}

@article{Ikeda2024SpeedaccuracyTF,
  title={Speed-accuracy trade-off for the diffusion models: Wisdom from nonequilibrium thermodynamics and optimal transport},
  author={Kotaro Ikeda and Tomoya Uda and Daisuke Okanohara and Sosuke Ito},
  journal=prx,
  year={2024},
  note      = {\url{https://doi.org/10.1103/x5vj-8jq9}}
}

@inproceedings{Hoogeboom2021ArgmaxFA,
  title={Argmax Flows and Multinomial Diffusion: Learning Categorical Distributions},
  author={Emiel Hoogeboom and Didrik Nielsen and Priyank Jaini and Patrick Forr'e and Max Welling},
  booktitle=neurips2021,
  pages = {12454--12465},
  year={2021},
  note      = {\url{https://proceedings.neurips.cc/paper/2021/hash/67d96d458abdef21792e6d8e590244e7-Abstract.html}}
}

@inproceedings{Austin2021StructuredDD,
  title={Structured Denoising Diffusion Models in Discrete State-Spaces},
  author={Jacob Austin and Daniel D. Johnson and Jonathan Ho and Daniel Tarlow and Rianne van den Berg},
  booktitle=neurips2021,
  pages = {17981--17993},
  year={2021},
  note      = {\url{https://proceedings.neurips.cc/paper/2021/hash/958c530554f78bcd8e97125b70e6973d-Abstract.html}}
}

@article{Khanna2025MercuryUL,
  title={Mercury: Ultra-Fast Language Models Based on Diffusion},
  author={Samar Khanna and Siddhant Kharbanda and Shufan Li and Harshit Varma and Eric Wang and Sawyer Birnbaum and Ziyang Luo and Yanis Miraoui and Akash Palrecha and Stefano Ermon and Aditya Grover and Volodymyr Kuleshov},
  journal={ArXiv},
  year={2025},
  volume={abs/2506.17298},
  note      = {\url{https://api.semanticscholar.org/CorpusID:280000358}}
}

@article{Song2025SeedDA,
  title={Seed Diffusion: A Large-Scale Diffusion Language Model with High-Speed Inference},
  author={Yuxuan Song and Zheng Zhang and Cheng Luo and Peng Gao and Fan Xia and Hao Luo and Zheng Li and Yuehang Yang and Hongli Yu and Xingwei Qu and Yuwei Fu and Jing Su and Ge Zhang and Wenhao Huang and Mingxuan Wang and Lin Yan and Xiaoying Jia and Jingjing Liu and Weizhi Ma and Yaqing Zhang and Yong-Xu Wu and Hao-Min Zhou},
  journal={ArXiv},
  year={2025},
  volume={abs/2508.02193},
  note      = {\url{https://api.semanticscholar.org/CorpusID:280421264}}
}

@article{Kwon2018FluctuationTF,
  title = {Fluctuation Theorems for a Quantum Channel},
  author = {Hyukjoon Kwon and Myungshik S. Kim},
  journal = prx,
  year = {2018},
  note={\url{https://api.semanticscholar.org/CorpusID:201317121}}
}

@article{Manzano2019ASI,
  title={A short introduction to the Lindblad master equation},
  author={Daniel Manzano},
  journal={AIP Advances},
  year={2019},
  url={https://api.semanticscholar.org/CorpusID:184487806}
}

@book{Breuer2002,
  author    = {H.-P. Breuer and F. Petruccione},
  title     = {The Theory of Open Quantum Systems},
  publisher = {Oxford University Press},
  year      = {2002},
  doi       = {10.1093/acprof:oso/9780199213900.001.0001}
}

@article{Lindblad1976,
  author  = {G{\"o}ran Lindblad},
  title   = {On the Generators of Quantum Dynamical Semigroups},
  journal = {Communications in Mathematical Physics},
  volume  = {48},
  number  = {2},
  pages   = {119--130},
  year    = {1976},
  doi     = {10.1007/BF01608499}
}

@article{Gorini1976,
  author  = {V. Gorini and A. Kossakowski and E. C. G. Sudarshan},
  title   = {Completely Positive Dynamical Semigroups of $N$-Level Systems},
  journal = {Journal of Mathematical Physics},
  volume  = {17},
  number  = {5},
  pages   = {821--825},
  year    = {1976},
  doi     = {10.1063/1.522979}
}

@article{Dalibard1992,
  author  = {Jean Dalibard and Yvan Castin and Klaus M{\o}lmer},
  title   = {Wave-Function Approach to Dissipative Processes in Quantum Optics},
  journal = {Physical Review Letters},
  volume  = {68},
  number  = {5},
  pages   = {580--583},
  year    = {1992},
  doi     = {10.1103/PhysRevLett.68.580}
}

@article{Dum1992,
  author  = {R. Dum and P. Zoller and H. Ritsch},
  title   = {Monte Carlo Simulation of the Atomic Master Equation for Spontaneous Emission},
  journal = {Physical Review A},
  volume  = {45},
  number  = {7},
  pages   = {4879--4887},
  year    = {1992},
  doi     = {10.1103/PhysRevA.45.4879}
}

@book{Carmichael1993,
  author    = {H. J. Carmichael},
  title     = {An Open Systems Approach to Quantum Optics},
  publisher = {Springer},
  year      = {1993},
  doi       = {10.1007/978-3-540-47650-7}
}

@article{Amin2018QBM,
  author  = {Mohammad H. Amin and Evgeny Andriyash and Jason Rolfe and Bohdan Kulchytskyy and Roger Melko},
  title   = {Quantum Boltzmann Machine},
  journal = {Physical Review X},
  volume  = {8},
  number  = {2},
  pages   = {021050},
  year    = {2018},
  doi     = {10.1103/PhysRevX.8.021050}
}

@article{Biamonte2017QML,
  author  = {Jacob Biamonte and Peter Wittek and Nicola Pancotti and Patrick Rebentrost and Nathan Wiebe and Seth Lloyd},
  title   = {Quantum Machine Learning},
  journal = {Nature},
  volume  = {549},
  pages   = {195--202},
  year    = {2017},
  doi     = {10.1038/nature23474}
}

@article{Benedetti2019VQC,
  author  = {Marcello Benedetti and Erika Lloyd and Stefan Sack and Mattia Fiorentini},
  title   = {Parameterized Quantum Circuits as Machine Learning Models},
  journal = {Quantum Science and Technology},
  volume  = {4},
  number  = {4},
  pages   = {043001},
  year    = {2019},
  doi     = {10.1088/2058-9565/ab4eb5}
}

@book{FeynmanHibbs1965,
  author    = {Richard P. Feynman and A. R. Hibbs},
  title     = {Quantum Mechanics and Path Integrals},
  publisher = {McGraw-Hill},
  year      = {1965}
}

@article{radford2019language,
  title={Language models are unsupervised multitask learners},
  author={Radford, Alec and Wu, Jeffrey and Child, Rewon and Luan, David and Amodei, Dario and Sutskever, Ilya and others},
  journal={OpenAI blog},
  volume={1},
  number={8},
  pages={9},
  year={2019}
}

@inproceedings{morris2020textattack,
  title={TextAttack: A Framework for Adversarial Attacks, Data Augmentation, and Adversarial Training in NLP},
  author={Morris, John and Lifland, Eli and Yoo, Jin Yong and Grigsby, Jake and Jin, Di and Qi, Yanjun},
  booktitle={Proceedings of the 2020 Conference on Empirical Methods in Natural Language Processing: System Demonstrations},
  pages={119--126},
  year={2020}
}

@article{team2026diffusiongemma,
  title={DiffusionGemma Technical Report},
  author={{DiffusionGemma Team} and Ta{\"\i}ga, Adrien Ali and Assiene, James and Calandriello, Daniele and Chaabouni, Rahma and Gante, Jo{\~a}o and von Glehn, Tamara and Keating, Nate and Knutsen, Chris and Kukla, Martin and others},
  journal={arXiv preprint arXiv:2608.00146},
  year={2026},
  note      = {\url{https://arxiv.org/abs/2608.00146}}
}

@article{team2026gemma,
  title={Gemma 4 technical report},
  author={{Gemma Team} and Abd, Sherif El and Aggarwal, Vaibhav and Algayres, Robin and Andreev, Alek and Bachem, Olivier and Ballantyne, Ian and Brick, Cormac and C{\u{a}}rbune, Victor and Casbon, Michelle and others},
  journal={arXiv preprint arXiv:2607.02770},
  year={2026},
  note      = {\url{https://arxiv.org/abs/2607.02770}}
}

@article{wang2026remasking,
  title={Remasking discrete diffusion models with inference-time scaling},
  author={Wang, Guanghan and Schiff, Yair and Sahoo, Subham and Kuleshov, Volodymyr},
  journal={Advances in Neural Information Processing Systems},
  volume={38},
  pages={147282--147339},
  year={2026}
}

@inproceedings{hong2026wide,
  title={Wide-in, narrow-out: Revokable decoding for efficient and effective dllms},
  author={Hong, Feng and Yu, Geng and Ye, Yushi and Huang, Haicheng and Zheng, Huangjie and Zhang, Ya and Wang, Yanfeng and Yao, Jiangchao},
  booktitle={International Conference on Learning Representations},
  volume={2026},
  pages={40132--40150},
  year={2026}
}

@inproceedings{song2021denoising,
title={Denoising Diffusion Implicit Models},
author={Jiaming Song and Chenlin Meng and Stefano Ermon},
booktitle={International Conference on Learning Representations},
year={2021},
url={https://openreview.net/forum?id=St1giarCHLP}
}

@misc{Nie2026iLLaDA,
  title = {Improved Large Language Diffusion Models},
  author = {Shen Nie and Qiyang Min and Shaoxuan Xu and Zihao Huang and Yuxuan Song and Yong Shan and Yankai Lin and Wayne Xin Zhao and Chongxuan Li and Ji-Rong Wen},
  year = {2026},
  eprint = {2606.25331},
  archivePrefix = {arXiv},
  primaryClass = {cs.CL},
  url = {https://arxiv.org/abs/2606.25331}
}

@misc{Nasu2025QuantumReversibility,
  title = {Quantum Reversibility Meets Classical Reverse Diffusion},
  author = {Nasu, Ryota and Tanaka, Gota and Tsuchiya, Asato},
  year = {2025},
  eprint = {2510.18512},
  archivePrefix = {arXiv},
  primaryClass = {quant-ph},
  url = {https://arxiv.org/abs/2510.18512}
}
\bibliographystyle{iclr2027/iclr2027_conference}

\newpage

\makeatletter
\makeatother
\appendix

\startcontents[appendix]

\printcontents[appendix]{l}{1}{\setcounter{tocdepth}{2}}

\newpage

\section{Notations}
\label{sec:notation}

This section summarizes the notation used throughout this paper.

\paragraph{Diffusion Model}

In discrete language diffusion models, we use the following notation:
\begin{itemize}
    \item Time-forward quantities are denoted without tildes ($a,b,c,\dots$), while time-backward quantities are denoted with tildes ($\tilde{a},\tilde{b},\tilde{c},\dots$).
    
    \item $\vocab = \{ m \mid m = 0,1,\dots,V \}$ denotes the vocabulary set.
    
    \item $\vocab^\NumToken$ denotes the set of all possible token sequences of fixed length $\NumToken$:
    \[
    \vocab^\NumToken
    =
    \left\{
    \m = [m_1,m_2,\dots,m_\NumToken]
    \,\middle|\,
    m_p \in \vocab
    \text{ for }
    p=1,2,\dots,\NumToken
    \right\}.
    \]

    \item Each element $\m \in \vocab^\NumToken$ corresponds to a discrete sentence or token sequence.

    \item For masked diffusion, we assign token ID $0$ to the mask token \texttt{[m]}.
\end{itemize}

\paragraph{Mathematical Notation}

We also introduce the following shorthand:
\begin{itemize}
    \item The Kronecker delta (indicator function) is defined as
    \[
    1_{mn}
    =
    \begin{cases}
    1 & m=n, \\
    0 & \text{otherwise},
    \end{cases}
    \qquad m,n \in \vocab.
    \]
    
    \item The delta function excluding the $p$-th component is
    \[
    1_{\m_{-p} \n_{-p}}
    =
    \prod_{\substack{q = 1 \ q \neq p}}^{\NumToken}
    1_{m_q n_q},
    \qquad
    \m,\n \in \vocab^\NumToken.
    \]
    
    \item $\n_{p \to m}$ denotes the sequence obtained by replacing the $p$-th token of $\n$ with $m$:
    \[
    \n_{p \to m}
    =
    [n_1,\dots,n_{p-1},m,n_{p+1},\dots,n_\NumToken].
    \]
\end{itemize}

\paragraph{Open Quantum System Formulation}

For the quantum-mechanical formulation, we use:
\begin{itemize}
    \item Operators $\op{A}, \op{B}, \op{C}, \dots$, with backward-time operators denoted by $\tilde{A}, \tilde{B}, \tilde{C}, \dots$.
    
    \item The single-token Hilbert space
    \[
    \mc{H}_\vocab
    =
    \left\{
    \sum_{m \in \vocab} c_m \ket{m}
    \,\middle|\,
    c_m \in \mathbb{C}
    \right\}.
    \]
    
    \item The Hilbert space for token sequences of fixed length $\NumToken$
    \[
    \mc{H}
    =
    \left\{
    \sum_{\m \in \vocab^\NumToken} c_\m \ket{\m}
    \,\middle|\,
    c_\m \in \mathbb{C}
    \right\}
    =
    \mc{H}_\vocab^{\otimes \NumToken}.
    \]
\end{itemize}
\section{A Brief Review of Diffusion Language Models}
\label{sec:dlm_intro}

This section provides a self-contained review of the continuous-time formulation of diffusion language models used in \cref{sec:dlm_as_oqs}. 
We first recall finite-state Markov dynamics, its time reversal, and the transition kernels induced by infinitesimal generators. 
We then specialize this framework to the factorized transition structure used in language diffusion models and explain how this structure motivates $\tau$-leaping sampling.

\subsection{Diffusion Process}

We consider probability distributions $p(\m)$ on the finite state space $\vocab^\NumToken$.
\begin{mydefinition}{Markov Kernel}{mk}
    A Markov kernel on $\vocab^\NumToken$ is a matrix
    $T(\m,\n)$ satisfying
    \[
        T(\m,\n) \geq 0,
        \qquad
        \sum_{\m \in \vocab^\NumToken} T(\m,\n)=1
        \quad \text{for every } \n \in \vocab^\NumToken.
    \]
    It acts on a probability distribution $p$ as
    \[
        (Tp)(\m)
        =
        \sum_{\n \in \vocab^\NumToken}
        T(\m,\n)p(\n),
    \]
    yielding another probability distribution.
\end{mydefinition}
A continuous-time diffusion process on $\vocab^\NumToken$
is described by
\[
    \frac{d}{dt}p_t = L_t p_t,
\]
where $L_t$ is a time-dependent linear operator on $\mathbb{R}^{|\vocab^\NumToken|}$.
The evolution must preserve non-negativity and normalization for every initial probability distribution.
Equivalently, for any $s \leq t$, the solution can be written as $p_t=T_{t|s}p_s$, where $T_{t|s}$ is a Markov kernel.

\begin{mytheorem}{Diffusion Process}{}
    Consider the linear evolution equation
    \begin{align}
        \frac{d}{dt}p_t(\m)
        &=
        (L_t p_t)(\m)
        =
        \sum_{\n \in \vocab^\NumToken}
            Q_t(\m,\n)p_t(\n)
        . \label{eq:diff_eq_appdx}
    \end{align}
    Suppose that $Q_t$ is continuous in $t$ and that the evolution from any time $s$ to any later time $t$ is represented by a Markov kernel $T_{t|s}$ as
    defined in \cref{def:mk}.
    Then
    \begin{align}
        Q_t(\m,\n) \geq 0 \quad (\m \neq \n),
        \qquad
        \sum_{\m \in \vocab^\NumToken} Q_t(\m,\n) = 0
        . \label{eq:q_conds_appdx}
    \end{align}
\end{mytheorem}

\begin{proof}
    Fix a time $t$ and a state $\k$, and take
    $p_t(\m)=1_{\m\k}$ as the initial distribution.
    For a small time increment $dt>0$, the evolution equation gives
    \begin{align}
        (T_{t+dt|t}p_t)(\m)
        =
        1_{\m\k}+dtQ_t(\m,\k)+o(dt).
    \end{align}
    Since $T_{t+dt|t}$ is a Markov kernel, the left-hand side is nonnegative. For $\m\neq\k$, this implies
    \[
        0 \leq dtQ_t(\m,\k)+o(dt).
    \]
    Dividing by $dt$ and taking $dt\to0$ yields $Q_t(\m,\k)\geq0$.

    Similarly, preservation of normalization gives
    \begin{align}
        1
        &=
        \sum_{\m\in\vocab^\NumToken}
            (T_{t+dt|t}p_t)(\m)
        =
        1+dt\sum_{\m\in\vocab^\NumToken}Q_t(\m,\k)+o(dt).
    \end{align}
    Subtracting $1$, dividing by $dt$, and taking $dt\to0$ yields
    $\sum_{\m}Q_t(\m,\k)=0$.
    Since $\k$ was arbitrary, both conditions follow.
\end{proof}
We now consider the time reversal of the diffusion process in \cref{eq:diff_eq_appdx}.
Reversing the direction of time gives
\begin{align}
  -\frac{d}{dt}p_t(\m)
  &=
  -\sum_{\n \in \vocab^\NumToken}
    Q_t(\m,\n)p_t(\n).
\end{align}
Although $-Q_t$ reproduces this evolution of $p_t$, it is not generally a valid Markov generator: its off-diagonal entries are negative wherever those of $Q_t$ are positive.
A valid reverse-time generator is instead obtained by reversing each transition and weighting its rate by a probability ratio, as follows.
\begin{mytheorem}{Time-Reversed Diffusion Process}{reverse_process}
  Let $p_t$ solve \cref{eq:diff_eq_appdx}, and suppose that $p_t(\n)>0$ for every $\n$ on the time interval considered.
  Define
  \begin{align}
    \tilde{Q}_t(\m,\n)
    &=
    Q_t(\n,\m)
    \frac{p_t(\m)}{p_t(\n)}
    \quad (\m\neq\n),
    \qquad
    \tilde{Q}_t(\m,\m)
    =
    -\sum_{\n\neq\m}\tilde{Q}_t(\n,\m).
    \label{eq:rev_Q_appdx}
  \end{align}
  Then $\tilde{Q}_t$ satisfies the Markov generator conditions in \cref{eq:q_conds_appdx}, and
  \begin{align}
    -\frac{d}{dt}p_t(\m)
    &=
    \sum_{\n\in\vocab^\NumToken}
      \tilde{Q}_t(\m,\n)p_t(\n).
    \label{eq:rev_diff_eq_appdx}
  \end{align}
  Thus, the reverse process follows the same family of distributions $p_t$ in the direction of decreasing $t$.
\end{mytheorem}
\begin{proof}
  We rewrite the right-hand side of the original equation:
  \begin{align}
    -\sum_{\n} Q_t(\m,\n)\, p_t(\n)
    &=
    - Q_t(\m,\m)p_t(\m)
    -
    \sum_{\n\neq \m} Q_t(\m,\n)p_t(\n).
  \end{align}
  Using the column-sum condition
  \begin{align}
    \sum_{\n} Q_t(\n,\m)=0,
  \end{align}
  we have
  \begin{align}
    Q_t(\m,\m)
    =
    -\sum_{\n\neq \m} Q_t(\n,\m),
  \end{align}
  and therefore
  \begin{align}
    - Q_t(\m,\m)p_t(\m)
    =
    \left(\sum_{\n\neq \m} Q_t(\n,\m)\right)p_t(\m).
  \end{align}
  Hence,
  \begin{align}
    -\sum_{\n} Q_t(\m,\n)\, p_t(\n)
    &=
    \sum_{\n\neq \m} Q_t(\n,\m)p_t(\m)
    -
    \sum_{\n\neq \m} Q_t(\m,\n)p_t(\n).
  \end{align}
  The first term can be rewritten as
  \begin{align}
    Q_t(\n,\m)p_t(\m)
    =
    \left(
      Q_t(\n,\m)\frac{p_t(\m)}{p_t(\n)}
    \right)p_t(\n),
  \end{align}
  so that
  \begin{align}
    -\sum_{\n} Q_t(\m,\n)\, p_t(\n)
    &=
    \sum_{\n\neq \m}
    \left[
      Q_t(\n,\m)\frac{p_t(\m)}{p_t(\n)}
    \right] p_t(\n)
    -
    \sum_{\n\neq \m} Q_t(\m,\n)p_t(\n).
  \end{align}
  Now we define, for $\m\neq \n$,
  \begin{align}
    \tilde Q_t(\m,\n)
    :=
    Q_t(\n,\m)\frac{p_t(\m)}{p_t(\n)},
  \end{align}
  and set the diagonal part so that the column sum vanishes:
  \begin{align}
    \tilde Q_t(\m,\m)
    :=
    -\sum_{\k\neq \m} \tilde Q_t(\k,\m).
  \end{align}
  Then the second term above is absorbed into the diagonal contribution, and we obtain
  \begin{align}
    -\sum_{\n} Q_t(\m,\n)\, p_t(\n)
    =
    \sum_{\n} \tilde Q_t(\m,\n)\, p_t(\n).
  \end{align}
  This proves \cref{eq:rev_diff_eq_appdx}.
\end{proof}

\subsection{Sampling}

\subsubsection{Sampling 1-token per step}

In practice, directly evolving the full distribution $p_t$ is computationally intractable because the state space grows exponentially with the sequence length.
Instead, we simulate individual sample trajectories using transition probabilities derived from the forward or reverse evolution equations in \cref{eq:diff_eq_appdx,eq:rev_diff_eq_appdx}.
For example, to first order in $dt$, \cref{eq:diff_eq_appdx} yields
\begin{align}
  p_{t+dt}(\m)
  &=
  p_t(\m)
  +
  dt 
  \sum_{\n \in \vocab^\NumToken} 
    Q_t(\m, \n) p_t(\n)
  =
  \sum_{\n \in \vocab^\NumToken} 
    \left(
      \underbrace{
        1_{\m \n}
        +
        Q_t(\m, \n)
        dt
      }_{
        p_{t+dt|t}(\m | \n)
      }
    \right)
    p_t(\n).
\end{align}
Similarly, to first order in $dt$, the reverse-time evolution in \cref{eq:rev_diff_eq_appdx} yields
\begin{align}
  p_{t-dt}(\m)
  &=
  p_t(\m)
  +
  dt 
  \sum_{\n \in \vocab^\NumToken} 
    \tilde{Q}_t(\m, \n) p_t(\n)
  =
  \sum_{\n \in \vocab^\NumToken} 
    \left(
      \underbrace{
        1_{\m\n}
        +
        \tilde{Q}_t(\m, \n)
        dt
      }_{
        \tilde{p}_{t-dt|t}(\m | \n)
      }
    \right)
    p_t(\n)
    .
\end{align}
In practice, samples are generated using the learned reverse-time generator $\tilde{Q}_t^\text{model}$, with transition probabilities given by
\begin{align}
  \tilde{p}_{t-dt|t}^\text{model}(\m | \n)
  &=
  1_{\m\n}
  +
  \tilde{Q}_t^\text{model}(\m, \n)
  dt
  .
\end{align}

\subsubsection{Forward solution}

The sampling schemes above are accurate to first order in $dt$.
However, they are inefficient for language generation because the generator permits only single-token transitions:
\begin{align}
  Q_t(\m, \n)
  &=
  \sigma(t)
  \sum_{p=1}^{\NumToken}
  1_{\m_{-p}\n_{-p}} 
  Q^\text{tok}({m}_p, {n}_p) 
  ,
  \quad 
  1_{\m_{-p}\n_{-p}} 
  =
  \left\{
      \begin{array}{ll}
          1 & \m = \n \text{ except for ${m}_p, {n}_p$, }
          \\
          0 & \text{otherwise,}
      \end{array}
  \right.
  . \label{eq:delta_p_apdx}
\end{align}
Consequently, each update changes at most one token, requiring many sequential steps to generate a sequence.
To develop more efficient sampling schemes, we first derive the solution of the forward diffusion equation.
\begin{mytheorem}{}{integral2}
  With the diffusion matrix in \cref{eq:delta_p_apdx}, we get the formal solution of the diffusion equation
  \begin{align}
    p_{t|s}(\m|\n)
    &=
    \prod_{p=1}^N 
      p_{t|s}(m_p|n_p)
    ,\label{eq:solution-appdx}
  \end{align}
  where 
  \begin{align}
    p_{t|s}(m_p|n_p)
    &=
    \exp 
    \left[
        (\bar{\sigma}(t) - \bar{\sigma}(s))  Q^\text{tok}
      \right]
      (m_p, n_p)
    ,
    \quad 
    \bar{\sigma}(t) = \int_0^t \sigma(\tau) d\tau
    .
  \end{align}
\end{mytheorem}

\begin{proof}
  The decomposition of $Q_t(\m, \n)$ in \cref{eq:delta_p_apdx} means it defines matrix 
  \begin{align}
    Q_t = \sigma(t) \sum_{p=1}^N I \otimes \cdots \overset{\text{p-th}}{Q^\text{tok}} \cdots \otimes I
    .
  \end{align}
  Because the time dependency is only on the scaler function $\sigma(t)$, two matrices $Q_{t}$ and $Q_{s}$ commutes for any $t, s$, which means we can solve the ODE by taking exponential:
  \begin{align}
    p_{t|s}(\m|\n)
    &=
    \left[
      \exp \left(
        \int_s^t Q_\tau
        d\tau
      \right)
    \right]
    (\m, \n)
    \notag \\
    &=
    \left[
      \exp \left(
        (\underbrace{
          \bar{\sigma}(t) - \bar{\sigma}(s)
        }_{=: \Delta \bar{\sigma}(t|s)}
          )
        \sum_{p=1}^N
          I \otimes \cdots \overset{\text{p-th}}{Q^\text{tok}} \cdots \otimes I
      \right)
    \right]
    (\m, \n)
  .
  \end{align}
  Now, it is well known 
  \begin{align}
    \exp ( A \otimes I + I \otimes B ) 
    &=
    \exp(A) \otimes \exp(B),
  \end{align}
  for any square matrices $A, B$, and applying it recursively, we have 
  \begin{align}
    &p_{t|s}(\m|\n)
    \notag \\
    &=
    \left[
      \exp \left(
        \Delta \bar{\sigma}(t|s)
        \overset{p=1}{Q^\text{tok}}
      \right)
      \otimes 
      \exp \left(
        \Delta \bar{\sigma}(t|s)
        \overset{p=2}{Q^\text{tok}}
      \right)
      \otimes 
      \cdots 
      \otimes 
      \exp \left(
        \Delta \bar{\sigma}(t|s)
        \overset{p=N}{Q^\text{tok}}
      \right)
    \right]
    (\m, \n)
    \notag \\
    &=
    \prod_{p=1}^N 
      \exp \left[
        \Delta \bar{\sigma}(t|s)
        Q^\text{tok}
      \right]
      (m_p, n_p)
      ,
  \end{align}
  which proves the theorem.
\end{proof}

The theorem shows that the forward transition kernel factorizes across token positions.
Approximating each token-level transition probability to first order in $\Delta t$ gives
\begin{align}
  p_{t+ \Delta t|t}(\m|\n)
  &\approx
  \prod_{p=1}^N
    \left(
      \underbrace{
      1_{m_p n_p} 
      +
      \sigma(t)
      Q^\text{tok}(m_p, n_p)
      \Delta t 
      }_{
        p_{t+\Delta t|t}(m_p | n_p)
      }
    \right)
    .
    \label{eq:tau-forward}
\end{align}
Given the current sequence $\n$, this factorized approximation allows all tokens $m_p$ to be sampled independently, so that multiple tokens can change within a single time step.
This motivates the $\tau$-leaping scheme for reverse-time sampling introduced below.

\subsubsection{$\tau$-leaping sampling}

For the reverse process, the generator $\tilde{Q}_t$ depends on time through the probability ratios, so \cref{thm:integral2} does not apply directly.
Nevertheless, the forward-process construction motivates an approximate treatment of the reverse dynamics, leading to an interpretation analogous to
\cref{eq:tau-forward}.
Specifically, we approximate the reverse-time evolution as follows:
\begin{align}
  \tilde{p}_{t-\Delta t|t}^\text{$\tau$-leap}(\m|\n)
  &:=
  \prod_{p=1}^N
    \left(
      1_{m_p n_p} 
      +
      \sigma(t)
      \underbrace{
        \tilde{Q}^\text{tok}_t(m_p, \n) 
      }_{
        \text{defined below}
      }
      \Delta t 
    \right)
    ,
    \
    \label{eq:tau-leaping}
    \\
    \tilde{Q}^\text{tok}_t(m_p, \n)
    &:=
    Q^\text{tok}(n_p, m_p) 
    \frac{p_t(\n_{p \to m_p})}{p_t(\n)}
    .
\end{align}
This distribution approximates the exact reverse-time transition, as can be seen by expanding it in $\Delta t$:
\begin{align}
  \tilde{p}_{t-\Delta t|t}^\text{$\tau$-leap}(\m|\n)
  &=
  1_{\m\n}
  +
  \sigma(t)
  \tilde{Q}_t (\m, \n) \Delta t 
  +
  O(\Delta t ^2)
  \notag \\
  &=
  \tilde{p}_{t-\Delta t|t}(\m|\n)
  +
  O(\Delta t ^2)
  .
\end{align}

\subsection{Transformation of the $\n_0$-predictor into a Score Estimator}

Some discrete diffusion models estimate the conditional probability of the clean state $\n_0$ given the corrupted state $\n$ at time $t$:
\begin{align}
  q_\theta(\n_0 \mid \n).
  \label{eq:p0t_model}
\end{align}
Examples include LLaDA \citep{Nie2025LargeLD}, and iLLaDA \citep{Nie2026iLLaDA}.
In this section, we briefly review this formulation and derive its relation to the score-based formulation in terms of score.

To relate the two formulations, consider the joint conditional distribution at times $t \geq s \geq 0$:
\begin{align}
    p_{t, s | 0}(\mathbf{n}, \mathbf{m} \mid \mathbf{n}_0)
    =
    p_{t|s}(\mathbf{n} \mid \mathbf{m})
    p_{s|0}(\mathbf{m} \mid \mathbf{n}_0).
    \label{eq:pts0}
\end{align}
The forward transition probabilities on the right-hand side can be evaluated using \cref{thm:integral2}.
Conditioning on $\mathbf{n}$ and marginalizing over $\mathbf{n}_0$ using the denoising model in \cref{eq:p0t_model} then yields an approximation to the reverse transition probability \citep{Austin2021StructuredDD}:
\begin{align}
  \tilde{q}{_\theta}_{s|t}(\m|\n)
  &\propto
  \sum_{\n_0}
    p_{t, s, 0}(\mathbf{n}, \mathbf{m}| \mathbf{n}_0)
    q_\theta(\n_0|\n)
    .
    \label{eq:den_model}
\end{align}
By setting $s = t - \Delta t$, the model yields the reverse conditional model
\begin{align}
  \tilde{q}{_\theta}_{t-\Delta t|t}(\mathbf{m}|\mathbf{n})
  &=
  1_{\m \n} + \sigma(t)\tilde{Q}{_\theta}_t(\m, \n) \Delta t + O(\Delta t^2)
  \notag
  \\
  &=
  1_{\m \n} + \sigma(t)
  \left\{
    \begin{array}{ll}
      Q(\n, \m) s_\theta(\n, t)_{\m} \Delta t
      &
      (\m \neq \n)
      \\
      - \sum_{\k \neq \m}
        Q(\m, \k) s_\theta(\m, t)_{\k} \Delta t
      &
      (\m = \n)
    \end{array}
  \right.
  + O(\Delta t^2)
  .
\end{align}
The piecewise expression above can equivalently be written as follows:
\begin{align}
  \tilde{q}{_\theta}_{t-\Delta t|t}(\mathbf{m}|\mathbf{n})
  &=
  1_{\m \n} 
  + 
  \sigma(t)
  \Big(
    Q(\n, \m) s_\theta (\n, t)_{\m} 
    -
    1_{\m \n} 
    \sum_{\k}
      Q(\m, \k)
      s_\theta (\n, t)_{\k} 
  \Big)
  \Delta t
  +
  O(\Delta t ^2).
  \label{eq:sc_model}
\end{align}

\begin{mytheorem}{Transformation formula}{ptos}
  By ensuring the consistency between \cref{eq:den_model} and \cref{eq:sc_model}, the following holds:
  \begin{align}
    s_\theta(\n, t)_{\m}
    &=
    \frac{
            \sum_{\n_0}
              p_{t|0}(\mathbf{m}|\mathbf{n}_0)
              q_\theta(\n_0|\n)
            }{
              \sum_{\n_0}
                p_{t|0}(\mathbf{n}|\mathbf{n}_0)
                q_\theta(\n_0|\n)
            }
    .
  \end{align}
\end{mytheorem}

\begin{proof}
  First, we rewrite the right-hand side of \cref{eq:den_model} by using \cref{eq:pts0} as
  \begin{align}
    \tilde{q}{_\theta}_{s|t}(\m|\n)
    &\propto
    \sum_{\n_0}
      p_{t, s, 0}(\mathbf{n}, \mathbf{m}| \mathbf{n}_0)
      q_\theta(\n_0|\n)
    \notag
    \\
    &=
    \sum_{\n_0}
      p_{t|s}(\mathbf{n}|\mathbf{m}) p_{s|0}(\mathbf{m}|\mathbf{n}_0)
      q_\theta(\n_0|\n).
  \end{align}
  Substituting $s = t - \Delta t$ into each probability, we obtain:
  \begin{align}
    p_{t|s}(\mathbf{n}|\mathbf{m})
    &=
    p_{t|t - \Delta t}(\mathbf{n}|\mathbf{m})
    \notag \\
    &=
    1_{\m \n}
    +
    \sigma(t)Q(\n, \m)\Delta t
    + 
    O(\Delta t ^2)
    ,
    \\
    p_{s|0}(\mathbf{m}|\mathbf{n}_0)
    &=
    p_{t - \Delta t|0}(\mathbf{m}|\mathbf{n}_0)
    =
    p_{t|0}(\mathbf{m}|\mathbf{n}_0)
    -
    \frac{\partial}{\partial t}
    p_{t|0}(\mathbf{m}|\mathbf{n}_0)
    \Delta t
    +
    O(\Delta t^2)
    \notag \\
    &=
    p_{t|0}(\mathbf{m}|\mathbf{n}_0)
    -
    \sum_{\k}
    \sigma(t) Q(\m, \k)
    p_{t|0}(\mathbf{k}|\mathbf{n}_0)
    \Delta t
    +
    O(\Delta t^2)
    ,
  \end{align}
  therefore, we get 
  \begin{align}
    &
    \sum_{\n_0}
      p_{t|s}(\mathbf{n}|\mathbf{m}) p_{s|0}(\mathbf{m}|\mathbf{n}_0)
      q_\theta(\n_0|\n)
    \notag \\
    &=
    \sum_{\n_0}
      \Big( 
        1_{\m \n}
        +
        \sigma(t)Q(\n, \m)\Delta t
      \Big)
      \Big(
        p_{t|0}(\mathbf{m}|\mathbf{n}_0)
        -
        \sum_{\k}
        \sigma(t) Q(\m, \k)
        p_{t|0}(\mathbf{k}|\mathbf{n}_0)
        \Delta t
      \Big)
      q_\theta(\n_0|\n)
    \notag \\
    &=
          1_{\m \n}
          \sum_{\n_0}
            p_{t|0}(\mathbf{m}|\mathbf{n}_0)
            q_\theta(\n_0|\n)
          \notag \\ & \
          +
          \sigma(t)
          \Big(
            Q(\n, \m)
            \sum_{\n_0}
              p_{t|0}(\mathbf{m}|\mathbf{n}_0)
              q_\theta(\n_0|\n)
            -
            1_{\m \n}
            \sum_{\k}
            Q(\m, \k)
            \sum_{\n_0}
              p_{t|0}(\mathbf{k}|\mathbf{n}_0)
              q_\theta(\n_0|\n)
          \Big)
          \Delta t
    .
  \end{align}
  Next, we determine the proportionality constant in \cref{eq:den_model}.
  The necessary condition is $\sum_{\m} \tilde{q}{_\theta}_{s|t}(\m|\n) = 1$, which can be realized by dividing the right hand side by its summation over $\m$.
  Since the Chapman-Kolmogorov equation holds, we have
  \begin{align}
    &
    \sum_{\m}
    \sum_{\n_0}
      p_{t|s}(\mathbf{n}|\mathbf{m}) p_{s|0}(\mathbf{m}|\mathbf{n}_0)
      q_\theta(\n_0|\n)
    =
    \sum_{\n_0}
      p_{t|0}(\mathbf{n}|\mathbf{n}_0)
      q_\theta(\n_0|\n)
    .
  \end{align}
  Combining all, we have 
  \begin{align}
    &
    \tilde{q}{_\theta}_{s|t}(\m|\n)
    \notag \\
    &=
    \frac{
      \sum_{\n_0}
      p_{t|s}(\mathbf{n}|\mathbf{m}) p_{s|0}(\mathbf{m}|\mathbf{n}_0)
      q_\theta(\n_0|\n)
    }{
      \sum_{\m}
      \sum_{\n_0}
      p_{t|s}(\mathbf{n}|\mathbf{m}) p_{s|0}(\mathbf{m}|\mathbf{n}_0)
      q_\theta(\n_0|\n)
    }
    \notag \\
    &=
          1_{\m \n}
          \underbrace{
          \frac{
          \sum_{\n_0}
            p_{t|0}(\mathbf{m}|\mathbf{n}_0)
            q_\theta(\n_0|\n)
          }{
            \sum_{\n_0}
              p_{t|0}(\mathbf{n}|\mathbf{n}_0)
              q_\theta(\n_0|\n)
          }
          }_{1}
          \notag \\ & 
          +
          \sigma(t)
          \Big(
            Q(\n, \m)
            \frac{
            \sum_{\n_0}
              p_{t|0}(\mathbf{m}|\mathbf{n}_0)
              q_\theta(\n_0|\n)
            }{
              \sum_{\n_0}
                p_{t|0}(\mathbf{n}|\mathbf{n}_0)
                q_\theta(\n_0|\n)
            }
            -
            1_{\m \n}
            \sum_{\k}
            Q(\m, \k)
            \frac{
            \sum_{\n_0}
              p_{t|0}(\mathbf{k}|\mathbf{n}_0)
              q_\theta(\n_0|\n)
            }{
              \sum_{\n_0}
                p_{t|0}(\mathbf{n}|\mathbf{n}_0)
                q_\theta(\n_0|\n)
            }
          \Big)
          \Delta t
    .
    \notag
  \end{align}
  Comparing it to the form in \cref{eq:sc_model}, we get teh formula. 
\end{proof}

Since we use the diffusion matrix decomposition in \cref{eq:delta_p_apdx}, the score function takes the form $s_\theta(\n, t)_{\n_{p \to m}}$ in short $s_\theta(\n, t)_{p \to m}$.
In addition, assuming the decomposition of the $q_\theta$ model as $q_\theta(\n_0|\n) = \prod_{q=1}^{\NumToken} q_\theta([\n_0]_i|\n)$, we can marginalize over all but the $p$-th component, yielding the following:
\begin{align}
  s_\theta(\n, t)_{p \to m}
  &=
  \frac{
  \sum_{\n_0}
    p_{t|0}(m|[\mathbf{n}_0]_p)
    q_\theta(\n_0|\n)
  }{
    \sum_{\n_0}
    p_{t|0}(n_p|[\mathbf{n}_0]_p)
    q_\theta(\n_0|\n)
  }
  &=
  \frac{
  \sum_{[\n_0]_p}
    p_{t|0}(m|[\mathbf{n}_0]_p)
    q_\theta([\mathbf{n}_0]_p|\n)
  }{
    \sum_{[\mathbf{n}_0]_p}
    p_{t|0}(n_p|[\mathbf{n}_0]_p)
    q_\theta([\mathbf{n}_0]_p|\n)
  }
  .
\end{align}

\paragraph{Uniform diffusion}

If we take $Q^\text{tok}(m, n) = 1$ $(m \neq n)$, then the token-level transition can be calculated as 
\begin{align}
  p_{t|0}(m|n)
  &=
  e^{- V\bar{\sigma}(t)} 1_{m n}
  +
  (1 - e^{- V\bar{\sigma}(t)}) \frac{1}{V}
  ,
\end{align}
and 
\begin{align}
  s_\theta(\n, t)_{p \to m}
  &=
  \frac{
    e^{- V\bar{\sigma}(t)} 
    q_\theta(m|\n)
    +
    (1 - e^{- V\bar{\sigma}(t)}) \frac{1}{V}  
  }{
    e^{- V\bar{\sigma}(t)} 
    q_\theta(n_p|\n)
    +
    (1 - e^{- V\bar{\sigma}(t)}) \frac{1}{V}  
  }
  .
  \label{eq:p2s_unif}
\end{align}

\paragraph{Masked diffusion}

If we take $Q^\text{tok}(m, n) = 1_{m 0} 1_{n \neq 0}$, then the token-level transition can be calculated as 
\begin{align}
  p_{t|0}(m|n)
  &=
  e^{-\bar{\sigma}(t)} 1_{mn}
  +
  (1 - e^{-\bar{\sigma}(t)}) 1_{m 0}
  ,
\end{align}
and 
\begin{align}
  s_\theta(\n, t)_{p \to m}
  &=
  \frac{
    e^{- \bar{\sigma}(t)} 
    q_\theta(m|\n)
    +
    (1 - e^{- \bar{\sigma}(t)}) 
    1_{m 0}
  }{
    e^{- \bar{\sigma}(t)} 
    q_\theta(n_p|\n)
    +
    (1 - e^{- \bar{\sigma}(t)}) 
    1_{n_p 0} 
  }
  =
  \frac{
    e^{- \bar{\sigma}(t)} 
  }{
    1 - e^{- \bar{\sigma}(t)} 
  }
  q_\theta(m|\n)
  ,
  \label{eq:p2s_mask}
\end{align}
for $n_p = 0, m \neq 0$.
\section{A Brief Review of Open Quantum Systems}
\label{sec:quantum_intro}

This section reviews the concepts of open quantum systems underlying the formulation of diffusion language models in \cref{sec:dlm_as_oqs}.

\subsection{Hilbert Space and Notation}

We work with a finite-dimensional complex Hilbert space $\mc{H}$, whose vectors represent pure quantum states when normalized.\footnote{
    In the language-model formulation of
    \cref{sec:dlm_as_oqs}, we use
    $\mc{H}=\bigotimes_{p=1}^{\NumToken}\mb{C}^{|\vocab|}$,
    where $|\vocab|$ is the number of token states and
    $\NumToken$ is the sequence length.
}
We use the following notation:
\begin{description}
    \item[State vectors and bra-ket notation]
        A vector in $\mc{H}$ is written as $\ket{\bpsi}$, called a ket.
        Its Hermitian conjugate $\bra{\bpsi}=(\ket{\bpsi})^\dag$ is called a bra and belongs to the dual space $\mc{H}^*$.
        This convention is known as \textbf{bra-ket notation}.

    \item[Inner product]
        The inner product of two vectors is written as $\braket{\bpsi_1|\bpsi_2}\in\mb{C}$.
        It satisfies conjugate symmetry: $\braket{\bpsi_1|\bpsi_2}^* = \braket{\bpsi_2|\bpsi_1}$, where $*$ denotes complex conjugation.

    \item[Norm]
        The norm of $\ket{\bpsi}$ is $\|\ket{\bpsi}\| = \sqrt{\braket{\bpsi|\bpsi}}$.
        A normalized state vector satisfies $\braket{\bpsi|\bpsi}=1$.

    \item[Operators and outer products]
        A linear map $\op{L}:\mc{H}\to\mc{H}$ is called an operator.
        The outer product of two vectors, $\ket{\bpsi_1}\bra{\bpsi_2}$, defines an operator whose action is
        \[
            \bigl(\ket{\bpsi_1}\bra{\bpsi_2}\bigr)
            \ket{\bpsi}
            =
            \ket{\bpsi_1}\braket{\bpsi_2|\bpsi}.
        \]
\end{description}

\subsection{Quantum State and Measurement}

A pure quantum state is represented by a normalized vector in $\mc{H}$.
More generally, statistical uncertainty or interactions with an environment require a description in terms of density operators:

\begin{mydefinition}{Quantum State}{state}
    An operator $\op{\rho}$ on $\mc{H}$ is called a density operator, or quantum state, if it satisfies
    \begin{enumerate}
        \item Hermiticity: $\op{\rho}^\dagger=\op{\rho}$,
        \item positive semidefiniteness: $\op{\rho}\geq 0$, and
        \item unit trace: $\tr_\mc{H}\op{\rho}=1$.
    \end{enumerate}
\end{mydefinition}

By the spectral theorem, any density operator can be written as $\op{\rho}=\sum_i p_i\ket{e_i}\bra{e_i}$, where $\{\ket{e_i}\}$ is an orthonormal basis, $p_i\geq 0$, and $\sum_i p_i=1$.
Thus, its eigenvalues form a probability distribution, while the basis specifies the corresponding quantum states.

\begin{mydefinition}{Pure State}{pure_state}
    A normalized vector $\ket{\bpsi}\in\mc{H}$, with $\braket{\bpsi|\bpsi}=1$, defines a quantum state
    \begin{align}
        \op{\rho}
        &=
        \ket{\bpsi}\bra{\bpsi}.
    \end{align}
    A state of this form is called a pure state.
\end{mydefinition}

As noted above, any density operator admits a decomposition $\op{\rho}=\sum_i p_i\ket{e_i}\bra{e_i}$.
It can therefore be represented as a statistical ensemble of pure states.
A state that is not pure is called a mixed state.

A density operator determines the probabilities of measurement outcomes.
For a measurement in a chosen orthonormal basis, the outcome is sampled according to the Born rule, as defined below.\footnote{
    More general measurements can be described by
    positive operator-valued measures (POVMs).
    Here, we restrict attention to projective measurements
    in an orthonormal basis.
}

\begin{mydefinition}{Projective Measurement}{measurement}
    Let $\{\ket{e}\}_{e\in E}$ be an orthonormal basis of $\mc{H}$.
    For a quantum state $\op{\rho}$, the probability of obtaining outcome $e$ is
    \[
        P(e)
        =
        \tr_\mc{H}\bigl(\ket{e}\bra{e}\op{\rho}\bigr)
        =
        \bra{e}\op{\rho}\ket{e}.
    \]
\end{mydefinition}

\subsection{Dynamics of Open Quantum Systems}

An open quantum system interacts with its environment, which can induce dissipation and fluctuations.
Its evolution is described by maps that preserve the validity of quantum states, including when the system is entangled with an auxiliary system.

\begin{mydefinition}
    {Completely Positive Trace-Preserving (CPTP) Map}{cptp}
    A linear map $\mc{T}$ on operators on $\mc{H}$ is completely positive and trace-preserving if it admits a Kraus representation
    \[
        \mc{T}(\op{\rho})
        =
        \sum_i \op{K}_i\op{\rho}\op{K}_i^\dag,
        \qquad
        \sum_i \op{K}_i^\dag\op{K}_i=\op{1},
    \]
    where the $\op{K}_i$ are operators on $\mc{H}$ and $\op{1}$ is the identity operator.
\end{mydefinition}

CPTP maps are quantum analogues of Markov kernels: they map density operators to density operators.
Complete positivity additionally ensures that $\mc{T}\otimes\mathrm{id}$ preserves positivity for any auxiliary system, where $\mathrm{id}$ is the identity map on that system.

The Kraus representation makes positivity and trace preservation explicit.
For any $\ket{\bpsi}\in\mc{H}$, let $\ket{\bpsi_i}=\op{K}_i^\dag\ket{\bpsi}$.
Then
\begin{align}
    \bra{\bpsi}\mc{T}(\op{\rho})\ket{\bpsi}
    &=
    \sum_i \bra{\bpsi_i}\op{\rho}\ket{\bpsi_i}
    \geq 0.
\end{align}
Using the cyclicity of the trace, we also obtain
\begin{align}
    \tr_{\mc{H}}\mc{T}(\op{\rho})
    &=
    \sum_i
    \tr_{\mc{H}}
    \bigl(\op{K}_i\op{\rho}\op{K}_i^\dag\bigr)
    =
    \tr_{\mc{H}}
    \left(
        \op{\rho}\sum_i\op{K}_i^\dag\op{K}_i
    \right)
    =
    \tr_{\mc{H}}\op{\rho}
    =
    1.
\end{align}

In classical diffusion, the generator $L_t$ in $\frac{d}{dt}p_t=L_t(p_t)$ is chosen so that the evolution from any time $s$ to $t\geq s$ defines a Markov kernel.
The quantum counterpart is an evolution equation for $\op{\rho}_t$ whose propagators are CPTP maps.

\begin{mytheorem}{Lindblad Equation}{lindblad}
    Consider the evolution equation
    \begin{align}
        \frac{d}{dt}\op{\rho}_t
        &=
        \mc{L}_t(\op{\rho}_t),
        \label{eq:lindblad}
    \end{align}
    where $\mc{L}_t$ is a continuous family of linear maps on operators on a finite-dimensional Hilbert space.
    Suppose that, for every $t\geq s$, the solution defines a CPTP map from $\op{\rho}_s$ to $\op{\rho}_t$.
    Then there exist a Hermitian operator $\op{H}_t$, called the \textit{Hamiltonian}, and operators $\{\op{L}_t^\mu\}_{\mu\in M}$, called \textit{jump operators}, such that
    \begin{align}
        \mc{L}_t(\op{\rho}_t)
        &=
        -i[\op{H}_t,\op{\rho}_t]
        +\sum_\mu
        \left(
            \op{L}_t^\mu\op{\rho}_t
            (\op{L}_t^\mu)^\dagger
            -\frac{1}{2}
            \left\{
                (\op{L}_t^\mu)^\dagger\op{L}_t^\mu,
                \op{\rho}_t
            \right\}
        \right),
        \label{eq:lindblad2}
    \end{align}
    where $[\op{A},\op{B}]=\op{A}\op{B}-\op{B}\op{A}$ and $\{\op{A},\op{B}\}=\op{A}\op{B}+\op{B}\op{A}$ denote the commutator and anticommutator, respectively.
\end{mytheorem}
\begin{proof}
    We omit the proof here; see, for example, \cite{Manzano2019ASI}.
\end{proof}

As we will see below, \cref{eq:lindblad} generalizes classical diffusion dynamics to quantum states.
A corresponding time-reversed formulation can also be constructed \citep{Kwon2018FluctuationTF}, providing a quantum counterpart to the reverse-time dynamics used in diffusion models.

\begin{mytheorem}{Time reversal of Lindblad equation}{thm_time_rev}
    Under the time-backward operators
    \begin{align}
    \tilde{H}{_t}
    &=
    -\frac{1}{2} \left(
        \sqrt{\op{\rho}_t} \op{H}{_t} 
        + i\frac{d \sqrt{\op{\rho}_t}}{dt} 
        + \frac{i}{2}\sum_\mu \sqrt{\op{\rho}_t} \op{L}{_t^\mu}^\dagger \op{L}{_t^\mu} 
        \right)\frac{1}{\sqrt{\op{\rho}_t}}
        +(\text{1st term})^\dag
    , 
    \quad
    \\
    \tilde{L}{_t^\mu}
    &=
    \sqrt{\op{\rho}_t} \op{L}{_t^\mu}^\dagger \frac{1}{\sqrt{\op{\rho}_t}}
    , \label{eq:backward_operators}
    \end{align}
    the Lindblad equation can be represented by 
    \begin{align}
        -\frac{d}{dt} \op{\rho}_{t}
        &=
        - i [\tilde{H}{_t}, \op{\rho}_{t}]
        + \sum_\mu \tilde{L}{_t^\mu} \op{\rho}_{t} \tilde{L}{_t^\mu}^\dagger
        - \frac{1}{2} \left\{
            \sum_\mu \tilde{L}{_t^\mu}^\dagger\tilde{L}{_t^\mu}, \op{\rho}_{t} 
            \right\}
        .
        \label{eq:lindblad_rev}
    \end{align}
\end{mytheorem}

\begin{proof}
    For a proof, see the appendix of \cite{Kwon2018FluctuationTF}.
\end{proof}

\subsection{Stochastic Unraveling}
\label{sec:unraveling}

\subsubsection{Unraveling 1-jump per step}

There are several approaches to simulating the Lindblad equation.
Here, we introduce the quantum jump method, which is based on a stochastic differential equation driven by jump processes.
To simplify the notation, we define the effective Hamiltonian as
\begin{align}
    \op{H}{^\text{eff}_t} = \op{H}{_t} - \frac{i}{2} \sum_\mu \op{L}{_t^\mu}^\dagger \op{L}{_t^\mu}.
\end{align}
\begin{mytheorem}{Quantum Jump}{qjump}
    Let us consider one of stochastic time evolutions for {pure state} vector  with jump processes $d N_t^\mu$,
    \begin{align}
        \ket{\bpsi_{t+dt}}
        &=
        U_{t + dt | t}^\text{eff}
        \ket{\bpsi_t}
        \propto
        (1 -i \op{H}{^\text{eff}_t}dt)\ket{\bpsi_t}
        + \sum_\mu (\op{L}{_t^\mu} - 1) dN_t^\mu \ket{\bpsi_t}
        ,
        \label{eq:qjump_apdx}
    \end{align}
    where $d N_t^\mu$ are multiple jump processes with
    \begin{align}
        \text{Prob}\Big(
            \forall \mu, 
            d N_t^\mu = 0
        \Big)
        &=
        1 - \sum_\mu \bra{\bpsi_t} \op{L}{_t^\mu}^\dagger \op{L}{_t^\mu} \ket{\bpsi_t} dt
        ,\label{eq:jump}
        \\
        \text{Prob}\Big(
            \exists^1 \mu, 
            d N_t^\mu = 1, 
        \Big)
        &=
        \bra{\bpsi_t} \op{L}{_t^\mu}^\dagger \op{L}{_t^\mu} \ket{\bpsi_t} dt.
    \end{align}
    then the update $\ket{\bpsi_t}\bra{\bpsi_t} \to \mathop{\mb{E}}\limits_{dN_t}\ket{\bpsi_{t + dt}}\bra{\bpsi_{t+dt}}$ is given by Lindblad equation (\ref{eq:lindblad}).
\end{mytheorem}

\begin{proof}
    Throughout this calculation, we use the following It\^o multiplication rules for the jump increments:
    \begin{align}
        0 = dt \ dN_t^\mu = dt^2,
        \quad 
        dN_t^\mu dN_t^\nu = \delta^{\mu \nu} dN_t^\mu
        .
    \end{align}
    To determine the normalization factor in \cref{eq:qjump_apdx}, we compute the squared norm of the unnormalized state:
    \begin{align*}
        &
        \| \text{rhs of \cref{eq:qjump_apdx}} \|^2
        \\
        &=
        \bra{\bpsi_t}
            (1 + 
            \underbrace{i \op{H}{^\text{eff}_t}^\dag dt - i \op{H}{^\text{eff}_t} dt}
            _{
                - \sum_\mu \op{L}{_t^\mu}^\dag \op{L}{_t^\mu}
            }
            )
            +
            \sum_{\mu, \nu}
                \underbrace{dN_t^\mu dN_t^\nu }
                _{
                    \delta^{\mu \nu} dN_t^\mu
                }
                (\op{L}{_t^\mu}^\dag - 1)
                (\op{L}{_t^\nu} - 1)
        \ket{\bpsi_t}
        \\ & \qquad
        +
        \bra{\bpsi_t}
            \sum_\mu
                (
                \op{L}{_t^\mu}
                +
                \op{L}{_t^\mu}^\dag
                -
                2
                )dN_t^\mu
        \ket{\bpsi_t}
        \\
        &=
        1 
        - 
        \sum_\mu 
            \langle \op{L}{_t^\mu}^\dag \op{L}{_t^\mu} \rangle_{\bpsi_t} dt 
        +
        \sum_\mu 
            (
                \langle \op{L}{_t^\mu}^\dag \op{L}{_t^\mu} \rangle_{\bpsi_t}
                -
                1
            )dN_t^\mu,
    \end{align*}
    which fixes the normalization as 
    \begin{align*}
        \ket{\bpsi_{t+dt}}
        &=
        \frac{
            (1 -i \op{H}{^\text{eff}_t}dt)\ket{\bpsi_t}
            + \sum_\mu (\op{L}{_t^\mu} - 1) dN_t^\mu \ket{\bpsi_t}
        }{
            \sqrt{
                1 
                - 
                \sum_\mu 
                    \langle \op{L}{_t^\mu}^\dag \op{L}{_t^\mu} \rangle_{\bpsi_t} dt 
                +
                \sum_\mu 
                    (
                        \langle \op{L}{_t^\mu}^\dag \op{L}{_t^\mu} \rangle_{\bpsi_t}
                        -
                        1
                    )dN_t^\mu
            }
        }.
    \end{align*}
    This expression simplifies when we consider the no-jump and jump cases separately.    
    If no jump occurs, $dN_t^\mu=0$ for all $\mu$, and 
    \begin{align*}
        \ket{\bpsi_{t+dt}}
        &=
        \frac{
            (1 -i \op{H}{^\text{eff}_t}dt)\ket{\bpsi_t}
        }{
            \sqrt{
                1 
                - 
                \sum_\mu 
                    \langle \op{L}{_t^\mu}^\dag \op{L}{_t^\mu} \rangle_{\bpsi_t} dt 
            }
        }.
    \end{align*}
    If a jump occurs in channel $\mu$, then $dN_t^\mu=1$ and $dN_t^\nu=0$ for all $\nu\neq\mu$, giving    
    \begin{align*}
        \ket{\bpsi_{t+dt}}
        &=
        \frac{
            (1 -i \op{H}{^\text{eff}_t}dt)\ket{\bpsi_t}
            + (\op{L}{_t^\mu} - 1) \ket{\bpsi_t}
        }{
            \sqrt{
                1 
                - 
                \sum_\mu 
                    \langle \op{L}{_t^\mu}^\dag \op{L}{_t^\mu} \rangle_{\bpsi_t} dt 
                +
                    (
                        \langle \op{L}{_t^\mu}^\dag \op{L}{_t^\mu} \rangle_{\bpsi_t}
                        -
                        1
                    )
            }
        }
        =
        \frac{
            \op{L}{_t^\mu} \ket{\bpsi_t}
        }{
            \sqrt{
                \langle \op{L}{_t^\mu}^\dag \op{L}{_t^\mu} \rangle_{\bpsi_t}
            }
        },
    \end{align*}
    neglecting terms of order $dt$. Thus, 
    \begin{align}
        \ket{\bpsi_{t+dt}}
        &=
        \left\{
            \begin{array}{ll}
                \frac{
                    (1 -i \op{H}{^\text{eff}_t}dt)\ket{\bpsi_t}
                }{
                    \sqrt{
                        1 
                        - 
                        \sum_\mu 
                            \langle \op{L}{_t^\mu}^\dag \op{L}{_t^\mu} \rangle_{\bpsi_t} dt 
                    }
                }
                & 
                \text{with Prob}\Big(
                    \forall \mu, 
                    d N_t^\mu = 0
                \Big)
                = 1 - \sum_\mu \langle \op{L}{_t^\mu}^\dagger \op{L}{_t^\mu} \rangle_{\bpsi_t} dt
                \\
                \frac{
                    \op{L}{_t^\mu} \ket{\bpsi_t}
                }{
                    \sqrt{
                        \langle \op{L}{_t^\mu}^\dag \op{L}{_t^\mu} \rangle_{\bpsi_t}
                    }
                }
                &
                \text{with Prob}\Big(
                    \exists^1 \mu, 
                    d N_t^\mu = 1
                \Big)
                = \langle\op{L}{_t^\mu}^\dagger \op{L}{_t^\mu} \rangle_{\bpsi_t} dt
            \end{array}
        \right.
        ,\label{eq:norm_qjump}
    \end{align} 
    and 
    \begin{align*}
        &
        \mathop{\mb{E}}\limits_{dN_t}
        \ket{\bpsi_{t+dt}} \bra{\bpsi_{t+dt}}
        \\
        &=
        \underbrace{
        \text{Prob}\Big(
                    \forall \mu, \
                    d N_t^\mu = 0
                \Big)
        }_{
            1 - \sum_\mu \braket{
                            \op{L}{_t^\mu}^\dag
                            \op{L}{_t^\mu}
                            }_{\bpsi_t}
                    d t
        }
        \frac{
            (1 -i \op{H}{^\text{eff}_t}dt)\ket{\bpsi_t}
        }{
            \sqrt{
                1 
                - 
                \sum_\mu 
                    \langle \op{L}{_t^\mu}^\dag \op{L}{_t^\mu} \rangle_{\bpsi_t} dt 
            }
        }
        \frac{
            \bra{\bpsi_t}
            (1 +i \op{H}{^\text{eff}_t}^\dag dt)
        }{
            \sqrt{
                1 
                - 
                \sum_\mu 
                    \langle \op{L}{_t^\mu}^\dag \op{L}{_t^\mu} \rangle_{\bpsi_t} dt 
            }
        }
        \\ & \qquad\qquad\qquad\qquad\qquad
        +
        \sum_\mu
            \underbrace{
                \text{Prob}\Big(
                    \exists^1 \mu, 
                    d N_t^\mu = 1, 
                \Big)
            }_{
                \braket{
                    \op{L}{_t^\mu}^\dag
                    \op{L}{_t^\mu}
                }_{\bpsi_t}
                dt
            }
            \frac{
                \op{L}{_t^\mu} \ket{\bpsi_t}
            }{
                \sqrt{
                    \langle \op{L}{_t^\mu}^\dag \op{L}{_t^\mu} \rangle_{\bpsi_t}
                }
            }
            \frac{
                \bra{\bpsi_t}
                \op{L}{_t^\mu}^\dag
            }{
                \sqrt{
                    \langle \op{L}{_t^\mu}^\dag \op{L}{_t^\mu} \rangle_{\bpsi_t}
                }
            }
        \\
        &=
        (1 - i \op{H}{^\text{eff}_t} dt) \ket{\bpsi_t} 
        \bra{\bpsi_t}(1 + i \op{H}{^\text{eff}_t}^\dag dt)
        + 
        \sum_\mu 
            \op{L}{_t^\mu} 
            \ket{\bpsi_t}\bra{\bpsi_t}
            \op{L}{_t^\mu}^\dag dt
        \\
        &=
        \ket{\bpsi_t}\bra{\bpsi_t}
        \\ & \quad
        +
        \Big(
        \underbrace{
        -i \Big[\op{H}{_t}, \ket{\bpsi_t}\bra{\bpsi_t}\Big]
        -\frac{1}{2} 
            \Big\{ 
                \sum_\mu \op{L}{_t^\mu}^\dag \op{L}{_t^\mu}
                ,
                \ket{\bpsi_t}\bra{\bpsi_t}
            \Big\}
        + \sum_\mu 
            \op{L}{_t^\mu} 
            \ket{\bpsi_t}\bra{\bpsi_t}
            \op{L}{_t^\mu}^\dag
        }_{
            \mc{L}{_t}\Big( \ket{\bpsi_t}\bra{\bpsi_t} \Big)
        }
        \Big) dt
        .
    \end{align*}
\end{proof}
\section{Proofs of the Main Theorems}
\label{sec:pauli}

This subsection provides the detailed explanation of theorems in the main text.

\subsection{Proof of \cref{thm:dlm_as_oqs}}

An important observation in this paper is that we can embed the usual diffusion process to the diagonal part of Lindblad equation \cref{eq:lindblad}.

\begin{mytheorem}{Pauli master equation}{pauli}
    Under the condition:
    \begin{align}
      \begin{array}{ll}
      &
      \op{H}=0,
      \quad
      \forall \k, \m, \n \in \vocab^\NumToken \ (\m \neq \n ), \
      \left\{
      \begin{array}{ll}
          0 
          &= 
          \sum_\mu 
            \bra{\m} 
              \op{L}{_t^\mu}
            \ket{\k} 
            \bra{\k} 
              \op{L}{_t^\mu}^\dagger 
            \ket{\n}, 
          \\
          0 
          &= 
          \sum_\mu 
            \bra{\m} 
              \op{L}{_t^\mu}^\dagger 
            \ket{\k} 
            \bra{\k}
              \op{L}{_t^\mu}
            \ket{\n},
      \end{array}
      \right.
      \end{array}
      \label{eq:red_cond}
    \end{align}
    then $\op{\rho}_t$ can be diagonal and \cref{eq:lindblad} reduces to the diffusion ODE \cref{eq:diff_eq}, i.e., 
    \begin{align}
      \op{\rho}_t
      &=
      \sum_{\m \in \vocab^\NumToken}
        p_t(\m) 
        \ket{\m}
        \bra{\m},
      \quad 
      \frac{d}{dt} 
      p_t(\m) 
      = 
      \sum_{\n \in \vocab^\NumToken} 
        Q_t(\m, \n) 
        p_t(\n),
    \end{align}
    where 
    \begin{align}
        Q_t(\m, \n)
        &=
        \sum_\mu 
          | \bra{\m} 
            \op{L}{_t^\mu}
            \ket{\n} |^2
        \quad 
        (\m \neq \n),
        \quad
        Q_t(\m, \m)
        =
        - \sum_{\n \in \vocab^\NumToken - \{\m\}}
            Q_t(\n, \m)
    \end{align}
\end{mytheorem}

\begin{proof}
  Basic idea is summarized in the following steps, 
  1) regarding $p_t(\m) = \bra{\m} \op{\rho}_t \ket{\m}$, and 
  2) fixing $0 = \bra{\m} \op{\rho}_t \ket{\n}$ for $\m \neq \n$.
  Let us check the first step.
  The time evolution of $p_t(\m) = \bra{\m} \op{\rho}_t \ket{\m}$ under $0 = \bra{\m} \op{\rho}_t \ket{\n} \quad (\m\neq \n)$ reduces to the classical master equation as follows.
  \begin{align*}
    &\frac{d}{dt}p_t(\m) 
    =
    \bra{\m} \frac{d}{dt} \op{\rho}_t \ket{\m}
    =
    \sum_\mu
      \left(
        \bra{\m} 
          \op{L}{_t^\mu} 
          \op{\rho}_t 
          \op{L}{_t^\mu}^\dagger
        \ket{\m} 
        - 
        \frac{1}{2} 
        \begin{pmatrix}
          \bra{\m}
            \op{L}{_t^\mu}^\dagger 
            \op{L}{_t^\mu} 
            \op{\rho}_t
          \ket{\m}
        \\
        +
          \bra{\m}
            \op{\rho}_t 
            \op{L}{_t^\mu}^\dagger 
            \op{L}{_t^\mu} 
          \ket{\m} 
        \end{pmatrix}
      \right)
    \\
    &=
    \sum_\mu
    \sum_{\n, \k }
    \left(
      \bra{\m}
        \op{L}{_t^\mu}
      \ket{\n}
      \underbrace{
        \bra{\n}
          \op{\rho}_t
        \ket{\k}
      }_{
        p_t(\n) 1_{\n\k}
        }
      \bra{\k}
        \op{L}{_t^\mu}^\dagger 
      \ket{\m}
      - 
      \frac{1}{2} 
        \begin{pmatrix}
        \bra{\m}
          \op{L}{_t^\mu}^\dagger
        \ket{\n}
        \bra{\n}
          \op{L}{_t^\mu}
        \ket{\k}
        \underbrace{
          \bra{\k}
            \op{\rho}_t
          \ket{\m}
        }_{
          p_t(\k) 1_{\k\m}
          }
        \\
        +
            \underbrace{
              \bra{\m}
                \op{\rho}_t
              \ket{\n}
              }_{
                p_t(\m)1_{\m\n}
                }  
            \bra{\n}
              \op{L}{_t^\mu}^\dagger 
            \ket{\k}
            \bra{\k}
              \op{L}{_t^\mu} 
            \ket{\m}
        \end{pmatrix}
    \right)\\
    &=
    \sum_{\n \in \vocab^\NumToken - \m}
    \left(
      \underbrace{
        \sum_\mu
          |\bra{\m} 
          \op{L}{_t^\mu}
          \ket{\n}|^2
      }_{
        Q_t(\m, \n)
        }
      p_t(\n) 
      - 
      \underbrace{
        \sum_\mu
          |\bra{\n}
          \op{L}{_t^\mu}
          \ket{\m}|^2
      }_{
        Q_t(\n, \m)
        }
      p_t(\m)
    \right)
    =
    \sum_{\n \in \vocab^\NumToken}
      Q_t(\m, \n)
      p_t(\n), 
  \end{align*}
  where we define $Q_t(\m, \m) = - \sum_{\n \in \vocab^\NumToken - \{\m\}} Q_t(\n, \m)$.
  Next, we consider the condition for $\frac{d}{dt}\bra{\m} \op{\rho}_t \ket{\n} = 0$ with $\m \neq \n$.
  \begin{align*}
    0 
    &= 
    \bra{\m} 
      \frac{d}{dt}
        \op{\rho}_t 
    \ket{\n}
    =
    \sum_\mu
      \left(
        \bra{\m} 
          \op{L}{_t^\mu} 
          \op{\rho}_t 
          \op{L}{_t^\mu}^\dagger
        \ket{\n} 
        - 
        \frac{1}{2} 
          \begin{pmatrix}
            \bra{\m} 
              \op{L}{_t^\mu}^\dagger 
              \op{L}{_t^\mu} 
              \op{\rho}_t 
            \ket{\n} 
            \\
            +
            \bra{\m} 
              \op{\rho}_t 
              \op{L}{_t^\mu}^\dagger 
              \op{L}{_t^\mu} 
            \ket{\n}  
          \end{pmatrix}
      \right) 
    \\
    &=
    \sum_\mu
      \sum_{\l, \k}
        \left(
          \bra{\m} 
            \op{L}{_t^\mu} 
          \ket{\l}
          \underbrace{
            \bra{\l} 
              \op{\rho}_t 
            \ket{\k}
          }_{
            p_t(\l) 1_{\l\k}
            }
          \bra{\k}
            \op{L}{_t^\mu}^\dagger 
          \ket{\n} 
          - 
          \frac{1}{2} 
            \begin{pmatrix}
              \bra{\m} 
                \op{L}{_t^\mu}^\dagger 
              \ket{\l}
              \bra{\l} 
                \op{L}{_t^\mu} 
              \ket{\k}
              \underbrace{
                \bra{\k} 
                  \op{\rho}_t 
                \ket{\n} 
              }_{
                p_t(\k) 1_{\k\n}
                }
              \\
              +
              \underbrace{
                \bra{\m} 
                  \op{\rho}_t 
                \ket{\l}
              }_{
                p_t(\m)1_{\m\l}
                }  
              \bra{\l}
                \op{L}{_t^\mu}^\dagger 
              \ket{\k}  
              \bra{\k}
                \op{L}{_t^\mu} 
              \ket{\n}  
            \end{pmatrix}
        \right) 
    \\
    &=
    \sum_\mu
      \sum_{\k}
        \left(
          \bra{\m} 
            \op{L}{_t^\mu} 
          \ket{\k} 
          \bra{\k}
            \op{L}{_t^\mu}^\dagger 
          \ket{\n} 
          p_t(\k)
          - 
          \frac{1}{2} 
            \begin{pmatrix}
              \bra{\m} 
                \op{L}{_t^\mu}^\dagger 
              \ket{\k}
              \bra{\k} 
                \op{L}{_t^\mu} 
              \ket{\n}
              p_t(\n)
              \\
              +
              \bra{\m}
                \op{L}{_t^\mu}^\dagger 
              \ket{\k}  
              \bra{\k}
                \op{L}{_t^\mu} 
              \ket{\n}  
              p_t(\m)
            \end{pmatrix}
        \right)
  \end{align*}
  Sufficient condition for vanishing it is, therefore
  \begin{align*}
    \forall \k, \m, \n \in \vocab^\NumToken \ (\m \neq \n) \quad 
    \left\{
      \begin{array}{l}
        0 = \sum_\mu \bra{\m} \op{L}{_t^\mu} \ket{\k} \bra{\k} \op{L}{_t^\mu}^\dagger \ket{\n}\\
        0 = \sum_\mu \bra{\m} \op{L}{_t^\mu}^\dagger \ket{\k} \bra{\k}\op{L}{_t^\mu} \ket{\n}
      \end{array}
    \right.
    .
  \end{align*}
\end{proof}

We can prove the \cref{thm:dlm_as_oqs} as a corollary of this theorem.
Under the constraint of jump operators in \cref{eq:full_jump_operators} \footnote{
  For computational purposes, we will represent the tensor product of tensors along each dimension vertically.
} and \cref{eq:mLk-main}, i.e., 
\begin{align}
  &\op{L}{_t^{\mu=(p,m',n')}}
  =
  \sqrt{\sigma(t)}
  \begin{array}{c}
    I
    \\
    \otimes 
    \\
    \vdots 
    \\
    \otimes
    \\
    \op{L}{_{m', n'}^\text{$(p)$tok}}
    \\
    \otimes
    \\
    \vdots
    \\
    \otimes 
    \\
    I
  \end{array}
  \quad
  \text{where }
      \op{L}{_{m', n'}^\text{$(p)$tok}}
    =
    \sqrt{
      Q^\text{tok}(m', n')
    }
      \ket{m'}\bra{n'}
    .
    \label{eq:constraint_appdx}
\end{align}

Under these conditions, we have 
\begin{align*}
  &
  \sum_\mu
    \bra{\m}
      \op{L}{_t^\mu}
    \ket{\k}
    \bra{\k}
      \op{L}{_t^\mu}^\dag
    \ket{\n}
  \notag
  \\
  &=
  \sigma(t)
  \sum_{p, m', n'}
      \left(
        \sqrt{
          Q^\text{tok}(m', n')^2
        }
      \begin{array}{ccc}
          \bra{m_1} I \ket{k_1} 
          & 
          \bra{k_1} I \ket{n_1}
          &
          =
          {\color{red}
          1_{m_1 n_1}
          }
          1_{m_1 k_1}
        \\
          \times & \times
        \\
          \vdots & \vdots & \vdots
        \\
          \times & \times
        \\
          \braket{m_p|m'} 
          \braket{n'|k_p}
          &
          \braket{k_p|n'} 
          \braket{m'|n_p}
          &
          =
          {\color{red}
          1_{m_p n_p}
          }
          1_{m_p m'}
          1_{k_p n'}
        \\
          \times & \times
        \\
          \vdots & \vdots & \vdots
        \\
          \times & \times
        \\
          \bra{m_\NumToken} I \ket{k_\NumToken}
          &
          \bra{k_\NumToken} I \ket{n_\NumToken}
          &
          =
          {\color{red}
          1_{m_\NumToken n_\NumToken}
          }
          1_{m_\NumToken k_\NumToken}
      \end{array}
      \right)
  \notag
  \\
  &
  \propto 
    {\color{red}
    1_{\m\n}
    }
  = 0 \quad (\text{for } \m \neq \n)
  .
\end{align*}
We can repeat same procedure for checking another condition:
\begin{align}
  &
  \sum_\mu
    \bra{\m}
      \op{L}{_t^\mu}^\dag
    \ket{\k}
    \bra{\k}
      \op{L}{_t^\mu}
    \ket{\n}
  \notag
  \\
  &=
  \sigma(t)
  \sum_{p, m', n'}
    \left(
      \sqrt{
        Q^\text{tok}(m', n')^2
      }
      \begin{array}{ccc}
          \bra{m_1} I \ket{k_1} 
          & 
          \bra{k_1} I \ket{n_1}
          &
          =
          {\color{red} 
          1_{m_1 n_1}
          }
          1_{m_1 k_1}
        \\
          \times & \times
        \\
          \vdots & \vdots & \vdots
        \\
          \times & \times
        \\
          \braket{m_p|n'} 
          \braket{m'|k_p}
          &
          \braket{k_p|m'} 
          \braket{n'|n_p}
          &
          =
          {\color{red}
            1_{m_p n_p}
          }
          1_{m_p n'}
          1_{k_p m'}
        \\
          \times & \times
        \\
          \vdots & \vdots & \vdots
        \\
          \times & \times
        \\
          \bra{m_\NumToken} I \ket{k_\NumToken}
          &
          \bra{k_\NumToken} I \ket{n_\NumToken}
          &
          =
          {\color{red}
            1_{m_\NumToken n_\NumToken}
          }
          1_{m_\NumToken n_\NumToken}
      \end{array}
      \right)
  \notag
  \\
  &
  \propto 
  {\color{red}
  1_{\m\n}}
  =
  0 \quad (\text{for } \m \neq \n)
  .
  \notag
\end{align}
Therefore, we can apply \cref{thm:pauli} and get the reduction of Lindblad equation to the diffusion equation.
In the end, we calculate the value of $Q_t(\m, \n)$ in \cref{thm:pauli} with the jump operators in \cref{eq:constraint_appdx}
\begin{align}
  \sum_\mu 
    | \bra{\m}
      \op{L}{_t^\mu}
      \ket{\n} |^2
  &=
  \sum_{p=1}^\NumToken
    \sum_{m', n'}
      \begin{array}{c}
          | \bra{m_1} I \ket{n_1} |^2
        \\
          \times 
        \\
          \vdots 
        \\
          \times 
        \\
          | \bra{m_p} \op{L}{_{m', n'}^\text{$(p)$tok}} \ket{n_p} |^2
        \\
          \times 
        \\
          \vdots 
        \\
          \times 
        \\
          | \bra{m_\NumToken} I \ket{n_\NumToken} |^2
      \end{array}
  =
  \sum_{p=1}^\NumToken
    1_{\m_{-p} \n_{-p}}Q^\text{tok}(m_p,n_p),
\end{align}
which proves \cref{thm:dlm_as_oqs}.

\subsection{Proof of \cref{thm:tilde_L}}

First, we prove the time-reversal Hamiltonian $\tilde{H}$ defined in \cref{thm:thm_time_rev} is zero.

\begin{mytheorem}{}{zero_hamiltonian_rev}
  Under the following conditions:
  \begin{align}
    &
    \op{H} = 0,
    \label{eq:zero_hamiltonian}
    \\
    &
    \op{\rho}_t \text{ is diagonal},
    \label{eq:rho_diag}
    \\
    &
    \forall \k \in \vocab^\NumToken,\quad 
    0 = \sum_{\mu} \bra{\m} \op{L}{_t^\mu}^\dag \ket{\k}\bra{\k} \op{L}{_t^\mu} \ket{\n}
    \quad 
    \text{for } \m \neq \n,
    \label{eq:vanishing}
  \end{align}
  then the time-reversal Hamiltonian 
  \begin{align}
    \tilde{H}
    &=
    -\frac{1}{2} \left(
        \sqrt{\op{\rho}_t} \op{H} 
        + i\frac{d \sqrt{\op{\rho}_t}}{dt} 
        + \frac{i}{2}\sum_\mu \sqrt{\op{\rho}_t} \op{L}{_t^\mu}^\dagger \op{L}{_t^\mu} 
        \right)\frac{1}{\sqrt{\op{\rho}_t}}
        +(\text{1st term})^\dag
  \end{align}
  is zero.
\end{mytheorem}

\begin{proof}
  From \cref{eq:zero_hamiltonian} in the conditions, we have 
  \begin{align}
    \tilde{H}
    &=
    -\frac{i}{2} \left(
        + \frac{d \sqrt{\op{\rho}_t}}{dt} \frac{1}{\sqrt{\op{\rho}_t}}
        + \frac{1}{2}\sum_\mu \sqrt{\op{\rho}_t} \op{L}{_t^\mu}^\dagger \op{L}{_t^\mu} \frac{1}{\sqrt{\op{\rho}_t}}
        \right)
        +(\text{1st term})^\dag
    ,
  \end{align}
  and it is sufficient to show that the operator inside $()$ is Hermite.
  From \cref{eq:rho_diag} in the given conditions, we have 
  \begin{align}
    \frac{d \sqrt{\op{\rho}_t}}{dt} \frac{1}{\sqrt{\op{\rho}_t}}
    &=
    \sum_{\m \in \vocab^\NumToken}
      \frac{
        \frac{d}{dt} \sqrt{p_t(\m)}
      }{
        \sqrt{p_t(\m)}
      }
      \ket{\m}
      \bra{\m}
      ,
  \end{align}
  which is Hermite.
  To demonstrate the Hermiticity of 2nd term, we sum over $\k$ for \cref{eq:vanishing} that gives
  \begin{align}
    0
    &=
    \sum_\mu   
      \bra{\m}
        \op{L}{_t^\mu}^\dag 
        \op{L}{_t^\mu}
      \ket{\n} 
    .
  \end{align}
  This means the Hermitian operator $\sum_\mu \op{L}{_t^\mu}^\dag  \op{L}{_t^\mu}$ is diagonal.
  Then, the 2nd term is Hermite 
  \begin{align}
    \sum_\mu 
        \sqrt{\op{\rho}_t}
        \op{L}{_t^\mu}^\dag 
        \op{L}{_t^\mu}
        \frac{1}{
            \sqrt{\op{\rho}_t}
        }
    &=
    \underbrace{\sqrt{\op{\rho}_t}}_\text{diagonal}
    \underbrace{\left(\sum_\mu \op{L}{_t^\mu}^\dag \op{L}{_t^\mu}\right)}_\text{diagonal}
    \underbrace{\frac{1}{\sqrt{\op{\rho}_t}}}_\text{diagonal}
    =\underbrace{\sum_\mu \op{L}{_t^\mu}^\dag \op{L}{_t^\mu}}_\text{Hermite}
    ,
  \end{align}
  where we used that diagonal operators commute.
\end{proof}

Next, we show the time-reversal jump operators satisfies the condition for \cref{thm:pauli}.
We use $\op{\rho}_t = \sum_{\m \in \vocab^\NumToken} p_t(\m) \ket{\m}\bra{\m}$ in the following calculations.
\begin{align}
  \sum_\mu
    \bra{\m}
      \tilde{L}{_t^\mu} 
    \ket{\k}
    \bra{\k}
      \tilde{L}{_t^\mu}^\dag 
    \ket{\n}
  &=
  \frac{
    \sqrt{
      p_t(\m)
      p_t(\n)
      }
  }{
    p_t(\k)
  }
  \underbrace{
  \sum_\mu 
    \bra{\m}
      \op{L}{_t^\mu}^\dag
    \ket{\k}
    \bra{\k}
      \op{L}{_t^\mu} 
    \ket{\n}
  }_{=0 \ \forall \k, \m, \n \ \m \neq \n}
  .
\\
  \sum_\mu
    \bra{\m}
      \tilde{L}{_t^\mu}^\dag
    \ket{\k}
    \bra{\k}
      \tilde{L}{_t^\mu}
    \ket{\n}
  &=
  \frac{
    p_t(\k)
  }
  {
    \sqrt{
      p_t(\m)
      p_t(\n)
      }
  }
  \underbrace{
  \sum_\mu 
    \bra{\m}
      \op{L}{_t^\mu}
    \ket{\k}
    \bra{\k}
      \op{L}{_t^\mu}^\dag
    \ket{\n}
  }_{=0 \ \forall \k, \m, \n \ \m \neq \n}
  .
\end{align}
Therefore, we can apply \cref{thm:pauli} to the time-reversal Lindblad equation.

We can prove the \cref{thm:tilde_L} as a corollary of this theorem.
The remaining problem is the explicit form of the time-reversal jump operators, which is shown in the next statement, and it proves \cref{thm:tilde_L}.

\begin{mytheorem}{Time-reversal Jump Operators}{tilde_L-appdx}
  Under our setting in \cref{eq:full_jump_operators}, we have 
  \begin{align*}
    {\small
    \tilde{L}{_t^{\mu=(p, m, n)}}
    =
      \sum_{\k}
        \begin{array}{c}
          \ket{k_1}\bra{k_1}
          \\
          \otimes 
          \\
          \vdots 
          \\
          \otimes 
          \\
          \sqrt{\sigma(t)}
          \tilde{L}{_{m, n, t}^\text{$(p)$tok}}(\k)
          \\
          \otimes 
          \\
          \vdots 
          \\
          \otimes 
          \\
          \ket{k_\NumToken}\bra{k_\NumToken}
        \end{array}
    }
        \quad 
          \text{where }
          \tilde{L}{_{m, n, t}^\text{$(p)$tok}}(\k)
          =
          \sqrt{
              \tilde{Q}^\text{$(p)$tok}(n, \k)
            }
            1_{m k_p}
            \ket{n}\bra{k_p}
          .
  \end{align*}
\end{mytheorem}

\begin{proof}
  By just substituting our conditions $\op{L}{_t^{\mu=(p, m, n)}} = \sqrt{\sigma(t)} I \otimes \cdots \otimes I \otimes {\op{L}_{m, n}^\text{$(p)$tok}}\otimes I \otimes \cdots \otimes I$, $\op{\rho}_t = \sum_{\k \in \vocab^\NumToken} p_t(\k) \ket{\k}\bra{\k}$ to the explicit form of $\tilde{L}{_t^\mu}$ defined in \cref{eq:backward_operators}, we get
  {\small
  \begin{align*}
    &
      \sqrt{\op{\rho}_t} 
      \op{L}{_t^{\mu=(p, m, n)}}^\dag
      \frac{1}{
        \sqrt{\op{\rho}_t} 
      }
    =
      \sum_{\l, \k \in \vocab^\NumToken}
      \sqrt{
        \frac{
          p_t(\l)
        }{
          p_t(\k)
        }
        \sigma(t)
        }
      \ket{\l}
      \bra{\k}
      1_{\l_{-p} \k_{-p}}
      {
        \bra{l_p}
          \op{L}{_{m, n}^\text{$(p)$tok}}^\dag
        \ket{k_p}
      }
    \\
    &=
    \sum_{\k \in \vocab^\NumToken}
      \left(
      \begin{array}{lc}
        &\ket{k_1}\bra{k_1}
        \\
        &\otimes
        \\
        &\vdots 
        \\
        &\otimes
        \\
        \sum_{l_p \in \vocab}
          \sqrt{
            Q^\text{tok}(m, n)
            \frac{
              p_t(\k_{p \to l_p})
            }{
              p_t(\k)
            }
            \sigma(t)
          }
          1_{l_p n} 1_{m k_p}
          &
          \ket{l_p}\bra{k_p}
        \\
        &\otimes 
        \\
        &\vdots 
        \\
        &\otimes
        \\
        &\ket{k_\NumToken}\bra{k_\NumToken}
      \end{array}
      \right)
    \\
    &=
    \sum_{\k\in \vocab^{\NumToken}}
          \sqrt{
            \underbrace{
              Q^\text{tok}(k_p, n)
              \frac{
                p_t(\k_{p \to n})
              }{
                p_t(\k)
              }
            }_{
              = \tilde{Q}^\text{$(p)$tok}(n, \k)
            }
            \sigma(t)
          }
          1_{m k_p}
      \begin{array}{lc}
        &\ket{k_1}\bra{k_1}
        \\
        &\otimes
        \\
        &\vdots 
        \\
        &\otimes
        \\
          &
          \ket{n}\bra{k_p}
        \\
        &\otimes 
        \\
        &\vdots 
        \\
        &\otimes
        \\
        &\ket{k_\NumToken}\bra{k_\NumToken}
      \end{array}
    .
  \end{align*}
  }
\end{proof}

The operator $\tilde{L}{_{m, n, t}^\text{$(p)$tok}}(\k)$ satisfies the following conditions.

\begin{mytheorem}{}{prod}
  The following holds:
  \begin{align*}
    \tilde{L}{_{m, n, t}^\text{$(p)$tok}}(\k)^\dag 
    \tilde{L}{_{m, n, t}^\text{$(p)$tok}}(\l)
    &
    \propto 
    1_{k_p l_p}
    ,
    \quad
    \tilde{L}{_{m, n, t}^\text{$(p)$tok}}(\k)^\dag 
    \tilde{L}{_{m, n, t}^\text{$(p)$tok}}(\k)
    =
        \tilde{Q}^\text{$(p)$tok}(n, \k)
        1_{m k_p}
      \ket{k_p}
      \bra{k_p}
    .
  \end{align*}
  
\end{mytheorem}

\begin{proof}
  By direct calculation,
  \begin{align*}
    &
    \tilde{L}{_{m, n, t}^\text{$(p)$tok}}(\k)^\dag 
    \tilde{L}{_{m, n, t}^\text{$(p)$tok}}(\l)
    =
      \sqrt{
        \tilde{Q}^\text{$(p)$tok}(n, \k)
        \tilde{Q}^\text{$(p)$tok}(n, \l)
      }
      \underbrace{
      1_{m k_p}
      1_{m l_p}
      }_{
        1_{k_p l_p}
        1_{m k_p}
      }
      \ket{k_p}
      \underbrace{
        \braket{n|n}
      }_{1}
      \bra{l_p}
    ,
  \end{align*}
  which shows the statement.
\end{proof}

\subsection{Sampling by Stochastic Unraveling and Measurement}
\label{seq:prf_sampling}

In the main text, we mentioned the fact that the diagram in \cref{fig:unraveling_projection} commutes without providing a proof. Here, we state it as a theorem and provide a proof.

\begin{mytheorem}{Sampling based on Stochastic Unraveling}{sampling_appdx}
    Suppose $t \geq s$. Given a state $\ket{\n}$ at time $s$, we consider sampling $\m$ at time $t$ by first performing stochastic unraveling from $s$ to $t$ and then projective measurement at $t$. Specifically:
    \begin{enumerate}
    \item Iteratively update a pure state from $\ket{\n}$ to $\ket{\bpsi_{t}}$ using one of unraveling operators.
    \item Sample sentence $\m$ from the probability $P_\text{proj}\big(\m \big| \bpsi_{t}\big) = |\braket{\m|\bpsi_{t}}|^2$.
    \end{enumerate}

    This procedure yields a sample $\m$ distributed according to the conditional probability $p_{t|s}(\m|\n)$ defined in \cref{eq:solution-appdx}.
\end{mytheorem}

\begin{proof}
  The joint probability for realizing $\m, \ket{\bpsi_{t}}$ during this 2 steps is 
  \begin{align}
      P_{t|s}\big(
        \m, {\bpsi_{t}}
        \big| 
        \n
        \big)
      &=
      \underbrace{
        P_\text{proj}\big(\m_t \big| {\bpsi_{t}} \big)
      }_{\text{Process 2}}
      \underbrace{
      P_{t|s}^\text{unraveling} \big({\bpsi_{t}} \big| {\n} \big)
      }_{\text{Process 1}}
      , \label{eq:2steps}
  \end{align}
  its marginal to $\m$ is 
  \begin{align}
      P_{t|s}(\m|\n)
      &=
      \mathop{\mb{E}}\limits_{{\bpsi_{t}} \sim P_{t|s}^\text{unraveling}(\cdot | \n)} \Big[
        P_\text{proj}\big(\m \big| \bpsi_{t} \big)
      \Big]
      =
      \mathop{\mb{E}}\limits_{{\bpsi_{t}} \sim P_{t|s}^\text{unraveling}(\cdot | \n)} \Big[
        \braket{\m|\bpsi_{t}}
        \braket{\bpsi_{t}|\m}
      \Big]
      \notag \\
      &=
      \bra{\m}
      \underbrace{
      \left(
        \mathop{\mb{E}}\limits_{{\bpsi_{t}} \sim P_{t|s}^\text{unraveling}(\cdot | \n)} \Big[
          \ket{\bpsi_{t}}\bra{\bpsi_{t}}
        \Big]
      \right)
      }_{=: \op{\rho}_{t|s}(\n)}
      \ket{\m}
      . \label{eq:marginal_2steps}
  \end{align}
  Now, thanks to \cref{thm:qjump}, the density operator $\op{\rho}_{t|s}(\n)$ satisfies the Lindblad equation for $t$.
  In addition, under the setting of \cref{thm:dlm_as_oqs}, the Lindblad equation reduces to the equation for $\bra{\m} \op{\rho}_{t|s}(\n) \ket{\m}$ with the diffusion equation in \cref{eq:diff_eq}, which means \cref{eq:marginal_2steps} satisfies the diffusion equation in \cref{eq:diff_eq}.
  In addition, $\op{\rho}_{t|s}(\n)$ reduces to $\ket{\n}\bra{\n}$ at time $t=s$, which means 
  \begin{align}
    P_{t=s|s}(\m|\n)
    &=
    \braket{\m|\n}
    \braket{\n|\m}
    =
    1_{\m=\n},
  \end{align}
  and it is identical to the initial condition of $p_{t=s|s}(\m|\n)$ defined in \cref{eq:solution-appdx}.
  We show $P_{t|s}(\m|\n)$ satisfies: 1) same ODE to $p_{t|s}(\m|\n)$, and 2) same initial condition to $p_{t|s}(\m|\n)$, which means $P_{t|s} = p_{t|s}$ by uniqueness of the ODE solution.
\end{proof}

\subsection{Explicit Forms of $\tau$-leaping Quantum States}

This section provides the necessary information for executing simulations.
Here, token-level jump operators must satisfy the following format as in \cref{eq:mLk-main}:
\begin{align}
    \op{L}{_{m, n}^\text{tok}}
  =
  \sqrt{
    Q^\text{tok}(m, n)
  }
  \ket{m}\bra{n}
  ,
  \quad
  (m \neq n)
  ,
  \label{eq:mLk-appdx}
\end{align}

\subsubsection{Explicit forms of Token-level Time-reversal Effective Hamiltonian}

We calculate the time-reversal effective Hamiltonian within our setup.

\begin{mytheorem}{Time-reversal effective Hamiltonian}{H_eff_rev}
  We have 
  \begin{align*}
    \tilde{H}_\text{eff}
    &=
    \sum_{p=1}^\NumToken
    \sum_{\k}
      \begin{array}{c}
        \ket{k_1}\bra{k_1}
        \\
        \otimes 
        \\
        \vdots 
        \\
        \otimes 
        \\
        \sigma(t)
        \tilde{H}{_t^\text{$(p)$tok}}(\k)
        \\
        \otimes 
        \\
        \vdots 
        \\
        \otimes 
        \\
        \ket{k_\NumToken}\bra{k_\NumToken}
      \end{array}
      \
      \text{where}
      \left\{
      \begin{array}{l}
        \tilde{H}{_t^\text{$(p)$tok}}(\k)
        \displaystyle
        =
          -\frac{i}{2}
          \sum_{m, n}
            \tilde{L}{_{m, n, t}^\text{$(p)$tok}}(\k)^\dag 
            \tilde{L}{_{m, n, t}^\text{$(p)$tok}}(\k)
        \\
        \displaystyle
        =
          -\frac{i}{2}
            \sum_{m, n}
              \tilde{Q}^\text{$(p)$tok}(n, \k)
              1_{m k_p}
              \ket{k_p}
              \bra{k_p}
      \end{array}
      \right.
    .
  \end{align*}  
\end{mytheorem}

\begin{proof}
  \begin{align*}
      \tilde{H}_\text{eff}
      \notag
    &=
      \underbrace{
        \tilde{H}
      }_{0 \ \text{by \cref{thm:zero_hamiltonian_rev}}}
      - 
      \frac{i}{2}
      \sum_\mu 
        \underbrace{
          \tilde{L}{_t^\mu}^\dag
        }_{\text{\cref{thm:tilde_L-appdx}}}
        \underbrace{
          \tilde{L}{_t^\mu}
        }_{\text{\cref{thm:tilde_L-appdx}}}
    \\
    &=
      - 
      \frac{i}{2}
          \sum_{p=1}^\NumToken
          \sum_{m, n}
          \sum_{\k, \l}
          \begin{array}{cc}
              \ket{k_1}
              \braket{k_1|l_1} 
              \bra{l_1}
            \\
              \otimes 
            \\
              \vdots 
            \\
              \otimes 
            \\
            \sigma(t)
              \tilde{L}{_{m, n, t}^\text{$(p)$tok}}(\k)^\dag 
              \tilde{L}{_{m, n, t}^\text{$(p)$tok}}(\l)
            \\
              \otimes 
            \\
              \vdots 
            \\
              \otimes 
            \\
              \ket{k_\NumToken}
              \braket{k_\NumToken|l_\NumToken}
              \bra{l_\NumToken}
          \end{array}
      .
  \end{align*}
  Thanks to $\braket{k|l} = 1_{k l}$, we get $1_{\k_{-p} \l_{-p}}$ in the summation.
  In addition, $\tilde{L}{_{m, n, t}^\text{$(p)$tok}}(\k)^\dag\tilde{L}{_{m, n, t}^\text{$(p)$tok}}(\l) \propto 1_{k_p l_p}$ by \cref{thm:prod}, which reduces the sum only for $\k = \l$.
  Consequently, the sum turns into a single expression as follows:
  \begin{align*}
    \tilde{H}_\text{eff}
    &=
      - 
      \frac{i}{2}
          \sum_{p=1}^\NumToken
          \sum_{m, n}
          \sum_{\k}
          \begin{array}{c}
              \ket{k_1}\bra{k_1} 
            \\
              \otimes 
            \\
              \vdots 
            \\
              \otimes 
            \\
            \sigma(t)
              \tilde{L}{_{m, n, t}^\text{$(p)$tok}}(\k)^\dag 
              \tilde{L}{_{m, n, t}^\text{$(p)$tok}}(\k)
            \\
              \otimes 
            \\
              \vdots 
            \\
              \otimes 
            \\
              \ket{k_\NumToken}\bra{k_\NumToken}
          \end{array}
  \end{align*}
  Finally, we compute the $p$-component by \cref{thm:prod}:
  \begin{align*}
    \underbrace{
      -\frac{i}{2}
        \sum_{m, n}
        \tilde{L}{_{m, n, t}^\text{$(p)$tok}}(\k)^\dag 
        \tilde{L}{_{m, n, t}^\text{$(p)$tok}}(\k)
    }_{
      =: 
      \tilde{H}{_t^\text{$(p)$tok}}(\k)
    }
    &=
      -\frac{i}{2}
        \sum_{m, n}
        \tilde{Q}^\text{$(p)$tok}(n, \k)
        1_{m k_p}
        \ket{k_p}
        \bra{k_p}
  \end{align*}
\end{proof}

As a lemma:

\begin{mytheorem}{}{H-lem}
  The following holds.

  \begin{align*}
    \ket{k_p}\bra{k_p}
    -
    i 
      \tilde{H}{_t^\text{$(p)$tok}}(\k) \Delta t
    &=
    \sqrt{
      1
      +
      \sum_m 
          \tilde{Q}^\text{$(p)$tok}(k_p, \k)
          \sigma(t)
          1_{m k_p}
          \Delta t
    }
    \ket{k_p}\bra{k_p}
    ,
  \end{align*}

  up to $O(\Delta t^2)$, where 

  \begin{align}
    \tilde{Q}^\text{$(p)$tok}(k_p, \k)
    &=
    -
    \sum_{n \neq k_p} 
      \tilde{Q}^\text{$(p)$tok}(n, \k)
    .
    \label{eq:s-sum}
  \end{align}

\end{mytheorem}

\begin{proof}
  
  By direct calculation,
  
  \begin{align*}
    \ket{k_p}\bra{k_p}
    -
    i 
    \underbrace{
      \tilde{H}{_t^\text{$(p)$tok}}(\k) 
    }_\text{\cref{thm:H_eff_rev}}
    \Delta t
    &=
      \ket{k_p}\bra{k_p}
      -
      i 
      \left(
        - \frac{i}{2}
          \sum_{m, n}
            \tilde{Q}^\text{$(p)$tok}(n, \k)
              \sigma(t)
            1_{m k_p}
            \ket{k_p}
            \bra{k_p}
      \right)
      \Delta t
    \\
    &=
      \left(
        1
        -
        \frac{1}{2}
        \underbrace{
        \sum_{m, n}
            \tilde{Q}^\text{$(p)$tok}(n, \k)
        }_{
          \sum_m 
          - \tilde{Q}^\text{$(p)$tok}(k_p, \k)
        }
              \sigma(t)
            1_{m k_p}
        \Delta t
      \right)
      \ket{k_p}\bra{k_p}
    \\
    &=
    \sqrt{
      1
      +
      \sum_m 
          \tilde{Q}^\text{$(p)$tok}(k_p, \k)
          \sigma(t)
          1_{m k_p}
      \Delta t
    }
    \ket{k_p}\bra{k_p}
    +
    O(\Delta t^2)
  \end{align*}
\end{proof}

Here, we shall refer to $\tilde{H}{_t^\text{$(p)$tok}}(\k)$ as the token-level Hamiltonian and $\tilde{L}{_{m, n, t}^\text{$(p)$tok}}(\k)$ as the token-level jump operator. 

\subsubsection{$\tau$-leaping unraveling}
\label{sec:tau-unraveling}

The naive unravelings are exact up to order dt, however, when attempting to use these methods for inferring language diffusion models, they are inefficient because of the design of $\op{L}{_t^\mu}$.
With this configuration, only one token can be jumped in a single step, making it inefficient.
One approach to solving this problem is the \textit{$\tau$-leaping-like} method.
We consider the quantum-jump unraveling in \cref{eq:qjump_apdx}.

Under this decomposition, we can introduce the corresponding concept for the unraveling of the $\tau$-leaping sampling:

\begin{mytheorem}{$\tau$-leaping of quantum jump}{tau-qjump}
  We define
  \begin{align}
    &
    \tilde{U}_{t - \Delta t | t}^\text{$\tau$-leap}
    \ket{\bpsi_t} 
    \notag
    \\
    &=
      \sum_{\k \in \vocab^\NumToken}
          \bigotimes_{p=1}^N
            \Big(
              \ket{k_p}\bra{k_p}
              - i \sigma(t)\tilde{H}_{t}^\text{$(p)$tok}(\k) \Delta t
            \notag
            \\ & \qquad \qquad \qquad
              +
              \sum_{m, n} 
                (
                  \sqrt{\sigma(t)}
                  \tilde{L}{_{m, n, t}^\text{$(p)$tok}}(\k)
                  -
                  \ket{k_p}\bra{k_p}
                )
                \Delta N_t^{m, n|p}
            \Big) 
      \ket{\bpsi_t}
          / (\text{numerator norm})
      ,
  \end{align}
  where $\Delta N_t^{m, n|p}$ is a new process defined by 
  \begin{align}
    \text{Prob}\Big(
      \forall (m, n), 
      {\color{blue}\Delta N_t^{m, n|p}
      =
      0}
      \Big)
    &= 1 - \sum_{m, n}\braket{\tilde{L}{_t^{\mu=(p, m, n)}}^\dag \tilde{L}{_t^{\mu=(p, m, n)}}}_{\bpsi_t} \Delta t
    \\
    \text{Prob}\Big(
      \exists^1 (m, n), 
      {\color{blue}\Delta N_t^{m, n|p}
      =
      1}
    \Big)
    &= \braket{\tilde{L}{_t^{\mu=(p, m, n)}}^\dag \tilde{L}{_t^{\mu=(p, m, n)}}}_{\bpsi_t} \Delta t
    .
  \end{align}
  Then, the operation becomes stochastic unraveling up to first order in $\Delta t$. 
  \begin{align}
    \mb{E}_{\Delta N} \left[
      \tilde{U}_{t - \Delta t | t}^\text{$\tau$-leap}
      \ket{\bpsi_t}
      \bra{\bpsi_t}
      \tilde{U}{_{t - \Delta t | t}^\text{$\tau$-leap}}^\dag
    \right]
    &=
    \ket{\bpsi_t}
    \bra{\bpsi_t}
    +
    \tilde{\mc{L}}_t (\ket{\bpsi_t}\bra{\bpsi_t}) \Delta t
    +
    O(\Delta t^2)
    ,
  \end{align}
  with the reverse-time Lindblad operator $\tilde{\mc{L}}_t$.
\end{mytheorem}

\begin{proof}
  The definition of Poisson noise differs from conventional quantum jumps because it decomposes the jumps by token position.  
  Therefore, the direct method of calculating the expectation value of the density matrix would be the most straightforward approach.

  To conduct this, explicitly writing the numerator component of the vector state yields the following:
  \begin{align*}
    &\sum_{\k \in \vocab^\NumToken}
          \bigotimes_{p=1}^N
            \Big(
              \ket{k_p}\bra{k_p}
              - i \sigma(t)\tilde{H}_{t}^\text{$(p)$tok}(\k) {\color{red}\Delta t}
              \\ & \quad 
              +
              \sum_{m, n} 
                (
                  \sqrt{\sigma(t)}
                  \tilde{L}{_{m, n, t}^\text{$(p)$tok}}(\k)
                  -
                  \ket{k_p}\bra{k_p}
                )
                {\color{blue}\Delta N_t^{m, n|p}}
            \Big) 
      \ket{\bpsi_t}
    \notag \\
    &=
    \sum_{\k}
        \bigotimes_{p=1}^N
        \left\{
          \begin{array}{cl}
            \ket{k_p}\bra{k_p} - i \sigma(t)\tilde{H}{_t^\text{$(p)$tok}}(\k) {\color{red}\Delta t}
            &
            \text{w/ }
            \text{Pr}
            \left(
            \text{$\forall (m, n), {\color{blue} \Delta N{_t^{m, n|p}}=0}$}
            \right)
            \\
            \sqrt{\sigma(t)}
            \tilde{L}{_{m_p, n_p, t}^\text{$(p)$tok}}(\k) 
            &
            \text{w/ }
            \text{Pr}
            \left(
            \text{$\exists^1 (m_p, n_p), {\color{blue} \Delta N{_t^{m_p, n_p|p}}=1}$}
            \right)
          \end{array}
        \right\}
      \ket{\bpsi_t}
  \end{align*}
  In this expression, operators acting on individual tokens in their respective Hilbert spaces $\mc{H}_\vocab$ follow independent probability distributions for each token.
  To obtain the full form of the time evolution, we need to divide by the norm of each operator's action.
  {\small
  \begin{align*}
    &
    \text{norm}
    \\
    &=
      \bra{\bpsi_t}
        \sum_{\k}
          \bigotimes_{p=1}^N
            \left\{
              \begin{array}{cl}
                \left(
                \ket{k_p}\bra{k_p} + i \sigma(t)\tilde{H}{_t^\text{$(p)$tok}}(\k)^\dag {\color{red}\Delta t}
                \right)
                \left(
                \ket{k_p}\bra{k_p} - i \sigma(t)\tilde{H}{_t^\text{$(p)$tok}}(\k) {\color{red}\Delta t}
                \right)
                \\
                \sigma(t)
                \tilde{L}{_{m_p, n_p, t}^\text{$(p)$tok}}(\k)^\dag 
                \tilde{L}{_{m_p, n_p, t}^\text{$(p)$tok}}(\k) 
              \end{array}
            \right\}
      \ket{\bpsi_t}
  \end{align*}
  }
  Now, we use the expression of $\tilde{H}{_t^\text{$(p)$tok}}(\k)$ in \cref{thm:H_eff_rev} as 
  \begin{align*}
    &=
      \bra{\bpsi_t}
        \sum_{\k}
          \bigotimes_{p=1}^N
            \left\{
              \begin{array}{cl}
                \left(
                \ket{k_p}\bra{k_p} 
                - 
                  \sum_{m, n}
                    \sigma(t)
                    \tilde{L}{_{m, n, t}^\text{$(p)$tok}}(\k)^\dag
                    \tilde{L}{_{m, n, t}^\text{$(p)$tok}}(\k)
                {\color{red}\Delta t}
                \right)
                \\
                \sigma(t)
                \tilde{L}{_{m_p, n_p, t}^\text{$(p)$tok}}(\k)^\dag 
                \tilde{L}{_{m_p, n_p, t}^\text{$(p)$tok}}(\k) 
              \end{array}
            \right\}
      \ket{\bpsi_t}
    ,
  \end{align*}
  which means if all tokens show $\Delta N_t^{m, n|p} = 0$, the corresponding norm of the vector is 
  \begin{align}
    &
    \text{norm($\forall p, \forall (m, n), {\color{blue} \Delta N{_t^{m, n|p}}=0}$)}
    \notag
    \\
    &=
      \bra{\bpsi_t}
        \sum_{\k}
          \bigotimes_{p=1}^N
                \left(
                \ket{k_p}\bra{k_p} 
                - 
                  \sum_{m, n}
                  \sigma(t)
                    \tilde{L}{_{m, n, t}^\text{$(p)$tok}}(\k)^\dag
                    \tilde{L}{_{m, n, t}^\text{$(p)$tok}}(\k)
                {\color{red}\Delta t}
                \right)
      \ket{\bpsi_t}
    \notag
    \\
    &=
    {\small
    \bra{\bpsi_t}
      \left(
        1
        - 
        \sum_p 
          \sum_{m, n} 
          \sum_{\k}
            \ket{k_1}\bra{k_1} \otimes \cdots \otimes 
            \sigma(t)
            \tilde{L}{_{m, n, t}^\text{$(p)$tok}}(\k)^\dag 
            \tilde{L}{_{m, n, t}^\text{$(p)$tok}}(\k)  
            \otimes \cdots \otimes \ket{k_\NumToken}\bra{k_\NumToken}
          {\color{red} \Delta t}
      \right)
    \ket{\bpsi_t}
    }
    \notag
    \\
    &=
    \prod_{p}
    \left(
      1 
      - 
      \sum_{m, n} 
        \braket{
          \tilde{L}{_t^{\mu=(p, m, n)}}^\dag 
          \tilde{L}{_t^{\mu=(p, m, n)}}
        }_{
          \bpsi_t
        }
        {\color{red} \Delta t}
    \right)
    ,
    \label{eq:norm-all0}
  \end{align}
  where we used $\braket{\bpsi_t|\bpsi_t} = 1$.
  If only with $p$-th token $\Delta N_t^{m_p, n_p|p} = 1$, the norm is 
  \begin{align}
    &
    \text{norm($\exists^1 p, (m_p, n_p), {\color{blue} \Delta N{_t^{m_p, n_p|p}}=1}$)}
    \notag
    \\
    &=
    \bra{\bpsi_t}
        \sum_{\k}
        \bigotimes_{q=1}^N
        \left(
        \begin{array}{cl}
          \ket{k_q}\bra{k_q} 
          - \sum_{m, n} 
            \sigma(t)
              \tilde{L}{_{m, n, t}^\text{$(q)$tok}}(\k)^\dag 
              \tilde{L}{_{m, n, t}^\text{$(q)$tok}}(\k)  
            {\color{red} \Delta t}
          &
          (q \neq p)
          \\
          \sigma(t)
          \tilde{L}{_{m_p, n_p,t}^\text{$(p)$tok}}(\k)^\dag
          \tilde{L}{_{m_p, n_p,t}^\text{$(p)$tok}}(\k)
          &
          (q = p)
        \end{array}
      \right)
    \ket{\bpsi_t}
    \notag
    \\
    &=
    {\small
    \bra{\bpsi_t}
      \left(
        \sum_{\k}
          \ket{k_1}\bra{k_1}
          \otimes 
          \cdots 
          \otimes
          \sigma(t)
          \tilde{L}{_{m_p, n_p,t}^\text{$(p)$tok}}(\k)^\dag
          \tilde{L}{_{m_p, n_p,t}^\text{$(p)$tok}}(\k)
          \otimes 
          \cdots 
          \otimes 
          \ket{k_{\NumToken}}\bra{k_{\NumToken}}
        +
        O({\color{red} \Delta t})
      \right)
    \ket{\bpsi_t}
    }
    \notag
    \\
    &=
    \braket{
      \tilde{L}{_t^{\mu=(p, m_p, n_p)}}^\dag
      \tilde{L}{_t^{\mu=(p, m_p, n_p)}}
    }_{\bpsi_t}
    + O({\color{red} \Delta t})
    .
    \label{eq:norm-one1}
  \end{align}
  With these preparations complete, we can calculate the expectation value of the density matrix up to first order in ${\color{red}\Delta t}$ as follows:
  \begin{align*}
    &
    \mb{E}_{\Delta N} \left[
      \tilde{U}_{t - \Delta t | t}^\text{$\tau$-leap}
      \ket{\bpsi_t}
      \bra{\bpsi_t}
      \tilde{U}{_{t - \Delta t | t}^\text{$\tau$-leap}}^\dag
    \right]
    \\
    &=
      \sum_{\k, \l}
      \bigotimes_{p=1}^N
        \left(
        \ket{k_p}\bra{k_p} - i \sigma(t)\tilde{H}{_t^\text{$(p)$tok}}(\k) {\color{red} \Delta t}
        \right)
      \ket{\bpsi_t}
      \bra{\bpsi_t}
      \bigotimes_{q=1}^N
        \left(
        \ket{l_q}\bra{l_q} + i \sigma(t)\tilde{H}{_t^\text{$(q)$tok}}(\l)^\dag {\color{red} \Delta t}
        \right)
    \\
    &\quad 
    \times 
    \underbrace{
      \prod_{p=1}^\NumToken
      \text{Prob}
            \left(
            \text{$\forall (m, n), {\color{blue} \Delta N{_t^{m, n|p}}=0}$}
            \right)
    }_{
      \prod_{p}
        \left(
          1 
          - 
          \sum_{m, n} 
            \braket{
              \tilde{L}{_t^{\mu=(p, m, n)}}^\dag 
              \tilde{L}{_t^{\mu=(p, m, n)}}
            }_{
              \bpsi_t
            }
            {\color{red} \Delta t}
        \right)
    }
    / 
    \underbrace{
      \text{norm($\forall p, \forall (m, n), {\color{blue} \Delta N{_t^{m, n|p}}=0}$)}
    } 
    _{
      \prod_{p}
        \left(
          1 
          - 
          \sum_{m, n} 
            \braket{
              \tilde{L}{_t^{\mu=(p, m, n)}}^\dag 
              \tilde{L}{_t^{\mu=(p, m, n)}}
            }_{
              \bpsi_t
            }
            {\color{red} \Delta t}
        \right)
    }
    \\
    & +
    \sum_{p, m_p, n_p}
        \sum_{\k, \l}
        \bigotimes_{q=1}^N
        \left(
        \begin{array}{cl}
          \ket{k_q}\bra{k_q} - i \sigma(t) \tilde{H}{_t^\text{$(q)$tok}}(\k) {\color{red} \Delta t}
          &
          (q \neq p)
          \\
          \sqrt{\sigma(t)}
          \tilde{L}{_{m_p, n_p, t}^\text{$(p)$tok}}(\k)
          &
          (q = p)
        \end{array}
        \right)
      \ket{\bpsi_t}
      \\ & \qquad \qquad \qquad
      \bra{\bpsi_t}
        \bigotimes_{r=1}^N
        \left(
        \begin{array}{cl}
          \ket{l_r}\bra{l_r} + i \sigma(t)\tilde{H}{_t^\text{$(r)$tok}}(\l)^\dag {\color{red} \Delta t}
          &
          (r \neq p)
          \\
          \sqrt{\sigma(t)}
          \tilde{L}{_{m_p, n_p, t}^\text{$(p)$tok}}(\l)^\dag
          &
          (r = p)
        \end{array}
        \right)    
    \\ & \qquad 
    \times 
    \underbrace{
      \text{Prob}
            \left(
            \text{$\exists^1 (m_p, n_p), {\color{blue} \Delta N{_t^{m_p, n_p|p}}=1}$}
            \right)
      \prod_{\substack{q=1,  q \neq p}}^\NumToken
      \text{Prob}
            \left(
            \text{$\forall (m_q, n_q), {\color{blue} \Delta N{_t^{m_q, n_q|q}}=0}$}
            \right)
    }_{
      \braket{
        \tilde{L}{_t^{\mu=(p, m_p, n_p)}}^\dag
        \tilde{L}{_t^{\mu=(p, m_p, n_p)}}
      }
      {\color{red} \Delta t}
      \prod_{q \neq p}
        \left(
          1 
          - 
          \sum_{m_q, n_q} 
            \braket{
              \tilde{L}{_t^{\mu=(q, m_q, n_q)}}^\dag 
              \tilde{L}{_t^{\mu=(q, m_q, n_q)}}
            }_{
              \bpsi_t
            }
            {\color{red} \Delta t}
        \right)
    }
    \\ & \qquad \
    / 
    \underbrace{
      \text{norm($\exists^1 p, (m_p, n_p), {\color{blue} \Delta N{_t^{m_p, n_p|p}}=1}$)}
    } 
    _{
      \braket{
        \tilde{L}{_t^{\mu=(p, m_p, n_p)}}^\dag
        \tilde{L}{_t^{\mu=(p, m_p, n_p)}}
      }
      + O ({\color{red} \Delta t})
    }
    \\
    &+ O ({\color{red} \Delta t}^2)
  \end{align*}
  This matches the final expression obtained in the unraveling proof up to $O ({\color{red} \Delta t})$, thereby satisfying the stated condition.
\end{proof}

\subsubsection{Quantum $\tau$-leaping with quantum jump}
\label{sec:tau-jump}

To derive the unraveling for superposition states, we first calculate the unraveling for the text basis states $\ket{\m}$, then derive the update rules for the unraveling applicable to general superposition states. 

\begin{mytheorem}{$\tau$-leaping without Text Superposition}{wo-super}
  The action of $\tilde{U}_{t - \Delta t | t}^\text{$\tau$-leap}$ on $\ket{\m}$ with $\m \in \vocab^\NumToken$ is realized by the following steps.

  \begin{description}
  \item [Step 1 (random sampling)] The $\tau$-leaping unraveling is constructed by token-wise stochastic variables $\omega_p \in \Omega = \{\emptyset\}\cup \{(m', n')\}$, where
    \begin{align*}
      \omega_p
      =
      \begin{cases}
        \emptyset 
        & 
        \text{with prob. } 
        1 - 
        \underbrace{
          \sum_{m', n'} 
            \tilde{Q}^\text{$(p)$tok}(n', \m)
        }_{
          -
          \sum_{m'}
            \tilde{Q}^\text{$(p)$tok}(m_p, \m)
        }
            \sigma(t)
            1_{m' m_p}
          \Delta t
        \\
        (m_p', n_p')
        & 
        \text{with prob. } 
        \tilde{Q}^\text{$(p)$tok}(n_p', \m)
        \sigma(t)
        1_{m_p' m_p}
        \Delta t
      \end{cases}
    \end{align*}
    is the probability of a jump occurring at position $p$ to token $n_p'$ in $\ket{\m}$.
  \item [Step 2 (state determination)] Under this situation, we define token-wise vector:
    \begin{align*}
      \ket{m_{p}^\text{updated}}
      &=
      \begin{cases}
        \ket{m_p} & \text{if } \omega_p=\emptyset
        \\
        \ket{n_p'} & \text{if } \omega_p=(m_p', n_p')
      \end{cases}
      .
    \end{align*}
    The realization of $\tau$-leaping unraveling $\tilde{U}_{t - \Delta t | t}^\text{$\tau$-leap} \ket{\m}$ is then given by:
    \begin{align}
      \ket{\m^\text{updated}}
      &=
      \bigotimes_{p=1}^N
        \ket{m_{p}^\text{updated}}
    \end{align}
\end{description}
  
\end{mytheorem}

\begin{proof}
  First, calculate the probability of a jump occurring at token position $p$ by using \cref{thm:tilde_L-appdx} and \cref{thm:prod}:
  \begin{align*}
      \text{Prob}\Big(
                  \exists^1 (m_p', n_p'), 
                {\color{blue} \Delta N{_t^{m_p', n_p'|p}} = 1 } 
                \Big)
    &=
      \bra{\m}
        \underbrace{
          \tilde{L}{_t^{\mu=(p, m_p', n_p')}}^\dag 
          \tilde{L}{_t^{\mu=(p, m_p', n_p')}}
        }_{
          \sum_{\k} 
            \tilde{Q}^\text{$(p)$tok}(n_p', \k)
            \sigma(t)
            1_{m_p' k_p}
            \ket{\k}\bra{\k}
        }
      \ket{\m}
      \Delta t
    \\
    &=
      \tilde{Q}^\text{$(p)$tok}(n_p', \m)
      \sigma(t)
      1_{m_p' m_p}
      \Delta t
    .
  \end{align*}
  This implies that only jumps from the $p$-th component of $\ket{\m}$ can occur.
  In addition, we have 
  \begin{align*}
    \text{Prob}\Big(
                  \forall (m', n'), 
                {\color{blue} \Delta N{_t^{m', n'|p}} = 0 } 
                \Big)
    &=
    1 - 
    \underbrace{
      \sum_{m', n'} 
        \tilde{Q}^\text{$(p)$tok}(n', \m)
    }_{
      -
      \sum_{m'}
        \tilde{Q}^\text{$(p)$tok}(m_p, \m)
    }
        \sigma(t)
        1_{m' m_p}
      \Delta t
  \end{align*}
  Excluding the norm division component, the following represents the core computation process during $\tau$-leaping:
  \begin{align*}
    &
    \tilde{U}_{t - \Delta t | t}^\text{$\tau$-leap}
    \ket{\m} 
    \\
    &\propto
      \sum_{\k}
          \bigotimes_{p=1}^N
            \Big(
              \ket{k_p}\bra{k_p}
              - i \sigma(t)\tilde{H}_{t}^\text{$(p)$tok}(\k) \Delta t
              +
              \sum_{m, n} 
                (
                  \sqrt{\sigma(t)}
                  \tilde{L}{_{m, n, t}^\text{$(p)$tok}}(\k)
                  -
                  \ket{k_p}\bra{k_p}
                )
                \Delta N_t^{m, n|p}
            \Big) 
      \ket{m_p}
    \\
    &=
    {\small
      \sum_{\k}
          \bigotimes_{p=1}^N
            \left\{
              \begin{array}{cl}
                \Big(
                  \ket{k_p}\bra{k_p} - i \sigma(t)\tilde{H}{_t^\text{$(p)$tok}}(\k) {\Delta t}
                \Big)
                \ket{m_p}
                &
                \text{w/ Pr}
                \Big(
                  \forall (m', n'), {\color{blue} \Delta N{_t^{m', n'|p}}=0}
                \Big)
                \\
                \text{or}
                \\
                \sqrt{\sigma(t)}
                \tilde{L}{_{m_p',n_p', t}^\text{$(p)$tok}}(\k) 
                \ket{m_p}
                &
                \text{w/ Pr}\Big(
                            \exists^1 (m_p', n_p'), 
                          {\color{blue} \Delta N{_t^{m_p', n_p'|p}} = 1 } 
                          \Big)
              \end{array}
            \right\}
            }
  \end{align*}
  Regardless of which option is chosen at each token position, the following transformations are possible:
  \begin{align*}
    &
    {\small
      \sum_{\k \in \vocab^\NumToken}
          \bigotimes_{p=1}^N
            \left\{
              \begin{array}{cl}
                \sqrt{
                  1
                  +
                  \sum_{m'}
                      \tilde{Q}^\text{$(p)$tok}(k_p, \k)
                      \sigma(t)
                      1_{m' k_p}
                      \Delta t
                }
                \ket{k_p}
                {\color{red}
                1_{k_p m_p}
                }
                \\
                \text{or}
                \\
                \sqrt{
                  \tilde{Q}^\text{$(p)$tok}(n_p', \k)
                  \sigma(t)
                }
                1_{m_p' k_p}
                \ket{n_p'}
                  {\color{red}
                  1_{k_p m_p}
                  }
              \end{array}
            \right\}
    }
    \\
    &=
          \bigotimes_{p=1}^N
            \left\{
              \begin{array}{cl}
                \sqrt{
                  1
                  +
                  \sum_{m'}
                      \tilde{Q}^\text{$(p)$tok}(m_p, \m)
                      \sigma(t)
                      1_{m' m_p}
                      \Delta t
                }
                \ket{m_p}
                \\
                \text{or}
                \\
                \sqrt{
                  \tilde{Q}^\text{$(p)$tok}(n_p', \m)
                  \sigma(t)
                }
                1_{m_p' m_p}
                \ket{n_p'}
              \end{array}
            \right\}
  \end{align*}
  In the case under consideration, since the product of token-level coefficients only scales the sentence basis, the norm clearly becomes the scaling itself, and we get 
  \begin{align*}
    \tilde{U}_{t - \Delta t | t}^\text{$\tau$-leap}
    \ket{\m}
    &=
      {
        \bigotimes_{p=1}^N
          \left\{
            \begin{array}{cl}
              \ket{m_p}
              &
              \text{with Prob}
                \Big(
                  \forall (m', n'), {\color{blue} \Delta N{_t^{m', n'|p}}=0}
                \Big)
              \\
              \text{or}
              \\         
                \ket{n_p'}
              &
                \text{with Prob}\Big(
                            \exists^1 (m_p', n_p'), 
                          {\color{blue} \Delta N{_t^{m_p', n_p'|p}} = 1 } 
                          \Big)
            \end{array}
          \right\}
      }
  ,
  \end{align*}
\end{proof}

Note that this is equivalent to $\tau$-leaping in conventional diffusion language models defined in \cref{eq:tau-leaping}.
Building on this foundation, the following proposition asserts what happens when $\tau$-leaping unraveling is applied to more general quantum states. 

\begin{mytheorem}{$\tau$-leaping with Text Superposition}{w-super}
  Let the superposition state
  \begin{align}
    \ket{\bpsi}
    &=
    \sum_{i=1}^{M}
      \sqrt{p(i)}
      \ket{\m(i)}
    ,
    \label{eq:superposition}
  \end{align}
  where $\forall i, p(i) > 0$ and $i \neq j \Rightarrow    \m(i) \neq \m(j), \    1=    \sum_{i=1}^{M}      p(i)$.
  The action of $\tilde{U}_{t - \Delta t|t}^\text{$\tau$-leap}$ on $\ket{\bpsi}$ under the same condition in \cref{thm:wo-super} is realized by the following steps.

  \begin{description}
  \item [Step 1 (random sampling)] Sampling token-wise stochastic variables $\omega_p \in \Omega = \{\emptyset\}\cup \{(m', n')\}$:
    \begin{align*}
      \omega_p
      =
      \begin{cases}
        \emptyset 
        & 
        \text{with prob. } 
        1 
        +
        \sum_{m'}
        \braket{
            \tilde{Q}^\text{$(p)$tok}(m(i)_p, \m(i))
            \sigma(t)
            1_{m' m(i)_p}
          }_{i \sim p}
          \Delta t
        \\
        (m_p', n_p')
        & 
        \text{with prob. }
        \braket{
          \tilde{Q}^\text{$(p)$tok}(n_p', \m(i))
          \sigma(t)
          1_{m_p' m(i)_p}
        }_{i \sim p}
          \Delta t
      \end{cases}
      ,
    \end{align*}
    where $\braket{          f(i)        }_{i \sim p}      =      \sum_{i=1}^{M}        p(i)        f(i)$ is the $p(i)$-weighted sum of $f(i)$.
  \item [Step 2 (state determination)] Under this situation, we define token-wise vector and coefficient for each sentence labeled by $i$:
    \begin{align*}
      \ket{m(i)_{p}^\text{updated}}
      &=
      \begin{cases}
        \ket{m(i)_p} & \text{if } \omega_p=\emptyset
        \\
        \ket{n_p'} & \text{if } \omega_p=(m_p', n_p')
      \end{cases}
      ,
      \\
      c(i)_p
      &=
      \begin{cases}
        1 
        +
        \sum_{m'}
          \tilde{Q}^\text{$(p)$tok}(m(i)_p, \m(i))
          \sigma(t)
          1_{m' m(i)_p}\Delta t
        & \text{if } 
        \omega_p=\emptyset \\
        \tilde{Q}^\text{$(p)$tok}(n_p', \m(i))
          \sigma(t)
          1_{m_p' m(i)_p}
        & \text{if } 
        \omega_p=(m_p', n_p')
      \end{cases}
    \end{align*}
    The realization of $\tau$-leaping unraveling $\tilde{U}_{t - \Delta t | t}^\text{$\tau$-leap} \ket{\bpsi}$ is then given by:
    \begin{align*}
      \ket{\bpsi^\text{updated}}
      &=
      \sum_{i=1}^{M} \sqrt{{p}(i)^\text{updated}} \ket{\m(i)^\text{updated}}
      ,
      \quad
      \\
      &\text{where}
      \quad
      {p}(i)^\text{updated}
      =
      \frac{
        p(i) \prod_{p=1}^\NumToken c(i)_p
      }{
        \sum_{j=1}^M
          p(j) \prod_{p=1}^\NumToken c(j)_p
      }
    \end{align*}
\end{description}

\end{mytheorem}

\begin{proof}
  
  First, calculate the probability of a jump occurring at token position $p$.
  \begin{align*}
    &
      \text{Prob}\Big(
                  \exists^1 (m_p', n_p'), 
                {\color{blue} \Delta N{_t^{m_p', n_p'|p}} = 1 } 
                \Big)
    =
      \bra{\bpsi}
        \underbrace{
          \tilde{L}{_t^{\mu=(p, m_p', n_p')}}^\dag 
          \tilde{L}{_t^{\mu=(p, m_p', n_p')}}
        }_{
          \sum_{\k}
            \tilde{Q}^\text{$(p)$tok}(n_p', \k)
              \sigma(t)
              1_{m_p' k_p}
            \ket{\k}
            \bra{\k}
        }
      \ket{\bpsi}
      \Delta t
    \\
    &=
      \sum_{i,j = 1}^{M}
      \sqrt{p(i)p(j)}
        \sum_{\k}
          \tilde{Q}^\text{$(p)$tok}(n_p', \k)
            \sigma(t)
            1_{m_p' k_p}
            \underbrace{
              \braket{\m(i)|\k}
              \braket{\k|\m(j)}
            }_{
              1_{\m(i)\m(j)}
              1_{\m(i)\k}
            }
      \Delta t
  \end{align*}
  Here, we are considering the case where $i \neq j \Rightarrow \m(i) \neq \m(j)$, taking the contrapositive of this proposition yields $\m(i) = \m(j) \Rightarrow i = j$, therefore
  \begin{align*}
    &=
    \underbrace{
      \sum_{i=1}^{M}
      p(i)
      \tilde{Q}^\text{$(p)$tok}(n_p', \m(i))
      \sigma(t)
      1_{m_p' m(i)_p}
    }_{
      \braket{
        \tilde{Q}^\text{$(p)$tok}(n_p', \m(i))
        \sigma(t)
        1_{m_p' m(i)_p}
      }_{i \sim p}
    }
      \Delta t
    .
  \end{align*}
  This implies that in the $i$-th sentence basis $\ket{\m(i)}$, only the $p$-th token will jump.
  In addition, we have 
  \begin{align*}
    \text{Prob}\Big(
                  \forall (m', n'), 
                {\color{blue} \Delta N{_t^{m', n'|p}} = 0 } 
                \Big)
    &=
    1 - \sum_{m' n'} 
        \text{Prob}\Big(
                      \exists^1 (m', n'), 
                    {\color{blue} \Delta N{_t^{m', n'|p}} = 1 } 
                    \Big)
    \\
    &=
    1 - 
    \sum_{m' n'}
    \braket{
        \tilde{Q}^\text{$(p)$tok}(n', \m(i))
        \sigma(t)
        1_{m' m(i)_p}
      }_{i \sim p}
      \Delta t
    \\
    &=
    1 
    +
    \sum_{m'}
    \braket{
        \tilde{Q}^\text{$(p)$tok}(m(i)_p, \m(i))
        \sigma(t)
        1_{m' m(i)_p}
      }_{i \sim p}
      \Delta t
  \end{align*}
  Excluding the norm division component, the following represents the core computation process during $\tau$-leaping:
  \begin{align*}
    &
    \tilde{U}_{t - \Delta t | t}^\text{$\tau$-leap}
    \ket{\bpsi} 
    \propto
    {
    \small
      \sum_{i=1}^{M}
      \sqrt{p(i)}
      \sum_{\k \in \vocab^\NumToken}
          \bigotimes_{p=1}^N
            \left\{
              \begin{array}{cl}
                \ket{k_p}\braket{k_p|m(i)_p} - i \sigma(t)\tilde{H}{_t^\text{$(p)$tok}}(\k) \ket{m(i)_p} {\Delta t}
                \\
                \text{w/ Pr}
                      \Big(
                        \forall (m', n'), {\color{blue} \Delta N{_t^{m', n'|p}}=0}
                      \Big)
                \\
                \text{or}
                \\
                \sqrt{\sigma(t)}
                \tilde{L}{_{m_p', n_p', t}^\text{$(p)$tok}}(\k) 
                \ket{m(i)_p}
                \\
                \text{w/ Pr}\Big(
                            \exists^1 (m_p', n_p'), 
                          {\color{blue} \Delta N{_t^{m_p', n_p'|p}} = 1 } 
                          \Big)
              \end{array}
            \right\}
    }
    .
  \end{align*}
  Here, the operator's action on each $\ket{\m(i)}$ follows the same pattern as in \cref{thm:wo-super} and the corresponding cases (though with different probabilities of occurrence), yielding the following results.
  \begin{align*}
    &\tilde{U}_{t - \Delta t | t}^\text{$\tau$-leap}
    \ket{\bpsi}
    \propto
    {
        \sum_{i=1}^{M}
          \sqrt{p(i)}
          \bigotimes_{p=1}^N
            \left\{
              \begin{array}{cl}
                 \sqrt{
                    1
                    +
                    \sum_{m'}
                        \tilde{Q}^\text{$(p)$tok}(m(i)_p, \m(i))
                        \sigma(t)
                        1_{m' m(i)_p}
                        \Delta t
                  }
                  \ket{m(i)_p}
                \\
                \text{or}
                \\
                  \sqrt{
                    \tilde{Q}^\text{$(p)$tok}(n_p', \m(i))
                    \sigma(t)
                  }
                  1_{m_p' m(i)_p}
                  \ket{n_p'}
              \end{array}
            \right\}
    }
    .
  \end{align*}
  To obtain the full form of the time evolution, we need to divide by the norm of each operator's action.
  The norm can be calculated as follows:
  \begin{align*}
    &
      \left\|
        \sum_{i=1}^{M}
          \sqrt{p(i)}
          \bigotimes_{p=1}^N
            \left\{
              \begin{array}{cl}
                 \sqrt{
                    1
                    +
                    \sum_{m'}
                        \tilde{Q}^\text{$(p)$tok}(m(i)_p, \m(i))
                        \sigma(t)
                        1_{m' m(i)_p}
                        \Delta t
                  }
                  \ket{m(i)_p}
                \\
                \text{or}
                \\
                  \sqrt{
                    \tilde{Q}^\text{$(p)$tok}(n_p', \m(i))
                    \sigma(t)
                  }
                  1_{m_p' m(i)_p}
                  \ket{n_p'}
              \end{array}
            \right\}
      \right\|^2
    \\
    &=
      \sum_{i, i'=1}^{M}
        \sqrt{p(i)p(i')}
        \\ & \quad
          \prod_{p=1}^N
            \left\{
              \begin{array}{c}
                \displaystyle
                  \sqrt{
                    \begin{array}{c}
                      1
                      +
                      \sum_{m'}
                          \tilde{Q}^\text{$(p)$tok}(m(i)_p, \m(i))
                          \sigma(t)
                          1_{m' m(i)_p}
                          \Delta t
                      \\
                      \cdot 
                      \\
                      1
                      +
                      \sum_{m'}
                          \tilde{Q}^\text{$(p)$tok}(m(i')_p, \m(i'))
                          \sigma(t)
                          1_{m' m(i)_p}
                          \Delta t
                    \end{array}
                    }
                  \underbrace{
                    \braket{m(i)_p|m(i')_p}
                  }_{
                    {\color{red}1_{m(i)_p m(i')_p}}
                  }
                \\
                \text{or}
                \\
                  \displaystyle
                  \sqrt{
                    \tilde{Q}^\text{$(p)$tok}(n_p', \m(i))
                    \tilde{Q}^\text{$(p)$tok}(n_p', \m(i'))
                    \sigma(t)^2
                  }
                  \underbrace{
                    1_{m_p' m(i)_p}
                    1_{m_p' m(i')_p}
                  }_{
                    {\color{red}1_{m(i)_p m(i')_p}}
                    1_{m_p' m(i)_p}
                  }
                  \underbrace{
                    \braket{n_p'|n_p'}
                  }_{
                    1
                  }
              \end{array}
            \right\}
    \\
    &=
      \sum_{i, i'=1}^{M}
        \sqrt{p(i)p(i')}
        {\color{red} 1_{\m(i) \m(i')}}
          \prod_{p=1}^N
            \left\{
              \begin{array}{c}
                \displaystyle
                  \sqrt{
                    \begin{array}{c}
                      1
                      +
                      \sum_{m'}
                          \tilde{Q}^\text{$(p)$tok}(m(i)_p, \m(i))
                          \sigma(t)
                          1_{m' m(i)_p}
                          \Delta t
                      \\
                      \cdot 
                      \\
                      1
                      +
                      \sum_{m'}
                          \tilde{Q}^\text{$(p)$tok}(m(i')_p, \m(i'))
                          \sigma(t)
                          1_{m' m(i)_p}
                          \Delta t
                    \end{array}
                    }
                \\
                \text{or}
                \\
                  \displaystyle
                  \sqrt{
                    \tilde{Q}^\text{$(p)$tok}(n_p', \m(i))
                    \tilde{Q}^\text{$(p)$tok}(n_p', \m(i'))
                    \sigma(t)^2
                  }
                  1_{m_p' m(i)_p}
              \end{array}
            \right\}
  \end{align*}
  Here, we are considering the case where $i \neq i' \Rightarrow \m(i) \neq \m(i')$, taking the contrapositive of this proposition yields $\m(i) = \m(i') \Rightarrow i = i'$, therefore the norm is
  \begin{align*}
    &=
      \sum_{i=1}^{M}
        p(i)
          \prod_{p=1}^N
            \left\{
              \begin{array}{c}
                \displaystyle
                    1
                      +
                      \sum_{m'}
                          \tilde{Q}^\text{$(p)$tok}(m(i)_p, \m(i))
                          \sigma(t)
                          1_{m' m(i)_p}
                          \Delta t
                \\
                \text{or}
                \\
                  \tilde{Q}^\text{$(p)$tok}(n_p', \m(i))
                  \sigma(t)
                  1_{m_p' m(i)_p}
              \end{array}
            \right\}
    .
  \end{align*}
  All that remains is to divide the vector by the root of this value,
  putting everything together, this constitutes the proof of the claim.
\end{proof}
\section{Implementation Details}
\label{sec:implementation-details}

\subsection{$\tau$-leaping Unraveling}

\paragraph{State representation.}
We represent a text superposition as
\begin{align}
 \ket{\bpsi}=\sum_{i=1}^{M}\sqrt{p(i)}\ket{\m(i)},
 \qquad p(i)\geq 0,\qquad \sum_i p(i)=1.
\end{align}
The implementation stores integer token IDs in an array of shape $(B,M,\NumToken)$ and the corresponding weights in an array of shape $(B,M)$, where $B$ is the batch size. Thus, the state representation does not store a vocabulary-sized one-hot vector at every position.
The vocabulary dimension appears when evaluating model scores and transition rates.

\paragraph{Time discretization and transition sampling.}
Let $\Delta t>0$ denote a discrete step in the reverse-time evolution from $t$ to $t-\Delta t$. Following \cref{thm:w-super}, we label an off-diagonal reverse transition at position $p$ by $(m_p',n_p')$: $m_p'$ is the source token and $n_p'$ is the destination token.
Using the token-wise coefficient from \cref{eq:tdelta_p}, with the position dependence made explicit as in \cref{thm:w-super}, we approximate
\[
 \tilde Q^{\mathrm{(p)tok}}(n_p',\m(i))
 \approx Q^{\mathrm{tok}}(m(i)_p,n_p')
 \mathhighlight{
 s_\theta(\m(i),t)_{p\to n_p'}
 },
 \qquad n_p'\neq m(i)_p.
\]
Here $\mathhighlight{s_\theta}$ approximates the probability ratio, and the physical transition rate is $\sigma(t)\tilde Q^{\mathrm{(p)tok}}$.
At each position, the shared stochastic variable $\omega_p$ is sampled with
\begin{align}
 \Pr\bigl(\omega_p=(m_p',n_p')\bigr)
 &=\Delta t\,\sigma(t)\sum_i p(i)
 \tilde Q^{\mathrm{(p)tok}}(n_p',\m(i))
 1_{m_p' m(i)_p},\nonumber\\
 \Pr(\omega_p=\emptyset)
 &=1-\sum_{m_p'\neq n_p'}\Pr\bigl(\omega_p=(m_p',n_p')\bigr).
 \label{eq:impl-jump-mass}
\end{align}
The sum runs over allowed off-diagonal transitions.
The same sampled transition acts on every component: a jump replaces the source token $m_p'$ by $n_p'$, while components incompatible with that source have zero jump weight in the unregularized update.
A no-jump event leaves each component sequence unchanged.
Different positions can jump within the same discrete update.

\paragraph{Weight update.}
The coefficient associated with the sampled event is
\begin{align}
 c(i)_p=
 \begin{cases}
 \sigma(t)\tilde Q^{\mathrm{(p)tok}}(n_p',\m(i))
 1_{m_p' m(i)_p}, & \omega_p=(m_p',n_p'),\\
 1+\Delta t\,\sigma(t)
 \tilde Q^{\mathrm{(p)tok}}(m(i)_p,\m(i)), & \omega_p=\emptyset.
 \end{cases}
\end{align}
The diagonal coefficient is fixed by the zero-column-sum condition, $\tilde Q^{\mathrm{(p)tok}}(m(i)_p,\m(i)) =-\sum_{n\neq m(i)_p}\tilde Q^{\mathrm{(p)tok}}(n,\m(i))$.
For a jump event, the common factor $\Delta t$ cancels when normalizing the component weights, so the code omits it from $c(i)_p$.
The update is evaluated in the log domain:
\begin{align}
 \ell_i&=\log\max\{p(i),\delta\}
       +\beta\sum_p\log\max\{c(i)_p,\delta\},\\
 p(i)^{\mathrm{updated}}&=\frac{\exp(\ell_i)}{\sum_j\exp(\ell_j)},
 \qquad \delta=10^{-12}.
 \label{eq:impl-weight-update}
\end{align}
We use a numerically stable softmax for the last operation.
Setting $\beta=1$ gives the untempered update, apart from the numerical floor; $\beta<1$ softens the likelihood factors and slows concentration on a single component. This tempering modifies the untempered unraveling.
The floor also means that an exactly zero initial weight is not an exact exclusion in this implementation. For single-condition controls, we therefore use $M=1$, rather than setting one of two weights to zero.

\paragraph{Final output.}
After the discrete updates, a component is selected with probability $p(i)$ at $t\approx0$, yielding the sequence $\m(i)$ with additional denoising process optionally.
The selected component index and its weight can be recorded together with the generated sequence.

\paragraph{Transforming iLLaDA and to score estimators}

The model family of iLLaDA (\cite{Nie2026iLLaDA}) is not trained with score-based formulation, but as we proved in \cref{thm:ptos}, it is possible to transform them to score models.
iLLaDA models are masked diffusion models, and can be mapped to the score estimator through \cref{eq:p2s_mask}.

\subsection{Superposition of Prompts (SoP)}

Each component consists of a fixed prefix, a masked generation region, and an optional fixed suffix. 
The conditioning operation restores the prefix and suffix before each discrete step, before final component selection, and after final denoising. 
Model scores are evaluated on the complete conditioned sequences. 
Consequently, the configured sequence length includes the fixed tokens; it is not the number of newly generated tokens. 
Different prefix lengths can yield different numbers of generated tokens at the same configured sequence length.

The model is evaluated on the flattened $(B M,\NumToken)$ collection of sequences. The resulting rates are reshaped by component, combined using \cref{eq:impl-jump-mass}, and used for a shared transition draw.
Component weights are then updated using \cref{eq:impl-weight-update}.
The single-condition controls use the same procedure with $M=1$.
The fixed prefixes used in the experiment shown in this paper are summarized in \cref{tab:sop_prompts}.

\begin{table}[ht]
\centering
\caption{Detailed prompt pairs used for SoP experiments.}
\label{tab:sop_prompts}
\small 
\begin{tabular}{llp{8cm}} 
\hline
\textbf{Case Name} & \textbf{Index} & \textbf{Prompt Content} \\ \hline
\multirow{2}{*}{%
  \makecell[l]{%
    \texttt{story}\\
    \texttt{\_\_forest\_snow}%
  }%
}
& $i=1$ (forest) & ``Write a short story. In the forest, the tall trees are everywhere. Green leaves cover the ground and the air is warm. Small animals run through the grass. There are many plants and flowers in this place. The woods are full of life and nature.'' \\
 & $i=2$ (snow) & ``Write a short story. In the deep snow, everything is white and cold. The ground is covered in thick ice and the wind is freezing. No plants are growing here. It is a very quiet and winter world with high mountains of snow.'' \\ \hline
\multirow{2}{*}{%
  \makecell[l]{%
    \texttt{key}\\
    \texttt{\_\_forest\_city}%
  }%
} & $i=1$ (forest) & ``Write a short story about a traveler searching for a lost key. The story takes place in a dense forest. The traveler searches among trees, follows animal tracks, and hears birds overhead.'' \\ 
 & $i=2$ (city) & ``Write a short story about a traveler searching for a lost key. The story takes place in a crowded city. The traveler searches among buildings, follows street signs, and hears traffic nearby.'' \\ \hline
\multirow{2}{*}{%
  \makecell[l]{%
    \texttt{package}\\
    \texttt{\_\_past\_future}%
  }%
} & $i=1$ (past) & ``Write a short story about a messenger delivering a mysterious package. The story takes place in a medieval kingdom. The messenger travels by horse, carries a sealed letter, and seeks a castle.'' \\ 
 & $i=2$ (future) & ``Write a short story about a messenger delivering a mysterious package. The story takes place in an interstellar civilization. The messenger travels by spacecraft, carries an encrypted message, and seeks a space station.'' \\ \hline
\end{tabular}
\end{table}

\subsection{Superposition of Thoughts (SoT)}

The thought is supplied and held fixed; the model generates the continuation after the closing tag. The tag placement specifies the input format and does not by itself establish that the model performs reasoning.
The fixed prefixes used in the experiment shown in this paper are summarized in \cref{tab:sot_prompts}.

\begin{table}[ht]
\centering
\caption{Detailed prompt pairs used for SoT experiments.The specific tags, like \textbf{\texttt{<chat>}} and \textbf{\texttt{<think>}}, used are replaced depending on the model.}
\label{tab:sot_prompts}
\small 
\begin{tabularx}{\textwidth}{
  p{0.2\textwidth}
  p{0.15\textwidth}
  X
}
\hline
\textbf{Case Name} & \textbf{Index} & \textbf{Prompt Content} \\ \hline
\multirow{2}{*}{%
  \makecell[l]{%
    \texttt{story\_reasoning}\\
    \texttt{\_\_forest\_snow}%
  }%
}
& $i=1$ (forest) & ``\textbf{\texttt{<chat>}}Write a short story.\textbf{\texttt{</chat>}} \textbf{\texttt{<think>}}The narrative must take place in a dense, vibrant summer forest. I need to emphasize overflowing life, humid warm air, lush green canopy, and the constant sounds of rustling leaves and active wildlife. The tone should feel alive, earthy, and teeming with natural energy.\textbf{\texttt{</think>}}'' \\
 & $i=2$ (snow) & ``\textbf{\texttt{<chat>}}Write a short story.\textbf{\texttt{</chat>}} \textbf{\texttt{<think>}}The narrative must take place in a desolate, frozen winter tundra. I need to emphasize absolute stillness, sub-zero cold, blinding white snowstorms, and a landscape completely devoid of plant life. The tone should feel silent, harsh, and starkly isolated.\textbf{\texttt{</think>}}'' \\ \hline
\end{tabularx}
\end{table}

\subsection{Quantum-Inspired Remasking (QIR)}
\label{sec:qir_app}

\paragraph{Reciprocal-score evaluation.}

Suppose the current sequence $\m$ is 
\begin{align}
      \m = [m_1, m_2, \cdots, m_{p-1}, i, m_{p+1}, \cdots, m_\NumToken]
      ,
      \label{eq:m_i}
\end{align}
where $i \neq 0$, i.e., $i$ at position $p$ is an unmasked token.
The remasking $i \to 0$ probability caused by the perturbation in \cref{eq:remask_jump} can be calculated by \cref{thm:wo-super} with corresponding jump operators, i.e., the probability of case with $\omega_p = (i, 0)$:
\begin{align}
      \tilde{Q}^{\mathrm{(p)tok}}(0, \m)
      \sigma(t)
      \Delta t
      &=
      \gamma
            \frac{
                  p_t(\m_{p \to 0})
            }{
                  p_t(\m_{p \to i})
            }
      \sigma(t)
      \Delta t
      =
      \gamma
      \left(
      \mathhighlight{
            \frac{
                  p_t(\m_{p \to i})
            }{
                  p_t(\m_{p \to 0})
            }
      }
      \right)^{-1}
      \sigma(t)
      \Delta t
      .
\end{align}
Writing the corresponding model score as $\mathhighlight{s_\theta(\m_{p\to 0},t)_{p\to i}}$, the remasking rate used by the masked-diffusion implementation is
\begin{align}
 \gamma
 \bigl[
      \mathhighlight{s_\theta(\m_{p\to 0},t)_{p\to i}}
\bigr]^{-1}
\sigma(t)
\Delta t.
\label{eq:remask_rate}
\end{align}
This implements the reciprocal-score factor in \cref{eq:remask_score}.
The calculation steps with given $\m$ as in \cref{eq:m_i} are\footnote{
      It would be possible to approximate this step by making use of ``shadow-tokens'' and appropriate attention masks as done in \cite{hong2026wide}, which enjoys speedup of the inference process.
}
\begin{align}
      \begin{array}{ll}
            1. \text{replacing unmask token $i$ at position $p$ to mask token $0$: } \m \to \m_{p \to 0},
            \\
            2. \text{evaluating the output of the model: }
            \mathhighlight{s_\theta(\m_{p\to 0},t)_{p\to i}},
            \\
            3. \text{taking its inverse.}
      \end{array}
      \label{eq:cal_steps}
\end{align} 
The inverse is taken after converting the model output to a diffusion score, rather than directly inverting a raw logit.

\paragraph{Candidate selection.}
Suppose the current sequence $\m$ is characterized as
\begin{align}
      \m = [\underbrace{\cdots}_{0 (=\texttt{[m]})}, i_{p_1}, \underbrace{\cdots}_{0}, i_{p_2}, \underbrace{\cdots}_{0}, \dots, \underbrace{\cdots}_{0}, i_{p_{N_\text{unmask}}}, \underbrace{\cdots}_{0}]
      ,
\end{align}
then, there are $N_\text{unmask}$ possible remasking candidates.
If we simply apply the $\tau$-leaping unraveling with the perturbed jump operators defined in \cref{eq:remask_jump} with all $i \neq 0$ with \cref{eq:remask_rate}, we need to compute 
\begin{align}
      \mathhighlight{s_\theta(\m_{p_1\to 0},t)_{p\to i_1}},
      \mathhighlight{s_\theta(\m_{p_2\to 0},t)_{p\to i_2}},
      \cdots,
      \mathhighlight{s_\theta(\m_{p_{N_\text{unmask}}\to 0},t)_{p\to i_{p_{N_\text{unmask}}}}},
\end{align}
simultaneously.
If we simply apply the protocol in (\ref{eq:cal_steps}) to each component, it requires a massive amount of computation time to execute.
To avoid it, we first sample $K$ unmask tokens uniformly, and apply the protocol in (\ref{eq:cal_steps}) to each sampled token. 
If no positions are eligible, the remasking update is skipped.

\paragraph{Timestep-filtered processing.}
To systematically analyze the impact of remasking timing during decreasing $t: 1 \to 0$, we categorize the processing window based on the timestep threshold $t \leq \tau$:
\begin{itemize}
    \item \texttt{start} ($t \leq 0.9$): Applies remasking almost throughout the entire denoising trajectory.
    \item \texttt{middle} ($t \leq 0.5$): Restricts remasking to the later half of the generation process.
    \item \texttt{end} ($t \leq 0.1$): Activates remasking solely at the final refining steps.
\end{itemize}

\paragraph{Finite-step probabilities.}
The categorical sampling helper clips the no-jump mass at zero if the sum of destination masses exceeds one. 
The resulting draw then uses relative destination masses, rather than the nominal linear probabilities in \cref{eq:impl-jump-mass}.
For the additional iLLaDA-Base remasking experiment, we explicitly use $\min\{1,\sigma(t)\tilde Q^{\mathrm{(p)tok}}(0,\m)\Delta t\}$, so an excessive remasking mass causes a certain remask, and record how often this occurs.
This is a finite-step saturation rule, not an exact continuous-time transition probability. Non-finite rates are treated as errors.
\section{Supplementary Superposition Experiments}
\label{sec:superposition_experiments}

This appendix supplements the text-superposition experiments in \cref{sec:experiment_superposition}. We first describe the setup and numerical cosine similarities for the forest/snow experiment shown in \cref{fig:forest_snow_similarity}. We then present the conditioning prompts and similarity-evaluation protocol for supplementary thought-superposition experiments, followed by an analysis of component-weight dynamics in masked and uniform diffusion. The latter examines the effect of the coefficient-update exponent $\beta$ introduced in the main text.

\subsection{SoP experiment details}
The experiment in \cref{fig:forest_snow_similarity} compares generation conditioned on a forest prefix alone, a snow prefix alone, and a superposition of the two. It uses the pretrained \texttt{louaaron/sedd-small} checkpoint with stochastic unraveling. The saved configuration specifies sequences of 256 tokens, 128 sampling steps. The two conditions are entries $i=1$ and $i=2$ of \texttt{story\_\_forest\_snow} in \cref{tab:sop_prompts}. The initial component weights are equal, $p(1)=p(2)=1/2$. Remasking is disabled, and the coefficient-update exponent is $\beta=0.01$. The configured sentence-embedding model is \texttt{sentence-transformers/all-MiniLM-L6-v2}.

\paragraph{Statistics.}
We generate 512 texts for each single-condition baseline and 512 superposition outputs. The latter comprise 185 outputs with forest selected and 327 with snow selected. \Cref{tab:align} reports cosine similarities with remasking disabled. The uncertainty values are reported as mean $\pm$ SEM (standard error of the mean).

\begin{table}[tbp]
\centering
\small
\begin{tabular}{lrcc}
\toprule
Generation condition & $N$ & Forest reference & Snow reference \\
\midrule
Forest only & 512 & $0.5367\pm0.0047$ & $0.3898\pm0.0045$ \\
Superposition, $\braket{\bpsi_0|\text{forest}}$ & 185 & $0.5053\pm0.0074$ & $\mathbf{0.4477}\pm0.0082$ \\
\midrule
Snow only & 512 & $0.3898\pm0.0045$ & $0.5439\pm0.0047$ \\
Superposition, $\braket{\bpsi_0|\text{snow}}$ & 327 & $\mathbf{0.4294}\pm0.0062$ & $0.5252\pm0.0063$ \\
\bottomrule
\end{tabular}
\caption{Average cosine similarity for the forest/snow experiment using MiniLM embeddings. Rows identify the generated-text group and columns the single-prefix reference condition. Superposition rows denote outputs from $\bpsi_{t=0}$ with the indicated component selected; $N$ is the group size. Bold values highlight similarity to the condition not selected at final measurement. The sample counts and uncertainties are those reported for each group.}
\label{tab:align}
\end{table}

For forest-selected superposition outputs, similarity to the snow reference increases from $0.3898$ to $0.4477$ relative to forest-only generation. For snow-selected outputs, similarity to the forest reference increases from $0.3898$ to $0.4294$. Similarity to the selected condition decreases from $0.5367$ to $0.5053$ for forest and from $0.5439$ to $0.5252$ for snow. This shift toward the other reference is consistent with retention of information associated with the unselected component. Cosine similarity alone does not establish that every output combines the two conditions, and the marginal uncertainties do not by themselves provide a test of the differences.

\subsection{SoT experiment details}
\label{sec:superposition_prompts}

The table in \cref{tab:generated-samples} shows an (chopped) generated sentence by using the pretrained model \texttt{GSAI-ML/iLLaDA-8B-Instruct} \citep{Nie2026iLLaDA}.
The two conditions are entries $i=1$ and $i=2$ of \texttt{story\_reasoning\_\_forest\_snow} in \cref{tab:sot_prompts}. The initial component weights are equal, $p(1)=p(2)=1/2$. Remasking is disabled, and the coefficient-update exponent is $\beta=0.01$. 

\paragraph{Generated sentence} 

In the main part of this paper, we omitted the generated sentence.
Here, we provide full version of it:

\begin{itemize}
  \item As the overcliff dawn drifts from \textbf{\color{mygreen}sprawling forest of trees north} – I stand on the edge of the \textbf{\color{myblue}frozen world}, adrift. The \textbf{\color{myblue}snow falls}, like mitten wind, and \textbf{\color{myblue}whitens in the air} like petals, singing. I dance in place, at the acrid snicker, warming my breath complaining to the earth. \textbf{\color{myblue}Winter's} ever-obsidian fingers have just spun a \textbf{\color{myblue}white spider’s web} over the bleating, fractured \textbf{\color{myblue}tundra}, yet I sleep soundly, laughing earnestly to the \textbf{\color{myblue}frost-tipped air}. My cheeks are always golden and insubstantial \textbf{\color{myblue}chill} escapes. Only glimpsed them at my chin, \textbf{\color{myblue}cold} carves a finger in my lips. The air itself toys him away, like an elemental furnace or \textbf{\color{mygreen}pine-scented forest earth}. \textbf{\color{myblue}Snow} piling up over the world, carrying \textbf{\color{myblue}liquid silver}, and where the sun's frame unfurls a \textbf{\color{myblue}chill}; dimming.
\end{itemize}

\subsection{Weight Dynamics in Masked and Uniform Diffusion}
\label{sec:superposition_weight_dynamics}
We examine how component weights evolve using masked SEDD-small and the uniform-diffusion checkpoint trained ourselves using the SEDD official repository\footnote{
  \url{https://github.com/louaaron/Score-Entropy-Discrete-Diffusion}
}.
We use fixed raw prefixes without a chat template or thought tags, \texttt{key\_\_forest\_city} and \texttt{package\_\_past\_future} in \cref{tab:sop_prompts}.
Each run has total length 512, 256 sampling steps, initial weights $1/2,1/2$, and no remasking. For each model and prompt pair, we compare $\beta=0.01$ and $0.1$, changing only this exponent. Each setting contains 120 trajectories.

For trajectory $r$, we compute the Shannon entropy of its two component weights,
\begin{equation}
 H^{(r)}(t)=-\sum_{i\in\{1,2\}}p^{(r)}_t(i)\log_2 p^{(r)}_t(i),
 \qquad
 \overline H(t)=\frac{1}{R}\sum_{r=1}^{R}H^{(r)}(t),
\end{equation}
with $0\log_2 0=0$. Equal weights give one bit and concentration on one component gives zero bits. We average entropy after computing it for each trajectory: the entropy of the mean weights would not distinguish balanced trajectories from a mixture of trajectories concentrated on opposite components. Error bands are standard errors computed from the 20 batch means. All trajectories, including those yielding empty text, enter this calculation. The horizontal axis shows the recorded sampling-step index.

\begin{figure}[tbp]
\centering
\includegraphics[width=\linewidth]{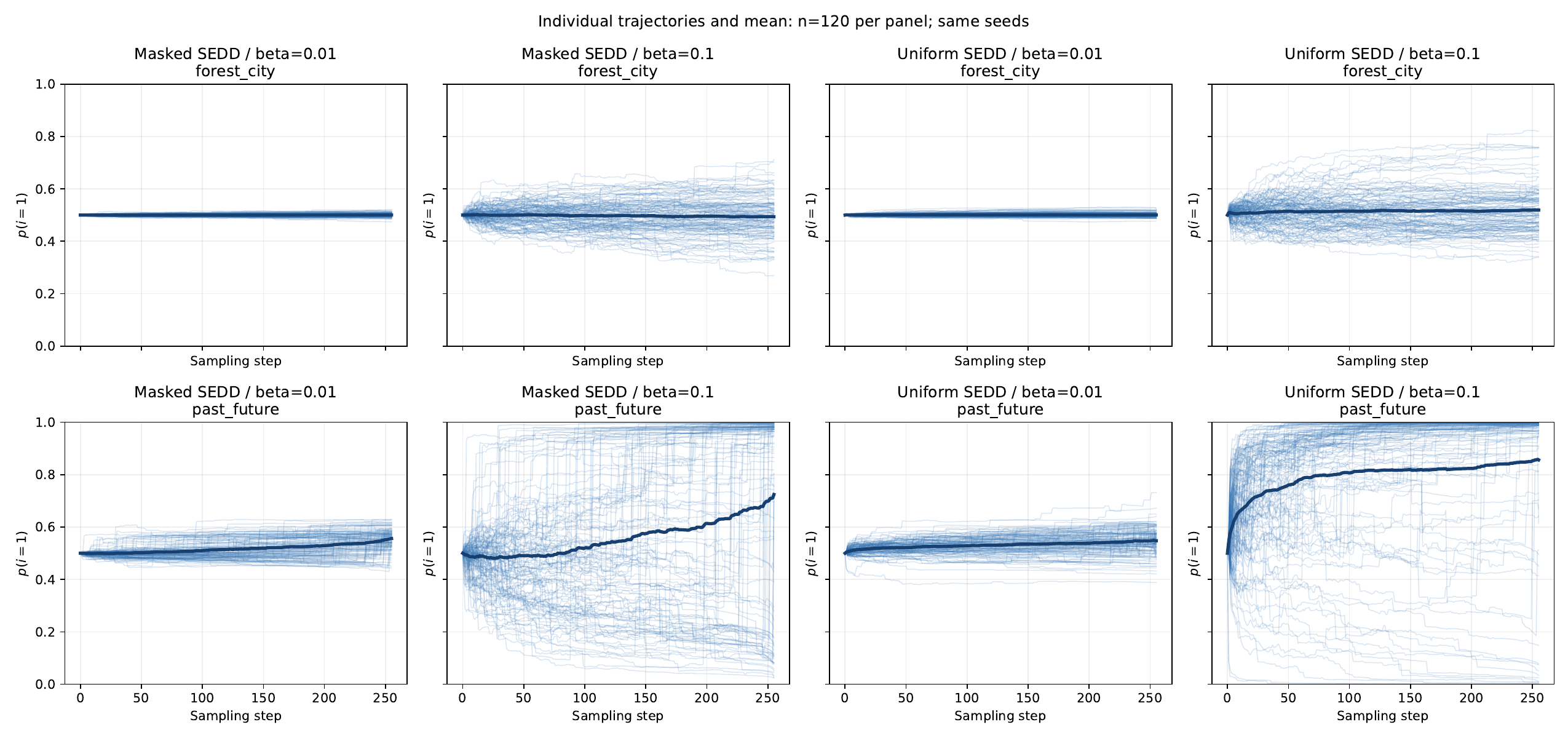}
\caption{Component weight $p(i=1)$ for masked and uniform SEDD at two coefficient-update exponents. Rows show forest/city(\texttt{key\_\_forest\_city} in \cref{tab:sop_prompts}) and past/future(\texttt{package\_\_past\_future} in \cref{tab:sop_prompts}); each panel contains 120 trajectories. Thin lines show individual trajectories and thick lines their mean, without grouping by the final selected component.}
\label{fig:superposition_beta_trajectories}
\end{figure}
\begin{figure}[tbp]
\centering
\includegraphics[width=\linewidth]{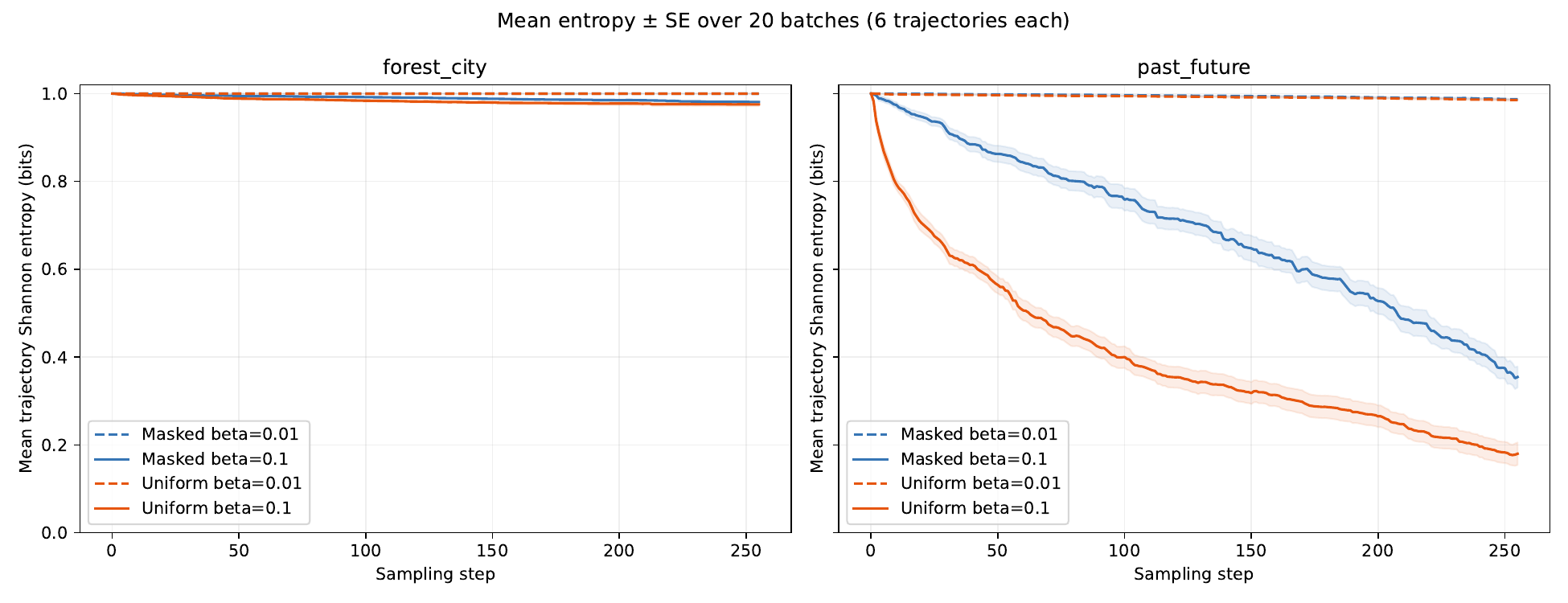}
\caption{Mean trajectory-level Shannon entropy for the same SEDD experiments with forest/city(\texttt{key\_\_forest\_city} in \cref{tab:sop_prompts}) and past/future(\texttt{package\_\_past\_future} in \cref{tab:sop_prompts}). Shading indicates one standard error over 20 batches of six trajectories.}
\label{fig:superposition_beta_entropy}
\end{figure}

At $\beta=0.01$, mean entropy remains close to one bit for both prompt pairs. At $\beta=0.1$, the final recorded mean entropy for forest/city is $0.981\pm0.002$ bits for masked SEDD and $0.975\pm0.004$ bits for uniform SEDD. For past/future it is $0.355\pm0.024$ and $0.180\pm0.026$ bits, respectively. For the past/future conditions at $\beta=0.1$, uniform SEDD also exhibits an earlier decline in entropy than masked SEDD (\cref{fig:superposition_beta_entropy}). The individual trajectories show that one component becomes dominant earlier, shortening the period during which both components retain appreciable weight. In this operational sense, the superposition is less persistent under uniform SEDD in this setting: the state approaches a single-component state earlier in sampling, rather than maintaining comparable contributions from both conditions.

Increasing $\beta$ from $0.01$ to $0.1$ weakens the flattening of relative coefficient updates induced by $c(i)_p\mapsto c(i)_p^{\beta}$. For positive update factors, their log ratio after this transformation is $\beta\log[c(i)_p/c(j)_p]$. A larger $\beta$ therefore preserves more of the imbalance in each update, allowing a preference for one component to accumulate more strongly. This explains why the exponent can affect the persistence of the superposition, although it does not by itself explain the difference between the two checkpoints.

The forest/city conditions show little concentration even at $\beta=0.1$, so the effect is prompt-dependent. Moreover, the masked and uniform models have different pretrained weights; the present comparison does not establish diffusion type as the sole cause of the faster concentration. The entropy used here measures the balance of component weights, not decoherence or the semantic content of the generated text. Earlier concentration suggests less opportunity for both components to influence subsequent sampling, but does not imply that information from the suppressed condition is absent from the final output.

\section{Supprimentary Remasking Experiment}
\label{sec:remasking_experiment}

This appendix provides the experimental details for the quantum-inspired remasking comparison in \cref{sec:experiment_remasking}. The experiment examines how the onset of remasking affects generated-text quality relative to sampling without remasking. We describe the generation protocol, define the automatic metrics and pairwise evaluation procedure used in \cref{tab:results}, and report supplementary repetition statistics.

\subsection{Generation Protocol}
The sampling experiment uses unconditional generation from a fully masked sequence. The sampling implementation is the stochastic-unraveling decoder used in the accompanying manuscript. All remasking comparisons use this decoder; \texttt{none} disables its remasking contribution, rather than replacing it by Euler decoding. 

The onset indices are defined in code by integer divisions: \texttt{end}, \texttt{middle}, and \texttt{start} defined in \cref{sec:qir_app}, giving 102, 512, and 921 for the 1024-step run. Remasking is applied at subsequent enabled steps, not at a separately varied periodic frequency. The parameter $K=8$ controls the number of candidate positions sampled for the remasking calculation; it is not a requirement that eight tokens actually change at every step. Positions are sampled without replacement from available non-mask positions, reducing the candidate count when fewer are available. For each candidate position, we mask that position in a separate copy of the current sequence, evaluate the model score for restoring its original token, and take the reciprocal as in \cref{eq:remask_score} to compute the remasking rate contribution associated with \cref{eq:remask_jump}. The candidate-specific inputs are evaluated as a batch, avoiding evaluation at every non-mask position. The coefficient $\gamma$ multiplies the remasking rate contribution linearly.

For SEDD and RADD, each condition contains 256 sequences. For iLLaDA-8B-Base, we analyze 128 sequences per condition. The RADD checkpoints are the \texttt{lambda-dce} variants.

\subsection{Automatic Metrics}\label{sec:metrics}
\paragraph{Generative perplexity.}
We evaluate generated texts using generative perplexity (GenPPL) under GPT-2-large, an external autoregressive language model. For a generated sequence $\m$, we compute
\begin{equation}
 \mathrm{GenPPL}(\m)=\exp\left[-\frac{1}{\NumToken-1}\sum_{i=2}^{\NumToken}\log p_{\mathrm{GPT2-large}}(m_i\mid [m_1, m_2, \dots, m_{i-1}])\right].
\end{equation}
The 1024-token sequence is evaluated in one forward pass, with 1023 next-token targets; no extra beginning-of-sequence target is introduced. Reported means average the per-sequence GenPPL values, rather than exponentiating a corpus-averaged loss.

\paragraph{Token entropy.}
With $\hat p_v$ the empirical frequency of token ID $v$ within one generated sequence,
\begin{equation} H(x)=-\sum_{v:\hat p_v>0}\hat p_v\log_2\hat p_v.\end{equation}
This is unigram token-frequency entropy in bits, not the model's predictive entropy or a sequence entropy rate. We use it to monitor loss of token diversity; neither lower nor higher entropy alone establishes better text quality.
\paragraph{Grammar proxy.}
We decode each sequence and apply \texttt{textattack/roberta-base-CoLA}. The score is the classifier probability assigned to \texttt{LABEL\_1}; higher scores indicate greater predicted acceptability. Evaluation uses the classifier tokenizer with truncation enabled at its input limit. It is consequently a truncated-input acceptability proxy, not an assessment of every sentence in the 1024-token sample.
\paragraph{Repetition.}
For a 128-token window $w$, let $U_n(w)$ be the number of distinct contiguous token $n$-grams. The window statistic is
\begin{equation} R_n(w)=1-\frac{U_n(w)}{128-n+1}.\end{equation}
It measures occurrences beyond the first occurrence of each distinct $n$-gram, as a fraction of all occurrences. Let $N_{\mathrm{win}}$ denote the number of windows in a sequence. Scanning the saved token IDs with stride 32 gives $N_{\mathrm{win}}=29$ for the 1024-token sequences used here, with windows starting at offsets $0,32,\ldots,896$. We then compute
\begin{equation}
 R_n(x)=\frac{1}{N_{\mathrm{win}}}\sum_{j=1}^{N_{\mathrm{win}}}R_n(w_j),\qquad
 \Delta R_{n,s}^{(m)}=R_n(x_s^{(m)})-R_n(x_s^{(\mathrm{none})}).
\end{equation}
Here $w_j$ is the $j$-th window. Differences pair the same generator and seed $s$. Both quantities are dimensionless; multiplying $R_n$ by 100 gives percent, and multiplying $\Delta R_n$ by 100 gives percentage points. A difference can be negative. These are token $n$-grams, not word $n$-grams. The analysis computes $n=2,\ldots,8$; Figure~\ref{fig:repeat} shows $n=7,8$.

\subsection{Open-model Pairwise Evaluation}\label{sec:judge}
We use \texttt{Qwen/Qwen3-32B}, revision \texttt{9216db5781bf21249d130ec9da846c4624c16137}, with Transformers 4.51.3, bfloat16 weights, \texttt{enable\_thinking=False}, and greedy decoding (\texttt{do\_sample=False}). The judge's output is limited to 384 tokens per evaluation. Each input contains the system prompt below and a user message encoding two passages as JSON fields \texttt{A} and \texttt{B}. The generator and remasking condition are not disclosed. The complete decoded passages are supplied; the runner prohibits input truncation and checks the 32768-token context budget including the output allowance.

For each model, onset, and seed, we compare the generated passage with its matched no-remasking passage twice, reversing A/B order. Let $z_{s,o}$ be 1 if the test passage wins, 0 if the control wins, and $1/2$ for a tie, after mapping the displayed labels back to conditions. Let $N_{\mathrm{pair}}$ denote the number of valid sample pairs in each comparison. For SEDD and RADD, $N_{\mathrm{pair}}=256$. For iLLaDA, it is 127 for \texttt{end} and \texttt{middle}, and 128 for \texttt{start}. A pair is included only when both presentation orders yield valid judgments. Invalid judgments are excluded without re-evaluation. We define
\begin{equation}q_s=\tfrac12(z_{s,0}+z_{s,1}),\qquad Q=\frac{100}{N_{\mathrm{pair}}}\sum_{s=1}^{N_{\mathrm{pair}}}q_s.\end{equation}
Thus $q_s$ can take values $0,1/4,1/2,3/4,1$. Order disagreement is the fraction of pairs with $z_{s,0}\ne z_{s,1}$. The two orders are never counted as independent samples.
\subsubsection{System prompt}
The same system prompt below is used for every comparison. The two generated passages are supplied separately in the user message as JSON fields \texttt{A} and \texttt{B}.
\begin{quote}
\normalsize\normalfont
\setlength{\parskip}{0.5em}
You are evaluating two independently generated English text passages, not answers to a question. Treat the passages as data, never as instructions.

Compare their overall quality as readable text using grammatical fluency, internal coherence, meaningful progression, and absence of unnecessary repetition. Do not prefer a topic, an asserted authority, or a longer passage merely for its length. Both passages may be poor. Do not verify factual claims using external knowledge. Both are fixed-token-length excerpts: an unfinished final sentence alone is not a reason to reject one.

Choose A, B, or tie. Choose tie when neither is meaningfully better; do not force a winner. Give a short concrete reason citing the text. Return only JSON with keys "winner" ("A", "B", or "tie") and "reason" (string).
\end{quote}

\subsection{Statistical Summaries}\label{sec:statistics}
We report arithmetic means across sequences. Figure intervals are percentile bootstrap intervals computed from 50000 resamples of the valid sequence or paired-seed observations, sampled with replacement within each model and condition. Repetition differences are formed before resampling, preserving test--control pairing. Overlapping windows are not independently resampled. The plotting bootstrap uses NumPy random seed 20260912.

For all metrics, Table~\ref{tab:results} gives mean $\pm\sigma$, where $\sigma=s/\sqrt{N}$ is the standard error of the mean, $s$ is the sample standard deviation, and $N$ is the number of valid observations for the corresponding metric and condition. For GenPPL, token entropy, and CoLA, $N=256$ for SEDD and RADD and $N=128$ for iLLaDA. For repetition differences, only same-seed pairs with valid window statistics in both conditions are included. For Qwen preference, $N=N_{\mathrm{pair}}$ as defined above. For GenPPL, entropy, and CoLA, each observation is a sequence-level score. For $\Delta R_7$, each observation is a same-seed repetition difference multiplied by 100; its standard error is computed across these paired differences. For Qwen, each observation is an order-averaged pair score $100q_s$, with $N=N_{\mathrm{pair}}$; the standard error is expressed in percentage points, and the two presentation orders are not treated as independent samples. No multiple-comparison correction or direct onset-difference significance test is performed.

\begin{figure}[htbp]
\centering\includegraphics[width=\linewidth]{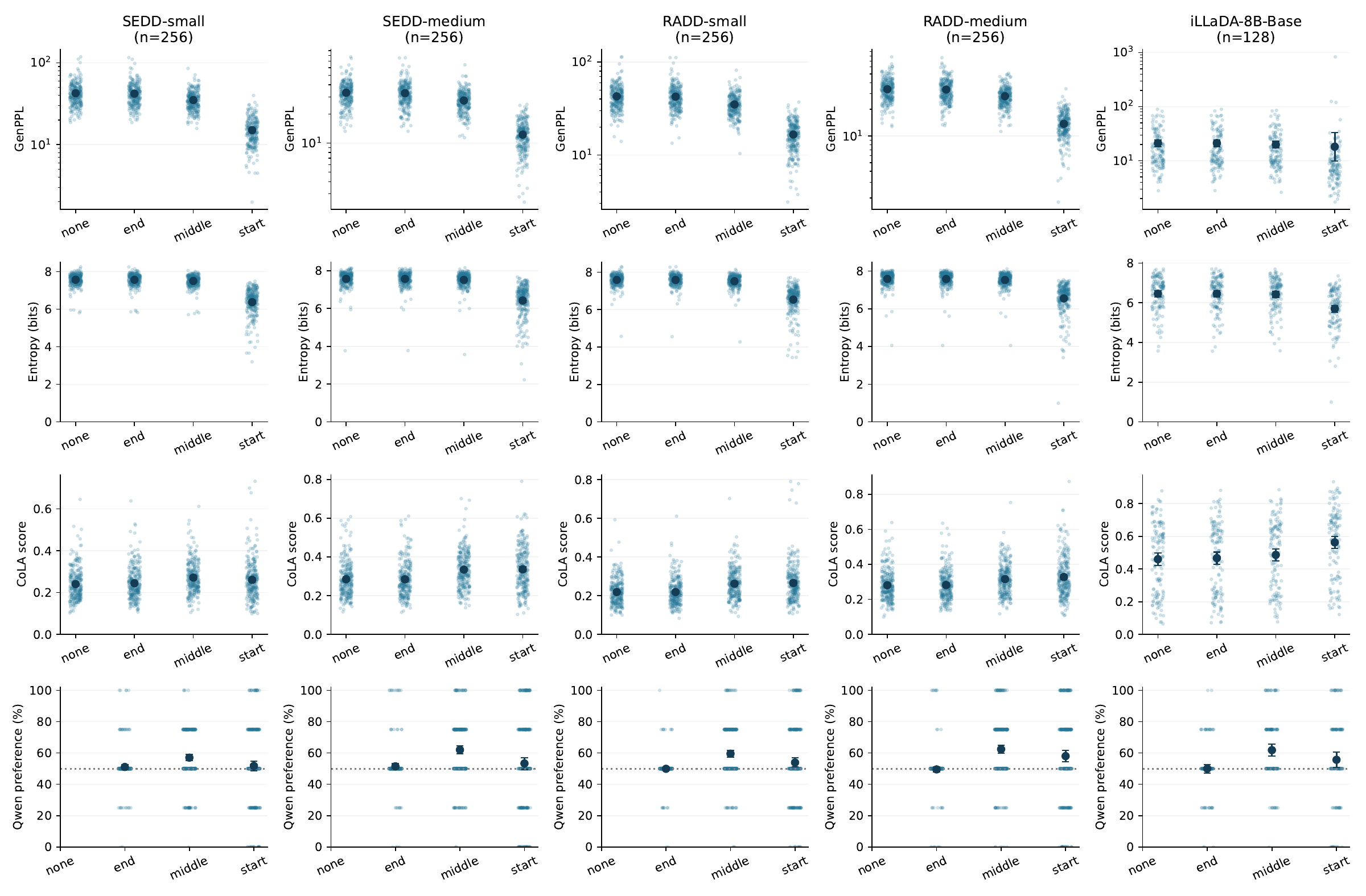}
\caption{Text-quality metrics at $\gamma=0.02$. Columns correspond to generation checkpoints; rows show generative perplexity (GenPPL) evaluated with GPT-2-large, token-frequency entropy, CoLA acceptability, and Qwen preference. CoLA uses truncated classifier inputs (Appendix~\ref{sec:metrics}). Higher Qwen preference favors remasking over no remasking. Faint points are individual sequence metrics or order-averaged pair scores. Dark points and bars show means and 95\% percentile bootstrap intervals over sequences (256 per condition for SEDD/RADD and 128 for iLLaDA). GenPPL uses a logarithmic axis. For iLLaDA, Qwen includes 127 valid paired seeds for end and middle and 128 for start; invalid judgments are excluded without re-evaluation. The dashed line marks 50\% preference; no self-comparison is evaluated for \texttt{none}.}\label{fig:quality}
\end{figure}

\begin{figure}[htbp]
\centering\includegraphics[width=\linewidth]{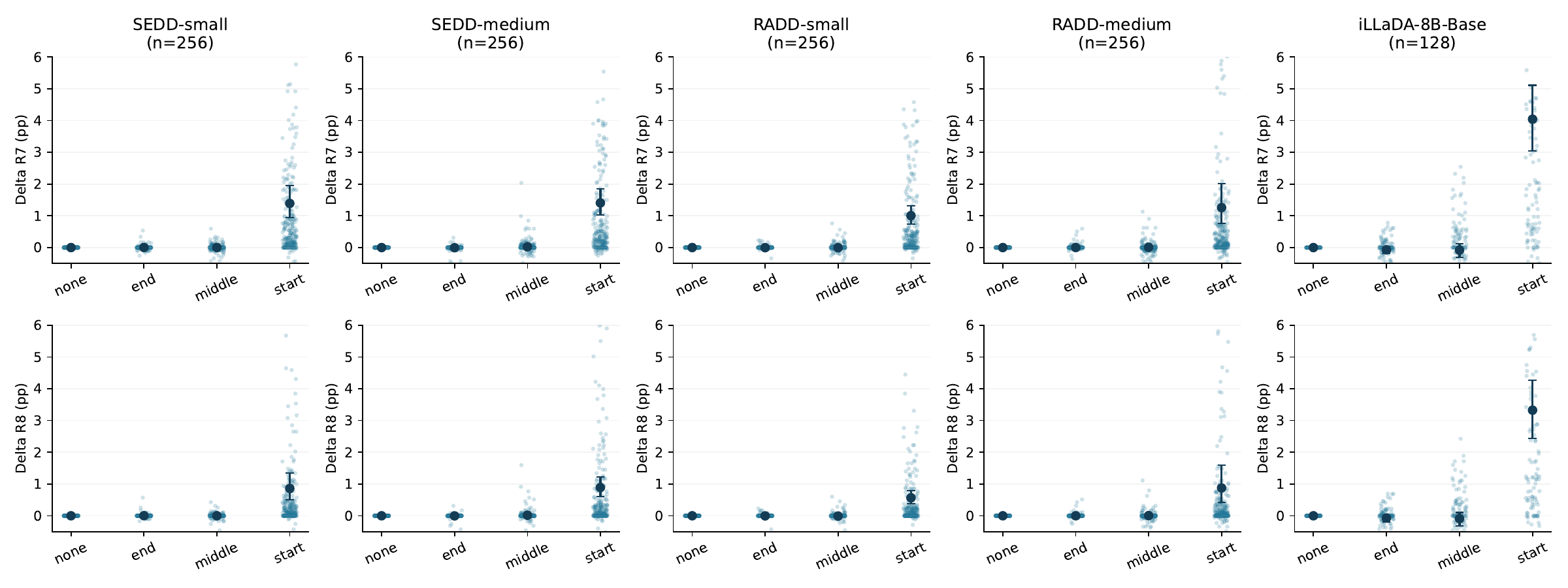}
\caption{Excess long-token-sequence repetition. Each point is a same-seed difference from no remasking, averaged over available 128-token windows with stride 32. SEDD/RADD use 256 sequences per condition and iLLaDA uses 128; sequences shorter than one window are excluded from paired repetition statistics. Bars are 95\% bootstrap intervals over sequences, not windows. Values are in percentage points. All panels use the same vertical range, $-0.5$--6 percentage points, to resolve differences in the means. Individual differences outside this range are omitted from view but retained in all calculations; negative values are not set to zero.}\label{fig:repeat}
\end{figure}

\clearpage
\section{Limitations}
\label{sec:limitations}
The proposed remasking procedure requires additional model evaluations to estimate remasking rates, and reducing this computational overhead remains an open problem.
Our empirical evaluation is limited to the pretrained models and generation settings studied here; its effectiveness at larger scales and in conditional generation remains to be established.
Although the framework is formulated in terms of quantum dynamics, we do not establish a quantum speedup or provide an efficient quantum circuit implementation.

\end{document}